\documentclass{article}

\newif\ifproceedings
\ifproceedings
\usepackage[final]{neurips_2023}
\else
\usepackage[preprint]{neurips_2023}
\fi

\usepackage[utf8]{inputenc} %
\usepackage[T1]{fontenc}    %
\usepackage{hyperref}       %
\usepackage{url}            %
\usepackage{booktabs}       %
\usepackage{amsfonts}       %
\usepackage{nicefrac}       %
\usepackage{microtype}      %
\usepackage{xcolor}         %
\usepackage[american]{babel}

\usepackage{natbib} %
\usepackage{braket}

\usepackage{mathtools} %
\usepackage{booktabs} %
\usepackage{tikz} %
\usepackage{amsthm}
\usepackage{nicefrac}       %
\usepackage{enumitem,multirow,multicol}
\usepackage{booktabs}

\usepackage{amssymb}
\usepackage{amsfonts}
\usepackage{amsmath}
\usepackage{caption}
\usepackage{subcaption}
\usepackage{algorithmicx, algpseudocode, algorithm}
\usepackage{xcolor}
\everypar{\looseness=-1	}
\allowdisplaybreaks
\hypersetup{
	colorlinks,
	linkcolor={red!40!gray},
	citecolor={blue!40!gray},
	urlcolor={blue!70!gray}
}

\DeclareMathOperator*{\argmax}{arg\,max}

\title{Likelihood Ratio Confidence Sets for \\ Sequential Decision Making}

\author{%
Nicolas Emmenegger\thanks{Equal contribution.} \\
ETH Zürich \\
\And 
Mojmír Mutný\footnotemark[1] \\
ETH Zürich \\
\And
Andreas Krause \\
ETH Zürich
}

\begin{document}

\maketitle

\begin{abstract}
  \looseness -1 Certifiable, adaptive uncertainty estimates for unknown quantities are an essential ingredient of sequential decision-making algorithms. Standard approaches rely on problem-dependent concentration results and are limited to a specific combination of parameterization, noise family, and estimator. In this paper, we revisit the likelihood-based inference principle and propose to use \emph{likelihood ratios} to construct \emph{any-time valid} confidence sequences without requiring specialized treatment in each application scenario. Our method is especially suitable for problems with well-specified likelihoods, and the resulting sets always maintain the prescribed coverage in a model-agnostic manner. The size of the sets depends on a choice of estimator sequence in the likelihood ratio. We discuss how to provably choose the best sequence of estimators and shed light on connections to online convex optimization with algorithms such as Follow-the-Regularized-Leader. To counteract the initially large bias of the estimators, we propose a reweighting scheme that also opens up deployment in non-parametric settings such as RKHS function classes. We provide a \emph{non-asymptotic} analysis of the likelihood ratio confidence sets size for generalized linear models, using insights from convex duality and online learning. We showcase the practical strength of our method on generalized linear bandit problems, survival analysis, and bandits with various additive noise distributions.
\end{abstract}

\newcommand{\spacebeforesubsections}{-0.15cm} %
\newcommand{\spacebeforesections}{-0.18cm} %
\newcommand{\spacebeforeparagraphs}{-0.13cm} %
\newcommand{\spaceaftertheorem}{-0.15cm} 
\newcommand{\spaceafterbanner}{-0.5cm}

\def\calC{\mathcal{C}}
\def\calD{\mathcal{D}}
\def\calE{\mathcal{E}}
\def\calF{\mathcal{F}}
\def\calG{\mathcal{G}}
\def\calL{\mathcal{L}}
\def\calO{\mathcal{O}}
\def\calS{\mathcal{S}}
\def\calT{\mathcal{T}}
\def\calM{\mathcal{M}}
\def\calX{\mathcal{X}}
\def\calH{\mathcal{H}}
\def\calD{\mathcal{D}}
\def\calN{\mathcal{N}}

\newcommand{\xstar}{{x_\star}}

\newcommand\calDD{\mathbf{\calD}}

\newcommand\DDelta{\mathbf{{\Delta}}}
\newcommand\Ppsi{\mathbf{{\Psi}}}
\newcommand\PPsi{\mathbf{{\Psi}}}
\newcommand\ppsi{\mathbf{{\psi}}}
\newcommand\pphi{\mathbf{{\phi}}}
\newcommand\PPhi{\mathbf{\Phi}}
\newcommand\LLambda{\mathbf{{\Lambda}}}
\newcommand\PPi{\mathbf{\Pi}}

\newcommand\ppi{\mathbf{\pi}}
\newcommand\cchi{\mathbf{\chi}}
\newcommand\aalpha{\mathbf{\alpha}}
\newcommand\bbeta{\mathbf{\beta}}
\newcommand\ggamma{\mathbf{\gamma}}
\newcommand\ddelta{\mathbf{\delta}}

\newcommand\rrho{\mathbf{\rho}}
\newcommand\xxi{\mathbf{\xi}}

\newcommand\er{R_{\text{eff}}}

\def\aa{\pmb{{a}}}
\newcommand\bb{\mathbf{{b}}}
\newcommand\cc{\mathbf{{c}}}
\newcommand\dd{\mathrm{{d}}}
\newcommand\ee{\mathbf{{e}}}
\newcommand\ff{\mathbf{{f}}}
\renewcommand\gg{\mathbf{{g}}}
\newcommand\ii{\mathbf{{i}}}
\newcommand\jj{\mathbf{{j}}}
\newcommand\kk{\mathbf{{k}}}
\renewcommand\ll{\mathbf{{l}}}
\newcommand\pp{\mathbf{{p}}}
\newcommand\qq{\mathbf{{q}}}
\newcommand\bs{\mathbf{{s}}}
\newcommand\nn{\mathbf{{n}}}
\newcommand\rr{\mathbf{{r}}}
\renewcommand\ss{\mathbf{{s}}}
\def\tt{\mathbf{{t}}}
\newcommand\uu{\mathbf{{u}}}
\newcommand\vv{\mathbf{{v}}}
\newcommand\ww{\mathbf{{w}}}
\newcommand\yy{\mathbf{{y}}}
\newcommand\zz{\mathbf{{z}}}
\newcommand\xx{\mathbf{{x}}}
\newcommand\mE{\mathbb{E}}
\newcommand\veczero{\mathbf{0}}
\newcommand\vecone{\mathbf{1}}

\newcommand\matzero{\mathbf{0}}
\newcommand\matone{\mathbf{1}}

\newcommand{\matlow}{\mathbf{{{\mathcal{L}}}}}
\newcommand{\matlowtil}{\mathbf{{\widetilde{\mathcal{L}}}}}
\newcommand{\matlowhat}{\mathbf{{\widehat{\mathcal{L}}}}}

\newcommand{\matup}{\mathbf{{{\mathcal{U}}}}}

\renewcommand\AA{\mathbf{{A}}}
\newcommand\BB{\mathbf{{B}}}
\newcommand\CC{\mathbf{{C}}}
\newcommand\DD{\mathbf{{D}}}
\newcommand\EE{\mathbf{{E}}}
\newcommand\GG{\mathbf{{G}}}
\newcommand\HH{\mathbf{{H}}}
\newcommand\II{\mathbf{{I}}}
\newcommand\JJ{\mathbf{{J}}}
\newcommand\KK{\mathbf{{K}}}
\newcommand\NN{\mathbf{{N}}}
\newcommand\MM{\mathbf{{M}}}
\newcommand\LL{\mathbf{{L}}}
\newcommand\PP{\mathbf{{P}}}
\newcommand\RR{\mathbf{{R}}}
\renewcommand\SS{\mathbf{{S}}}
\newcommand\TT{\mathbf{{T}}}
\newcommand\UU{\mathbf{{U}}}
\newcommand\WW{\mathbf{{W}}}
\newcommand\VV{\mathbf{{V}}}
\newcommand\XX{\mathbf{{X}}}
\newcommand\YY{\mathbf{{Y}}}

\newcommand{\N}{\mathbb{N}}

\newcommand\MMtil{\mathbf{{\tilde{M}}}}
\newcommand\AAtil{\mathbf{{\tilde{A}}}}
\newcommand\BBtil{\mathbf{{\tilde{B}}}}
\newcommand\LLtil{\mathbf{{\tilde{L}}}}
\newcommand\MMtilde{\mathbf{{\tilde{M}}}}
\newcommand\XXtil{\mathbf{{\tilde{X}}}}

\newcommand\AAn{\mathbf{\mathcal{A}}}
\newcommand\ZZ{\mathbf{{Z}}}

\newcommand\AAhat{\mathbf{\widehat{{A}}}}
\newcommand\AAapprox{\mathbf{\widetilde{{A}}}}
\newcommand\DDhat{\mathbf{\widehat{{D}}}}
\newcommand\DDapprox{\mathbf{\widetilde{{D}}}}
\newcommand\LLhat{\mathbf{\widehat{{L}}}}
\newcommand\LLapprox{\mathbf{\widetilde{{L}}}}
\newcommand\MMhat{\mathbf{\widehat{{M}}}}
\newcommand\MMapprox{\mathbf{\widetilde{{M}}}}
\newcommand\ZZhat{\mathbf{\widehat{{Z}}}}

\newcommand\DDtil{\mathbf{\widetilde{{D}}}}

\newcommand\fftil{\mathbf{\tilde{{f}}}}
\newcommand{\cctil}{\tilde{\cc}}
\newcommand\sstil{\mathbf{\tilde{{s}}}}
\newcommand\xxtil{\mathbf{\tilde{{x}}}}
\newcommand\yytil{\mathbf{\tilde{{y}}}}
\newcommand\wwtil{\mathbf{\tilde{{w}}}}

\newcommand\Otil{\widetilde{O}}

\newcommand\xhat{{\hat{{x}}}}
\newcommand\uhat{{\hat{{u}}}}
\newcommand\uuhat{\mathbf{{\hat{u}}}}
\newcommand\vhat{{\hat{{v}}}}
\newcommand\what{{\hat{{w}}}}
\newcommand\that{{\hat{{\theta}}}}
\newcommand\ttilde{{\tilde{{\theta}}}}

\newcommand\Ghat{{\widehat{{G}}}}
\newcommand\GGhat{\mathbf{\widehat{G}}}

\newcommand\R{\mathbb{R}}

\newcommand\ffhat{\mathbf{\hat{{f}}}}

\newcommand\cchat{\mathbf{\widehat{{c}}}}
\newcommand\sshat{\mathbf{{\widehat{s}}}}
\newcommand\xxhat{\mathbf{{\widehat{x}}}}
\newcommand\yyhat{\mathbf{\widehat{{y}}}}
\newcommand\xxbar{\overline{\mathbf{{x}}}}
\newcommand\yybar{\overline{\mathbf{{y}}}}
\newcommand\xxstar{{\mathbf{{x}}^{*}}}
\newcommand\yystar{{\mathbf{{y}}^{*}}}
\newcommand\tstar{{\theta_\star}}
\newcommand{\ttimes}{\theta^\times}
\newcommand{\tbar}{\bar{\theta}}

\newcommand\ffbar{\overline{\mathbf{{f}}}}

\newcommand\energy{\mathcal{E}}

\newcommand{\prodt}{\prod_{s=1}^t}
\newcommand{\prodbigt}{\prod_{t=1}^T}
\newcommand{\sumt}{\sum_{s=1}^t}
\newcommand{\sumbigt}{\sum_{t=1}^T}
\newcommand{\setcomprehension}[2]{\left\{#1 \; \middle\vert \; #2 \right\} }
\newcommand{\condE}[2]{\mathbb{E}\left [#1 \; \middle\vert \; #2 \right ] }\newcommand{\expE}[1]{\mathbb{E}\left [#1  \right ] }
\newcommand{\indicator}[1]{\mathbb{I}\left(#1\right)}

\newcommand{\expct}[2]{{}\mathop{\mathbb{E}}_{#1}\left[#2\right]}
\newcommand{\E}{\mathop{{}\mathbb{E}}}
\renewcommand{\P}[1]{\mathop{\mathbb{P}}(#1)}
\newcommand{\Pstar}[1]{\mathbb{P}_\tstar\left(#1\right)}
\newcommand{\Psub}[2]{\mathop{\mathbb{P}_{#1}}(#2)}
\newcommand{\norm}[1]{\lvert\lvert #1 \rvert \rvert}
\newcommand{\abs}[1]{\lvert #1 \rvert}
\newcommand{\inner}[2]{\langle #1,\, #2 \rangle}

\newcommand{\Var}[1]{\mathop{{}Var}\left[#1\right]}
\newcommand{\Ex}[1]{{}\mathop{\mathbb{E}}_{#1}}
\newcommand\tr{\mathrm{Tr}}

\newcommand{\schurto}[2]{\ensuremath{\textsc{Sc}\!\left[#1\right]_{#2}}}

\renewcommand{\sc}[2]{\schurto{#1}{#2}}

\newcommand{\trp}{\top}

\newcommand{\pinv}{+}

\newcommand{\proj}{\PPi}

\newcommand{\lik}{\mathcal{L}}
\newcommand{\sep}{\,\vert\,}
\newcommand{\Prob}{\mathrm{P}}

\newtheorem{theorem}{Theorem}
\newtheorem{lemma}{Lemma}
\newtheorem{proposition}{Proposition}
\newtheorem{corollary}{Corollary}
\newtheorem{conjecture}{Conjecture}

\theoremstyle{plain}
\newtheorem{condition}{Condition}
\newtheorem{definition}{Definition}
\newtheorem{example}{Example}
\newtheorem{assumption}{Assumption}
\newtheorem{remark}{Remark}

\vspace{\spacebeforesections}
\section{Introduction}\label{sec:intro} \looseness -1
\vspace{\spacebeforeparagraphs} \looseness -1 One of the main issues addressed by machine learning and statistics is the estimation of an unknown \emph{model} from noisy observations. For example, in supervised learning, this might concern learning the dependence between an input (covariate) $x$ and a random variable (observation) $y$. In many cases, we are not only interested in an estimate $\that$ of the true model parameter $\tstar$, but instead in a set of plausible values that $\tstar$ could take. Such confidence sets are of tremendous importance in sequential decision-making tasks, where uncertainty is used to drive exploration or risk-aversion needs to be implemented, and covariates are iteratively chosen based on previous observations. This setting includes problems such as bandit optimization, reinforcement learning, or active learning. In the former two, the confidence sets are often used to solve the \emph{exploration-exploitation} dilemma and more generally influence the selection rule \citep{mukherjee2022chernoff}, termination rule \citep{Katz-Samuels2020a}, exploration \citep{Auer2002} and/or risk-aversion \citep{makarova2021riskaverse}.

\begin{figure*}[t]
	\centering
        \begin{subfigure}[b]{0.2\textwidth}
            \centering
            \includegraphics[width=1\textwidth]{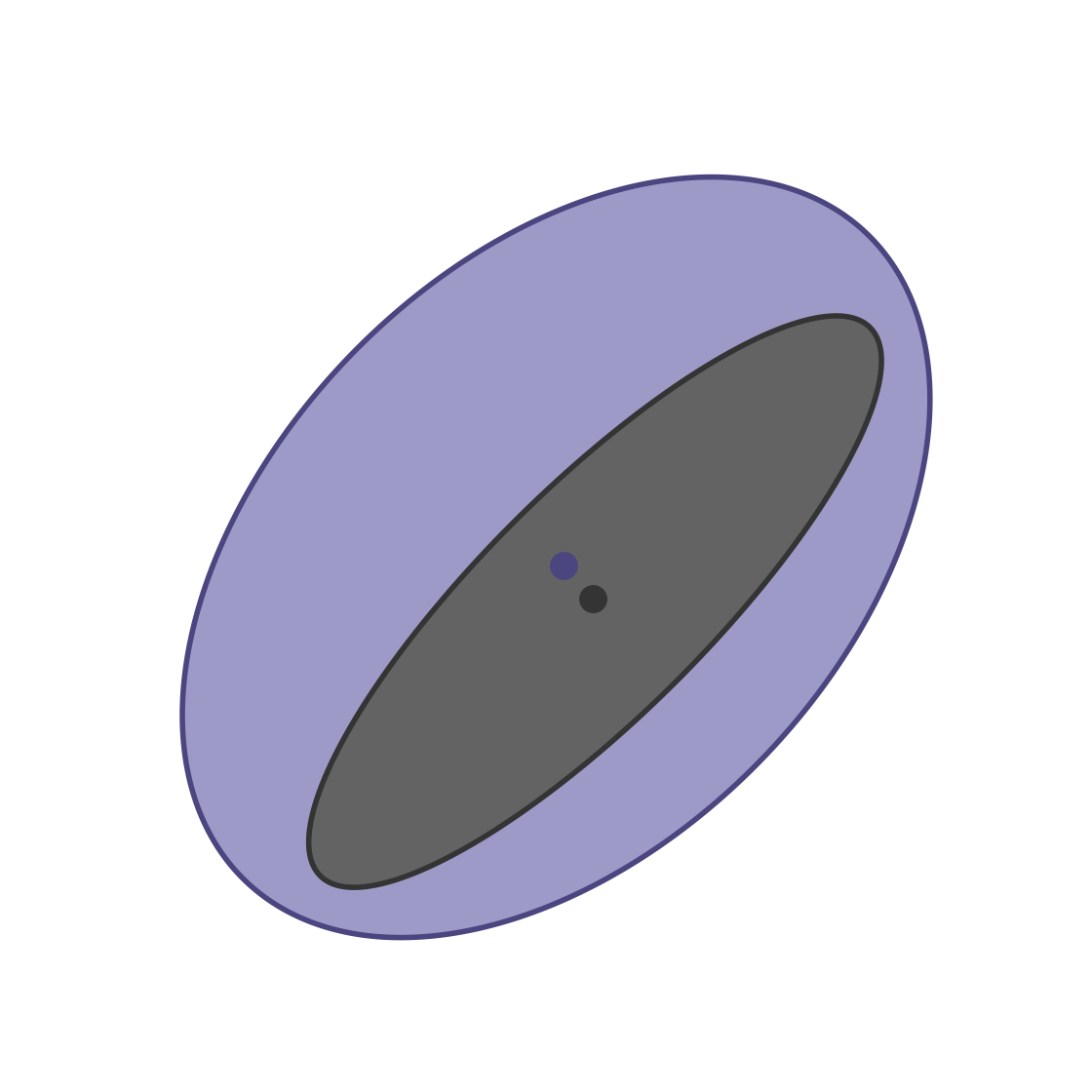}
            \caption{Gaussian $\mathcal{L}$}
        \end{subfigure}
        \begin{subfigure}[b]{0.2\textwidth}
            \centering
            \includegraphics[width=1\textwidth]{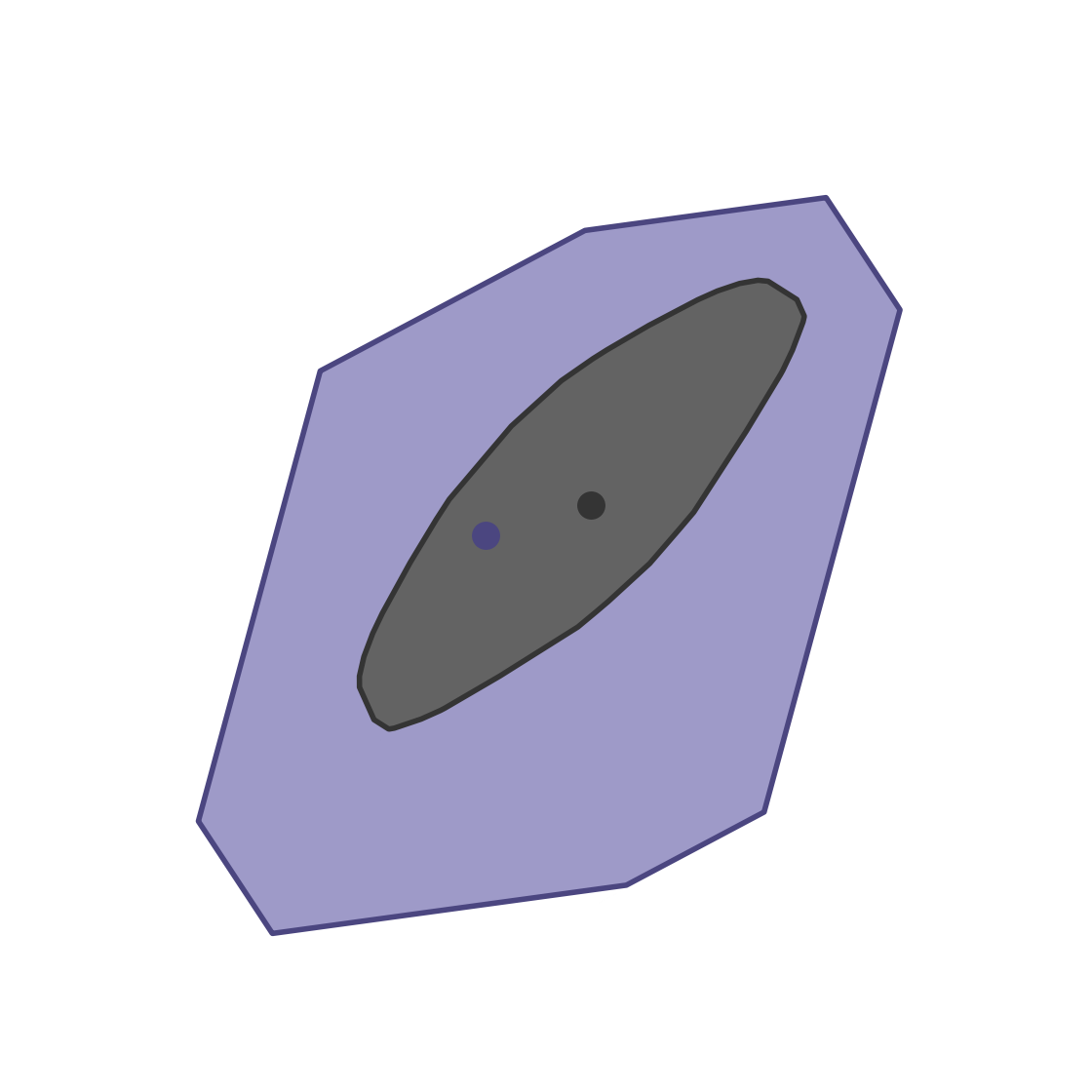}
            \caption{Laplace $\mathcal{L}$}
        \end{subfigure}
        \begin{subfigure}[b]{0.55\textwidth}
            \centering
            \includegraphics[width=\textwidth]{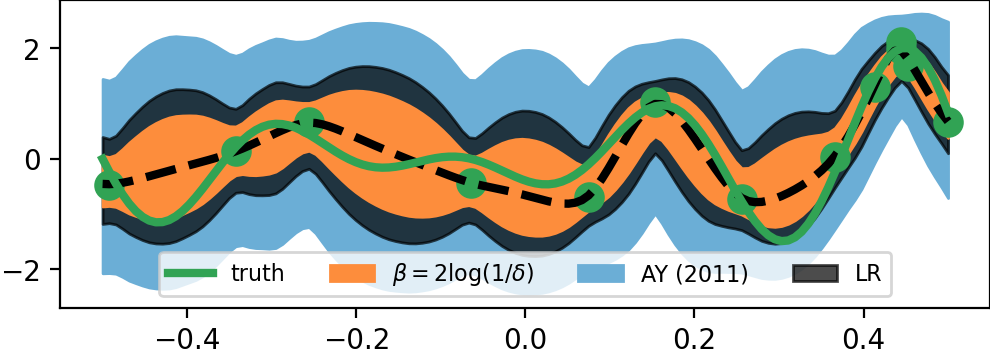}
            \caption{Gaussian $\mathcal{L}$ in RKHS}
        \end{subfigure}
	\caption{(a) and (b) show examples of confidence sets defined via level sets of the log-likelihood function in 2D at two dataset sizes, for Gaussian (a) and Laplace (b) likelihoods respectively. The sets inherit the geometry of the likelihood, and are not always ellipsoidal.  (c) shows confidence bands on an RKHS function in a bandit game searching for the optimum. We compare prior work on confidence sets \citep{abbasi:improved}, our LR sets, and a common heuristic (orange). Our sets are nearly as small as the commonly used heuristic, but have provable coverage and can vastly improve sequential decision-making tasks such as bandits by quickly eliminating hypotheses.} 
 \label{fig:banner}
 \vspace{\spaceafterbanner}
\end{figure*}

When we interact with the environment by gathering data sequentially based on previous confidence sets, we introduce correlations between past noisy observations and future covariates. Data collected in this manner is referred to as {\em adaptively gathered} \citep{ramdas:universal}. Constructing estimators, confidence sets, and hypothesis tests for such non-i.i.d.~data comes with added difficulty. Accordingly, and also for its importance in light of the reproducibility crisis \citep{reproducibilitycrisis}, the task has attracted significant attention in the statistics community in recent years \citep{Ramdas2022c}.

Instead of deriving explicit concentration inequalities around an online estimator, we construct confidence sets {\em implicitly} defined by an inclusion criterion that is easy to evaluate in a computationally efficient manner and requires little statistical knowledge to implement. Roughly speaking, given a model $p_\theta(y \sep x)$ that describes the conditional dependence of the observation $y$ given the covariate $x$ under parameter $\theta$, we will build sets based on a \emph{weighted} modification of the sequential likelihood ratio statistic \citep{Robbins1972, ramdas:universal}
\begin{equation} \label{eq:lrp}
    R_t(\theta) := \frac{\lik_t(\{\that_s\}_{s=1}^t)}{\lik_t(\theta)} := \frac{\prodt p^{w_s}_{\that_{s}}(y_s  \sep x_s)}{\prod_{i=1}^t p^{w_s}_{\theta}(y_s \sep x_s)},
\end{equation}
\looseness -1 where $\{\that_s\}_s$ is a running estimator sequence that we are free to choose, but which may only depend on \emph{previously} collected data.  Parameters $\theta$ for which this statistic is small, i.e., for which $R_t(\theta) \leq 1/\alpha$ will be included in the set (and considered {\em plausible}). Examples of sets in a parametric and non-parametric setting are shown in Figure~\ref{fig:banner}. The weighting terms $w_s \in (0,1]$ are crucial for dealing with inherent irregularities of many conditional observation models but can be flexibly chosen. Classically, these are set to $w_s=1$. The full exposition of our method with choice of estimators and weights is given in Section~\ref{sec:methodology}. Apart from being easy to use and implement, our approach also comes with performance guarantees. These sets maintain a provable $1-\alpha$ coverage -- a fact we establish using Ville's inequality for supermartingales \citep{ville:etudecritique}, which is known to be essentially tight for martingales \citep[see][for a discussion]{howard:timeuniform}. Therefore, in stark contrast to alternate methods, our confidence sequence is {\em fully data-dependent}, making it empirically tighter than competing approaches. Despite the rich history of sequential testing and related confidence sets going back to \citet{wald:sequentiallrt} and \citet{Robbins1972}, these sets have found little use in the interactive machine learning community, which is a gap we fill in the present paper.

\vspace{\spacebeforeparagraphs}
\paragraph{Contributions} \looseness -1
In this work, we revisit the idea of using likelihood ratios to generate anytime-valid confidence sets. The main insight is that whenever the likelihood of the noise process is known, the likelihood ratio confidence sets are fully specified. They inherit their geometry from the likelihood function, and their size depends on the quality of our estimator sequence. We critically evaluate the likelihood ratio confidence sets and, in particular, we shed light on the following aspects:
\textbf{Firstly}, for generalized linear models, we \emph{theoretically} analyze the geometry of the LR confidence sets under mild assumptions. We show their geometry is dictated by Bregman divergences of exponential families \citep{chowdhury2022bregman}.
\textbf{Secondly}, we show that the size of the confidence set is dictated by an online prediction game. The size of these sets depends on a sequence of estimators $\{\that_s\}_{s=1}^t$ that one uses to estimate the unknown parameter $\tstar$. We discuss how to pick the estimator sequence in order to yield a provably small radius of the sets, by using the Follow-the-Regularized-Leader algorithm, which implements a regularized maximum-likelihood estimator. We prove that the radius of the confidence sets is nearly-worst-case optimal, and accordingly, they yield nearly-worst-case regret bounds when used in generalized linear bandit applications. However, due to their data-dependent nature, they can be much tighter than this theory suggests.
\textbf{Thirdly}, we analyze the limitations of classical (unweighted) LR sets when the underlying conditional observation model is not identifiable. In this case, the resulting (inevitable) estimation bias unnecessarily increases the size of the confidence sets. To mitigate this, we propose an adaptive reweighting scheme that decreases the effect of uninformed early bias of the estimator sequence on the size of the sets downstream. The reweighting does not affect the coverage guarantees of our sets and utilizes an elegant connection to (robust) powered likelihoods \citep{ramdas:universal}. \textbf{Finally}, thanks to the adaptive reweighting scheme, our sets are very practical as we showcase experimentally.
We demonstrate that our method works well with exponential and non-exponential family likelihoods, and in parametric as well as in kernelized settings. We attribute their practical benefits to the fact that they \emph{do not depend on (possibly loose) worst-case parameters}.
\vspace{\spacebeforesections}
\section{The Likelihood Method}\label{sec:methodology}  \looseness -1
\vspace{\spacebeforeparagraphs}The sequential likelihood ratio process (LRP) in \eqref{eq:lrp} is a statistic that compares the likelihood of a given model parameter, with the performance of an adaptively chosen estimator sequence. As noted above, we generalize the traditional definition, which would have $w_s \equiv 1$, and define a corresponding confidence set as
\begin{equation}\label{eq:confidencesetdefinition}
    \calC_t = \setcomprehension{\theta}{R_t(\theta) \leq 1/\alpha}.
\end{equation}
The rationale is that the better a parameter $\theta$ is at explaining the data $\{(x_s, y_s)\}_{s}^t$ from the true model $\tstar$, the smaller this statistic will be, thereby increasing its chances to be included in $\calC_t$. When we construct $R_t$, the sequence of $x_s$, $w_s$ and $\that_s$ cannot depend on the noisy observation $y_s$. Formally, consider the filtration $(\calF_s)_{s=0}^\infty$ with sub-$\sigma$-algebras $\calF_s = \sigma(x_1,\ldots, y_1, \ldots x_{s}, y_{s}, x_{s+1})$. We require that $\that_s$ and $w_s$ are $\calF_{s-1}$-measurable. Under these very mild assumptions and with arbitrary weights $w_s \in (0,1]$, we can show coverage, i.e., our (weighted) confidence sets uniformly track the true parameter with probability $1-\alpha$.
\begin{theorem}\label{thm:coverage} The stochastic process 
	$R_t(\tstar)$ in \eqref{eq:lrp} is a non-negative supermartingale with respect to the filtration $(\calF_t)$ and satisfies $R_0(\tstar) \equiv 1$.
    In addition, the sequence $\calC_t$ from \eqref{eq:confidencesetdefinition} satisfies
	$
	\Pstar{\exists t \, : \, \tstar \not \in \mathcal{C}_t} \leq \alpha.
	$
\end{theorem}
\looseness -1 The last statement follows by applying  Ville's inequality for super-martingales on $R_t(\tstar)$. The proof closely follows \citet{ramdas:universal}. While coverage is always guaranteed irrespective of the estimator sequence $\{\that_s\}$, we would like to make the sets as small as possible at fixed coverage, which we do by picking a well-predicting estimator sequence. 
\vspace{\spacebeforesubsections}
\subsection{The Estimator Sequence Game} 
\vspace{\spacebeforeparagraphs} The specification of the LR process (LRP) allows us to choose an arbitrary estimator sequence $\{\that_s\}_s$. To understand the importance of the sequence, let us introduce $\tstar$ to the definition of $R_t$ in \eqref{eq:lrp}, and divide by $\lik_t(\{\that_s\}_{s=1}^t)$. This gives the equivalent formulation
\[ \mathcal{C}_t := \left\{ \theta ~\Big|~ \frac{\lik_t(\tstar)}{\lik_t(\theta)} \leq \frac{1}{\alpha} \frac{\lik_t(\tstar)}{\lik_t(\{\that_s\}_{s=1}^t)} \leftarrow \text{confidence parameter} \right\}.\]
We see that the predictor sequence does not influence the geometry of the confidence set, which is fully specified by the likelihood function. We also observe that the ratio on the right-hand side serves as a confidence parameter controlling the size (radius) of the confidence sets measured under the likelihood ratio distance to $\tstar$. If the confidence parameter goes to zero, only $\tstar$ is in the set. The better the estimator sequence is at predicting the data, the smaller the inclusion threshold, and hence the smaller the sets will ultimately be. Specifically, taking the $\log$, we would like to \emph{minimize}
\begin{align}
\label{eq:see-regret}
    \mathcal{R}_t:=\log \frac{\lik_t(\tstar)}{\lik_t(\{\that_s\}_{s=1}^t)} =  \sum_{s=1}^t {-\log(p_{\that_{s}}^{w_s}(y_s  \sep x_s))} - \sum_{s=1}^t {-\log(p_{\tstar}^{w_s}(y_s  \sep x_s))}. %
\end{align}
The quantity $\mathcal{R}_t$ corresponds to a regret in an online prediction game, as will become apparent below.

\vspace{\spacebeforeparagraphs}
\paragraph{Online Prediction Game} \looseness-1
Online optimization is a mature field in interactive learning  \citep{Cesa-Bianchi2006, orabona:book}. The general goal is to minimize a sequence of loss functions as in Eq.~\eqref{eq:see-regret} and compete against a baseline, which typically is the best-in-hindsight prediction, or -- in our case -- given by the performance of the fixed parameter $\tstar$. Specifically, at every timestep $s$, iteratively, the agent  chooses an action $\that_s$ based on $\calF_{s-1}$, and a loss function $f_s(\theta)$ is revealed. In most of the online optimization literature, $f_s$ can be chosen adversarially. In our prediction game, we know the whole form of loss function $f_s(\theta) = -\log(p^{w_s}_{\theta}(y_s \sep x_s))$, as can be seen in \eqref{eq:see-regret}, and not just $f_s(\that_s)$. Opposed to traditional assumptions in online prediction, in our case, $f_s$ are non-adversarial, but have a stochastic component due to $y_s$. Also, contrary to most instances of online prediction, we do not compare against the best-in-hindsight predictor, but $\tstar$ instead, as this is more meaningful in our setting. 

\vspace{\spacebeforeparagraphs} \paragraph{Online Optimization Algorithms} \looseness -1 
Generally, we seek an algorithm that incurs low regret. Here, we focus on {\em Follow-the-Regularized Leader (FTRL)}, which corresponds exactly to using regularized maximum likelihood estimation, making it a natural and computationally practical choice. The update rule is defined in Alg.~\ref{alg:online} (Line 3). While other algorithms could be considered, FTRL enjoys the optimal regret rate for generalized linear regression as we show later, and is easily implemented. In order to run the algorithm, one requires a sequence of strongly convex regularizers. For now, let us think of it as $\psi_s(\theta) =\lambda ||\theta||_2^2$, which we use in practice. However, one can derive a tighter analysis for a slightly modified, time-dependent regularization strategy for generalized linear models as we show in Sec.~\ref{sec:online-regret}. 
\vspace{\spacebeforesubsections}
\subsection{Adaptive Reweighting: Choosing the Right Loss} \vspace{\spacebeforeparagraphs}
 \label{sec:bias}
There is yet more freedom in the construction of the LR, via the selection of the {\em loss function}. Not only do we select the predictor sequence, but also the {\em weights} of the losses via $w_t$. This idea allows controlling the influence of a particular data point $(x_t,y_t)$ on the cumulative loss based on the value of $x_t$. For example, if we know a priori that for a given $x_t$ our prediction will be most likely bad, we can opt out of using the pair $(x_t,y_t)$ by setting $w_t = 0$. 
Below we will propose a weighting scheme that depends on a notion of {\em bias}, which captures how much of the error in predicting $y_t$ is due to our uncertainty about $\that_t$ (compared to the uncertainty we still would have {\em knowing} $\tstar$). Sometimes this \emph{bias} is referred to as \emph{epistemic} uncertainty in the literature, while the residual part of the error is referred to as \emph{aleatoric}.  
Putting large weight on a data point heavily affected by this bias
might unnecessarily increase the regret of our learner (and hence blow up the size of the confidence set). Note that, conveniently, even if we put low weight (zero) on a data point, nothing stops us from using this sample point to improve the estimator sequence in the next prediction round. As we will show below, our reweighting scheme is crucial in defining a practical algorithm for Reproducing Kernel Hilbert Space (RKHS) models and in high signal-to-noise ratio scenarios. Since we do not know $\tstar$, our strategy is to compute an \emph{estimate of the bias of the estimator $\that_t$} and its effect on the value of the likelihood function for a specific $x$ that we played. We use the value of the bias to rescale the loss via $w_t$ such that its effect is of the same magnitude as the statistical error (see Algorithm~\ref{alg:online}; we call this step \textsc{bias-weighting}). 
\vspace{\spacebeforeparagraphs}
\paragraph{Intuition} To give a bit more intuition, suppose we have a Gaussian likelihood. Then the negative log-likelihood of $(x_t,y_t)$ with weighting is proportional to $\frac{w_t}{\sigma^2}(y_t - x_t^\top \that_t)^2$. Now, if $x_t$ does not lie in the span of the data points $\{x_s\}_{s=1}^{t-1}$ used to compute $\that_t$, it is in general unavoidable to incur large error, inversely proportional to $\sigma^2$. To see this, let us decompose
the projection onto $x_t$ as\begin{equation}\label{eq:bias-first-mention}
x_t^\top({\that_t - \tstar}) = \underbrace{x_t^\top(\that_t - \mathbb{E}[\that_t])}_{\text{statistical error}} + \underbrace{x_t^\top(\mathbb{E}[\that_t] - \tstar)}_{\operatorname{bias}_{x_t}(\that_t)},
\end{equation} \looseness -1
where the first term represents the statistical error up to time $t$, while the second, bias, is deterministic, and independent of the actual realization $y$, depending only $\tstar$. Estimators with non-zero bias are \emph{biased}. 
Plugging this into the likelihood function, we see that in expectation $\frac{1}{\sigma^2} \mathbb{E}[(y_t - x_t^\top \that_t)^2|\mathcal{F}_{t-1}] \lesssim
 \frac{1}{\sigma^2}{\operatorname{bias}^2_{x_t}(\that_t)} + \epsilon^2 + \frac{C}{t}$, where $\epsilon^2$ is the unavoidable predictive error in expectation (due to a noisy objective) and is a constant independent of $\sigma^2$. $\frac{C}{t}$ is the statistical error, and $C$ is independent of $\sigma^2$. Note that the bias term scales inversely with the variance, and leads to unnecessarily big confidence parameters for small $\sigma^2$. 

\looseness -1 In fact, the problem is that we use the likelihood to measure the distance between two parameters, but this is only a ``good`` distance once the deterministic source of the error (bias) vanishes. For this reason, without weighting, the incurred regret blows up severely in low-noise settings. To counter this, we balance the deterministic estimation bias and noise variance via proper selection of $w_t$. In this case, it turns out that $w_t = \frac{\sigma^2}{\sigma^2 + \operatorname{bias}^2_{x_t}(\that_t) }$ ensures that the overall the scaling is independent of $\sigma^2$. While the choice of weights $\{w_s\}_s^t$ influences the geometry of the confidence sets, with a \emph{good data collection and estimation strategy} the bias asymptotically decreases to zero, and hence the weights converge to $1$. 
\paragraph{Bias estimation} \vspace{\spacebeforeparagraphs}
In order to generalize this rule beyond Gaussian likelihoods, we need a proper generalization of the bias. Our generalization is motivated by our analysis of generalized linear models, but the method can be applied more broadly. The role of the squared statistical error (variance) is played by the inverse of the smoothness constant of the negative log-likelihood functions $f_s$, denoted by $L$. This is the usual smoothness, commonly seen in the convex optimization literature. We consider penalized likelihood estimators with strongly convex regularizers (Alg.~\ref{alg:online}, line 3). For this estimator class, we define the bias via a hypothetical stochastic-error-free estimate $\that_t^\times$, had we access to the expected values of the gradient loss functions (a.k.a. score). We use the first-order optimality conditions and the indicator function of the set $\Theta$, $i_{\Theta}$, to define the error-free-estimate $\that_t^\times$, and the bias of the estimator $\that_t$ as 
\begin{equation}\label{eq:bias-calculation}
    \operatorname{bias}^2_{x_t}(\that_t) = (x_t^\top(\tstar - \that_t^\times))^2 \quad \text{with} \quad \mE\left[\sum_{s=1}^{t-1} \nabla \log p_{\that_t^\times}(y_s|x_s)\right] - \nabla \psi_t(\that_t^\times) + i_{\Theta}(\that^\times_t)= 0,
\end{equation}
where the expectation denotes a sequence of expectations conditioned on the prior filtration.
This notion of bias coincides with the definition of bias in Eq.~\eqref{eq:bias-first-mention} for the Gaussian likelihood. This quantity cannot be evaluated in general, however, we prove a computable upper bound.
\begin{theorem}[Bias estimate]\label{thm:bias}
Let the negative log-likelihood have the form, $-\log p_\theta (y_s|x_s) = g(x_s^\top \theta)$, where $g:\mathbb{R}\rightarrow \mathbb{R}$ is $\mu$ strongly-convex and let the regularizer be $\psi_t(\theta)=\lambda \norm{\theta}_2^2$ making the overall objective strongly convex. Then, defining $\VV_t^{\mu;\lambda} = \sum_{s=1}^t \mu x_s x_s^\trp + \lambda \II$, we can bound
\begin{equation}\label{eq:bias-approximation}
 \operatorname{bias}^2_{x}(\that_t) \leq 
 2\lambda\norm{\tstar}^2_2 x^\top(\VV_t^{\mu;\lambda})^{-1}x.
\end{equation}
\end{theorem}
\vspace{\spacebeforeparagraphs}
\looseness -1 The proof is deferred to App.~\ref{app:bias-proof}, and requires elementary convex analysis. 

This leads us to propose the weighting scheme $w_t = \frac{1/L}{\operatorname{bias}^2_{x_t}(\that_t)+ 1/L}$. We justify that this is a sensible choice by analyzing the confidence set on the GLM class in Section \ref{sec:theory}, which satisfies the smoothness and strong-convexity conditions. We show that this rule properly balances the stochastic and bias components of the error in the regret as in \eqref{eq:see-regret}. However, this rule is more broadly applicable beyond the canonical representation of GLM or the GLM family altogether.
\begin{algorithm}
    \begin{algorithmic}[1]
    \caption{Constructing the LR Confidence Sequence}
        \label{alg:online}
    \State \textbf{Input:} convex set $\Theta \subset \R^d$, confidence level $\alpha > 0$, likelihood $p_\theta(y|x)$, regularizers $\{\psi_t\}_t$
    \For{$t \in \N_0$}    
        \State $\that_t = 
        \arg\min_{\theta \in \Theta} \sum_{s=1}^{t-1} -\log p_\theta(y_s\sep x_s) + \psi_t(\theta)$ \Comment{FTRL}
        \State $w_t= \begin{cases} \frac{1/L} {1/L+\operatorname{bias}^2_{x_t}(\that_t)}& \textsc{this work} \\ 1 & \textsc{classical} \\
        \end{cases}$ \Comment{\textsc{bias-weighting}~$\operatorname{bias}_{x_t}(\that_t)$ in Eq.~\eqref{eq:bias-calculation} or Eq.\eqref{eq:bias-approximation}} 
        \State $\calC_t = \setcomprehension{\theta \in \Theta}{\prodt\frac{p^{w_s}_{\that_{s}}(y_s\sep x_s)} {p^{w_s}_\theta(y_s \sep x_s)} \leq \frac{1}{\alpha}}.$ \Comment{Confidence set}
\EndFor
\end{algorithmic}
\end{algorithm} %
\vspace{\spaceafterbanner}
\vspace{0.1cm}
\vspace{\spacebeforesections}
\section{Theory: Linear Models}\label{sec:theory} \looseness -1
\vspace{\spacebeforeparagraphs} While the {\em coverage} (i.e., ``correctness'') of the likelihood ratio confidence sets is always guaranteed, their worst-case {\em size} (affecting the ``performance'') cannot be easily bounded in general. We analyze the size and the geometry of the LR confidence sequence in the special but versatile case of generalized linear models.
\vspace{\spacebeforesubsections}
\subsection{Generalized Linear Models} \looseness -1 
\vspace{\spacebeforeparagraphs}
We assume knowledge of the conditional probability model $p_\theta(y|
x)$, where the covariates $x \in \mathcal{X} \subset \R^d$, and the true underlying model parameter lies in a set $\Theta \subset \R^d$.  If $t$ is indexing (discrete) time, then $x_t$ is acquired sequentially, and the -- subsequently observed --  $y_t$ is sampled from an exponential family distribution parametrized as
\begin{equation} \label{eq:glmmodel}
    p_\theta(y\sep x_t) = h(y)\exp\left(T(y)\cdot x_t^\trp\theta -A(x_t^\trp \theta) \right).
\end{equation}
Here, $A$ is referred to as the \emph{log-partition function} of the conditional distribution, and $T(y)$ is the sufficient statistic. The function $h$ is the base measure, and has little effect on our further developments, as it cancels out in the LR. %
Examples of commonly used exponential families (Gaussian, Binomial, Poisson or Weibull) with their link functions can be found in Table \ref{tb:examples} in App.~\ref{app:glm-table}.

In order to facilitate theoretical analysis for online algorithms, we make the following assumptions about the covariates $x \in \mathcal{X}$ and the set of plausible parameters $\Theta$.
\begin{assumption} \label{ass:bounded} The covariates are bounded, i.e.,  $\sup_{x\in \mathcal{X}}\norm{x}_2 \leq 1$, and the set $\Theta$ is contained in an $\ell_2$-ball of radius $B$.
We will also assume that the log-partition function is strongly convex, that is, that there exists $\mu := \inf_{z \in [-B, B]} A''(z)$, 
and that  $A$ is L-smooth, i.e. $L:= \sup_{z \in [-B, B]} A''(z)$. 
\end{assumption} \vspace{\spaceaftertheorem}
These assumptions are common in other works addressing the confidence sets of GLMs \citep{filippi:glm, faury:logistic}, who remark that the dependence on $\mu$ is undesirable. However, in contrast to these works, our confidence sets {\em do not} use these assumptions in the construction of the sets. We only require these for our theoretical analysis. As these are worst-case parameters, the practical performance can be much better for our sets. 
\vspace{\spacebeforesubsections}      
\subsection{Geometry and Concentration} \looseness -1 \vspace{\spacebeforeparagraphs}
Before stating our results, we need to define a distance notion that the convex negative log-likelihoods induce. For a continuously differentiable convex function $f$, we denote the Bregman divergence as
$   D_f(a,b) := f(a) - f(b) - \nabla f(b)^\trp (a - b).$ The $\nu$-regularized sum of log-partition functions is defined as
\begin{equation}\label{eq:log-parition-definition}
    Z_t^\nu(\theta) := \sum_{s=1}^t w_s A(x_s^\trp \theta) + \frac{\nu}{2}\norm{\theta}^2_2.
\end{equation}\looseness -1
This function will capture the geometry of the LR confidence sets. The confidence set size depends mainly on two terms. One refers to a notion of complexity of the space referred to as \emph{Bregman information gain}:
$ \Gamma_t^\nu(\ttilde_t) = \log\left(\frac{\int_{\R^d} \exp(-\frac{\nu}{2}\norm{\theta}^2_2)\mathrm{d}\theta }{\int_{\R^d} \exp(-D_{Z_t^\nu}(\theta, \ttilde_t))\mathrm{d}\theta}\right)$,
first defined by \citet{chowdhury2022bregman} as a  generalization of the \emph{information gain} of \citet{srinivas:noregret}, $\gamma_t^\nu = \log\left( 
{\det(\sum_{i=1}\frac{\mu}{\nu}x_i x_i^\top + \II )}  \right)$ for Gaussian likelihoods. We will drop the superscript whenever the regularization is clear from context and simply refer to $\gamma_t$. This term appears because one can relate the decay of the likelihood as a function of the Bregman Divergence from $\tstar$ with the performance of a (regularized) maximum likelihood estimator via convex (Fenchel) duality. In particular, if $\ttilde_t$ is a regularized MLE, $\Gamma_t^\nu := \Gamma_t^\nu(\ttilde_t)$ will asymptotically scale as $\mathcal{O}(d \log t)$ \citep[cf.][for further discussion]{chowdhury2022bregman}. For Gaussian likelihoods and $w_s \equiv 1$, it coincides with the classical information gain independent of $\ttilde_t$. The second term that affects the size is the regret $\mathcal{R}_t$ of the online prediction game over $t$ rounds we introduced previously in \eqref{eq:see-regret}. These two parts together yield the following result:
\begin{theorem}
    \label{theorem:bregmanball}
    Let $\nu > 0$ and $\alpha,\delta \in (0,1)$. For the level $1-\alpha$ confidence set $\calC_t$ defined in \eqref{eq:confidencesetdefinition} under the GLM in \eqref{eq:glmmodel}, with probability $1-\delta$, for all $t\geq 1$, any $\theta \in \calC_t$ satisfies
    \begin{align} \label{eq:bregmanball}
        &D_{Z_t^\nu}(\theta, \tstar) \leq \frac{4L}{\mu} \xi_t + 2\log\left(\frac{1}{\delta}\right) + 2\mathcal{R}_t, \end{align}
    where 
    $ \xi_t = \left(\log\left(\frac{1}{\alpha}\right) + {\nu B^2} + \Gamma_t^\nu \right) $
    and $L,\mu$ are defined as above and finally $\mathcal{R}_t$ is the regret of the game in Eq. \eqref{eq:see-regret}.
\end{theorem}
The set defined via the above divergence does not coincide with the LR confidence set. It is slightly larger due to a term involving $\nu$ (as in Eq. \eqref{eq:log-parition-definition}).  This is a technical consequence of our proof technique, where the gradient of $Z_{t}^\nu$ needs to be invertible, and regularization is added to this end.
We note that this $\nu>0$ can be chosen freely. Note that the theorem involves two confidence levels,  $\alpha$ and $\delta$: $\alpha$ is a bound on the Type I error -- coverage of the confidence sets -- while $\delta$ upper bounds the probability of a large radius -- and is therefore related to the power and Type II error of a corresponding hypothesis test.  
The proof of the theorem is deferred to App.~\ref{app:proof-bregman}. 

To give more intuition on these quantities, let us instantiate them for the Gaussian likelihood case with $w_s \equiv 1$. In this scenario, $Z_{t}^\nu(\theta) = \sum_{s=1}^t \frac{1}{2\sigma^2}\norm{\theta}_{x_s x_s^\top}^2 + \frac{\nu}{2} \norm{\theta}_2^2$, and the (in this case symmetric) Bregman divergence is equal to $D_{Z_t^\nu}(\tstar, \theta) = \frac{1}{2} \norm{\theta - \tstar}^2_{\mathbf{V}_t^{\sigma^{-2};\nu}}$, where $\VV_t^{\mu;\nu} =\sum_{s=1}^t{\mu x_s x_s^\top + \nu\mathbf{I}}$, which means that our confidence sets are upper bounded by a ball in the same norm as those in the seminal work on linear bandits \citep{abbasi:improved}.
\vspace{\spacebeforesubsections}
\subsection{Online Optimization in GLMs: Follow the Regularized Leader} 
\vspace{\spacebeforeparagraphs}
\label{sec:online-regret}
The size of the confidence sets in Theorem~\ref{theorem:bregmanball} depends on the regret of the online prediction game involving the estimator sequence. We now bound this regret when using the Follow-the-Regularized-Leader (FTRL) algorithm in this setting. This high probability bound is novel to the best of our knowledge and may be of independent interest. We state in a weight-agnostic manner first, and then with our particular choice. The latter variant uses a specifically chosen regularizer. In this case, we can track the contribution of each time-step towards the regret separately.
\begin{theorem}\label{thm:ftrl-main} Let $\psi_t(\theta) = \lambda \norm{\theta}_2^2$. Assume Assumption \ref{ass:bounded}, and additionally that $A$ is $L$-smooth everywhere in $\mathbb{R}^d$, and let $w_t \in [0,1]$ be arbitrary. Then, with probability $1-\delta$ the regret of FTRL (Alg. \ref{alg:online}) satisfies for all $t\geq 1$
\begin{equation}
    \mathcal{R}_t \leq \lambda B^2 + \frac{L}{\mu} (\gamma_t^\lambda + 2\log({1}/{\delta})) + \frac{2 L^2B^2}{\mu}\gamma_t^\lambda. \label{eq:ftrlunweighedeq}
\end{equation}
\end{theorem} \vspace{\spaceaftertheorem} 
The regret bounds are optimal in the orders of $\gamma_t^\lambda$, matching lower bounds of \citet{ouhamma:onlinelr}, as for linear models $\gamma_t = \mathcal{O}(d \log t)$. Combining results of Thm.~\ref{thm:ftrl-main} with Thm.~\ref{theorem:bregmanball}, we get a confidence parameter that scales with $\calO(\sqrt{\gamma_t})$, for confidence sets of the form $||\theta - \tstar||_{\VV_t}$, which coincides with the best-known confidence sets in this setting in the worst-case \citep{Abbasi2012-online}. The requirement of global $L-$smoothness can be relaxed to $L-$smoothness over $\Theta$. With a more elaborate (but less insightful) analysis, we can show that we achieve a $\Tilde\calO(\gamma_t)$ bound even in this case. The proofs of these results are deferred to App.~\ref{app:proof-unweighted}, App.~\ref{app:proof-ftrl-azoury} and App.~\ref{app:beyond-smooth} respectively.
\vspace{\spacebeforeparagraphs}
\paragraph{Regret, Weighting and Estimation Bias}\looseness -1  Interestingly, the term in Thm.~\ref{thm:ftrl-main} involving the (crude) proxy to the bias -- the bound $B$ -- is not scaled by the same  $L/\mu$ factors as the other terms in the regret bound~\eqref{eq:ftrlunweighedeq} and in Theorem~\ref{theorem:bregmanball}. Namely, the prefactor is $L^2/\mu$ instead of $L/\mu$. This extra dependence manifests itself in the unnecessary penalization through the estimation bias we introduced in Sec.~\ref{sec:bias}, particularly in low-noise settings. We addressed this issue by picking the weights $\{w_t\}$. While the above theorem holds for any valid weighting, it does not exhibit the possible improvement from using specific weights.

We argued earlier that the error in prediction should not be measured by the likelihood function if there is deterministic error, since initially, we are fully uncertain about the value of  $\tstar^\top (\cdot)$ outside the span of previous observations. 
Of course, if our goal would be to purely pick weights to minimize $\mathcal{R}_t$, then $w_s = 0$ would lead to zero regret and hence be optimal. However, the likelihood ratio would then be constant, and uninformative. In other words, the associated log-partition Bregman divergence in Theorem~\ref{theorem:bregmanball} would be trivial and not filter out any hypotheses. Clearly, some balance has to be met. 
With this motivation in mind, we proposed a \emph{nonzero} weighting that decreases the regret contribution of the bias, namely $w_t = \frac{1/L}{1/L + \operatorname{bias}^2_{x_t}(\that_t)}$. The advantage of this choice becomes more apparent when we use the regularizer $\psi_t(\theta) = \lambda ||\theta||^2 + A(x_t^\top \theta)$ to obtain the following result.

\begin{theorem}\label{thm:ftrl-main-2} Let $\psi_s(\theta) = \lambda ||\theta||^2 + A(x_s^\top \theta)$. Assume Assumption \ref{ass:bounded}, and additionally that $A$ is $L$-smooth everywhere in $\mathbb{R}^d$, and choose $w_s = \frac{1/L}{1/L + \operatorname{bias}_{x_s}(\that_s)^2}$. Additionally, let the sequence of $x_s$ be such that, $\sum_s (1-w_s) (f_s(\tstar) - f_s(\bar{\theta}_{s+1})) \leq L/\mu \gamma_t^\lambda$, where $\bar\theta_{s}$ is the FTRL optimizer with the regularizer $\lambda \norm{\theta}_2^2$ from Theorem~\ref{thm:ftrl-main}  \footnote{
Note that $\bar\theta_{s+1}$ corresponds to a regularized MLE that \emph{did} observe the data pair $(x_s,y_s)$.}. Then, with probability $1-\delta$ the regret of FTRL (Alg. \ref{alg:online}) satisfies for all $t\geq 1$
\[\mathcal{R}_t \leq \lambda B^2 + \frac{2L}{\mu} \left(\gamma_t^\lambda + \log\left(\frac{1}{\delta}\right)\right) + \frac{L}{\mu}\sum_{s=1}^t \frac{B^2}{1/L + \operatorname{bias}^2_{x_s}(\that_s)} \Delta \gamma_s^\lambda,\]
where  $\Delta \gamma_s^\lambda = \gamma_{s+1}^\lambda - \gamma_s^\lambda$.
\end{theorem} \vspace{\spaceaftertheorem}

One can see that for points where the information gain $\Delta \gamma_s$ is large (corresponding to more unexplored regions of the space, where the deterministic source of error is then large), the weighting scheme will make sure that the multiplicative contribution of $B^2$ is mitigated, along with having the correct prefactor $L/\mu$. The reader may wonder how this result is useful when we replace $\operatorname{bias}^2_{x_s}(\that_s)$ with the upper bound from Thm.~\ref{thm:bias}. While instructive, our bound still only makes the bias proxy $B^2$ appear in front of the information gain $\Delta\gamma_t$, instead of the more desireable bias itself. In the latter case, we could also directly make use of the upper bound and get an explicit result only using an upper bound on the bias. We leave this for future work.

We point out that this choice of $\psi_s(\theta)$ in Theorem~\ref{thm:ftrl-main-2} corresponds to the Vovk-Azoury-Warmuth predictor \citep{vovk:forecaster, azoury:forecaster} in the online learning literature. This choice is helpful in order to track the bias contribution more precisely in our proof.

\vspace{\spacebeforesections}
\section{Application: Linear and Kernelized Bandits}\label{sec:applications} \looseness -1
\vspace{\spacebeforeparagraphs}
\looseness -1 Our main motivation to construct confidence sets is bandit optimization. A prototypical bandit algorithm -- the Upper Confidence Bound (UCB) \citep{Auer2002} -- sequentially chooses covariates $x_s$ in order to maximize the reward $\sum_{s=1}^t r_\tstar(x_s)$, where $r_\tstar$ is the unknown pay-off function parametrized by $\tstar$. UCB chooses the action $x_s$ which maximizes the optimistic estimate of the reward in each round, namely
\begin{equation}\label{eq:ucb}
 x_s = \argmax_{x\in \mathcal{X}} \max_{\theta \in \mathcal{C}_{s-1}} r_\theta(x),\end{equation}
where $\mathcal{C}_{s-1}$ is some confidence set for $\tstar$, and can be constructed with Algorithm~\ref{alg:online} from the first $s-1$ data points. An important special case is when $r_\tstar$ is linear \citep{abe1999associative} or modelled by a generalized linear model \citep{filippi:glm}. In that case, the inner optimization problem is convex as long as $\mathcal{C}_{s-1}$ is convex. The outer optimization is tractable for finite $\mathcal{X}$. In the applications we consider, our confidence sets are convex, and we easily solve the UCB oracle using convex optimization toolboxes. 
\paragraph{Extension to RKHS} \looseness -1 We introduced the framework of LR confidence sets only for finite-dimensional Euclidean spaces. However, it can be easily extended to Reproducing Kernel Hilbert Spaces (RKHS) \citep{Cucker2002}. The definition of the LR process in \eqref{eq:lrp} is still well-posed, but now the sets are subsets of the RKHS, containing functions $f \in \mathcal{H}_k$. An outstanding issue is how to use these sets in downstream applications, and represent them tractably as in Figure~\ref{fig:banner}. Conveniently, even with infinite-dimensional RKHSs, the inner-optimization in \eqref{eq:ucb} admits a Lagrangian formulation, and the generalized representer theorem applies \citep{Schoelkopf2001, mutny:poisson}. In other words, we can still derive a pointwise upper confidence band as $\operatorname{ucb}(x) = \max_{f\in \mathcal{H}_k, \norm{f}_k\leq B, f \in C_{s}} \braket{f,k(x,\cdot)}$ in terms of $\{x_j\}_{j=1}^s \cup \{x\}$, leading to a $s+1$-dimensional, tractable optimization problem. 

We also point out that the weighting is even more paramount in the RKHS setting, as the bias never vanishes for many infinite dimensional Hilbert spaces \citep{Mutny2022b}. For this purpose, our weighting is of paramount practical importance, as we can see in Figure \ref{fig:experiments}a), where the gray arrow represents the significant improvement from reweighting. 

\vspace{\spacebeforesubsections}
\subsection{Instantiation of the Theory for Linear Bandits} \looseness -1
\vspace{\spacebeforeparagraphs}
Before going to the experiments, we instantiate our theoretical results from Sec.~\ref{sec:theory} to the important and well-studied special case of linear payoffs. In that case, $r_\theta(x) = \inner{x}{\theta}$ and the agent observes $y_s = \inner{x_s}{\tstar} + \eta_s$ upon playing action $x_s$, where $\eta_s \sim \mathcal{N}(0, \sigma^2)$. We are interested in minimizing the so-called cumulative pseudo-regret, namely, $ \mathfrak{R}_t = \sum_{s=1}^t [\inner{\xstar}{\tstar} - \inner{x_s}{\tstar}]$, where $x_\star$ refers to the optimal action. Using the set from~\eqref{eq:confidencesetdefinition} along with Theorem~\ref{theorem:bregmanball} and the FTRL result of Theorem~\ref{thm:ftrl-main} we can get a regret bound for the choice $w_s \equiv 1$.
\begin{theorem}\label{theo:linearbandits:vanilla}
    Let $w_s \equiv 1$. For any $\lambda \geq \frac{1}{\sigma^2}$, with probability at least $1-3\delta$, for all $t \in \N$ we have
    \begin{equation*}
        \mathfrak{R}_t \leq 6\sqrt{t\gamma_t^\lambda}\left(\sigma\sqrt{\log(1/\delta) + \gamma_t^\lambda} + \sigma\lambda^{1/2} B +B\sqrt{\gamma_t^\lambda}\right).
    \end{equation*}
\end{theorem}
Our results are optimal in both $d$ and $t$ up to constant and logarithmic factors. The proof is deferred to App.~\ref{app:linearbandits}, but is an instantiation of the aforementioned theorems, along with a standard analysis. There, we also compare to the seminal result of \cite{abbasi:improved}, which does not suffer from the dependence on $B\sqrt{\gamma_t}$. We attribute this to the incurred bias in the absence of the reweighting scheme.

\looseness -1 For the weighted likelihood ratio, we can obtain a result similar to the above, but multiplied by an upper bound on $\sup_{s\geq 1} w_s^{-1}$. This is undesirable, as our experiments will show that the reweighting scheme vastly improves performance. While this could be somewhat mitigated by using the Theorem~\ref{thm:ftrl-main-2} instead of Theorem~\ref{thm:ftrl-main} to bound the FTRL regret, a better result should be achievable using our weighting scheme that improves upon Theorem~\ref{theo:linearbandits:vanilla} and possibly even matches \cite{abbasi:improved} exactly in the worst-case. We leave this for future work.

\vspace{\spacebeforesubsections}
\subsection{Experimental Evaluation}
\vspace{\spacebeforeparagraphs}
\looseness -1 In this subsection, we demonstrate that the practical applicability goes well beyond the Gaussian theoretical result from the previous subsection. In the examples below, we always use the UCB algorithm but employ different confidence sets. In particular, we compare our LR confidence sets for different likelihood families with alternatives from the literature, notably classical sub-family confidence sets \citep{abbasi:improved, mutny:poisson}, and the robust confidence set of \citet{Neiswanger21a}. In practice, however, the radius of these confidence sets is often tuned heuristically. We include such sets as a baseline \emph{without} provable coverage as well. The main take-home message from the experiments is that among all the estimators and confidence sets that enjoy \emph{provable} coverage, our confidence sets perform the best, on par with successful heuristics. For all our numerical experiments in Figure~\ref{fig:experiments}, the true payoff function is assumed to be an infinite dimensional RKHS element. For further details and experiments, please refer to App.~\ref{app:experiments}.
\begin{figure*}
    \centering
	\includegraphics[width = 1\textwidth]{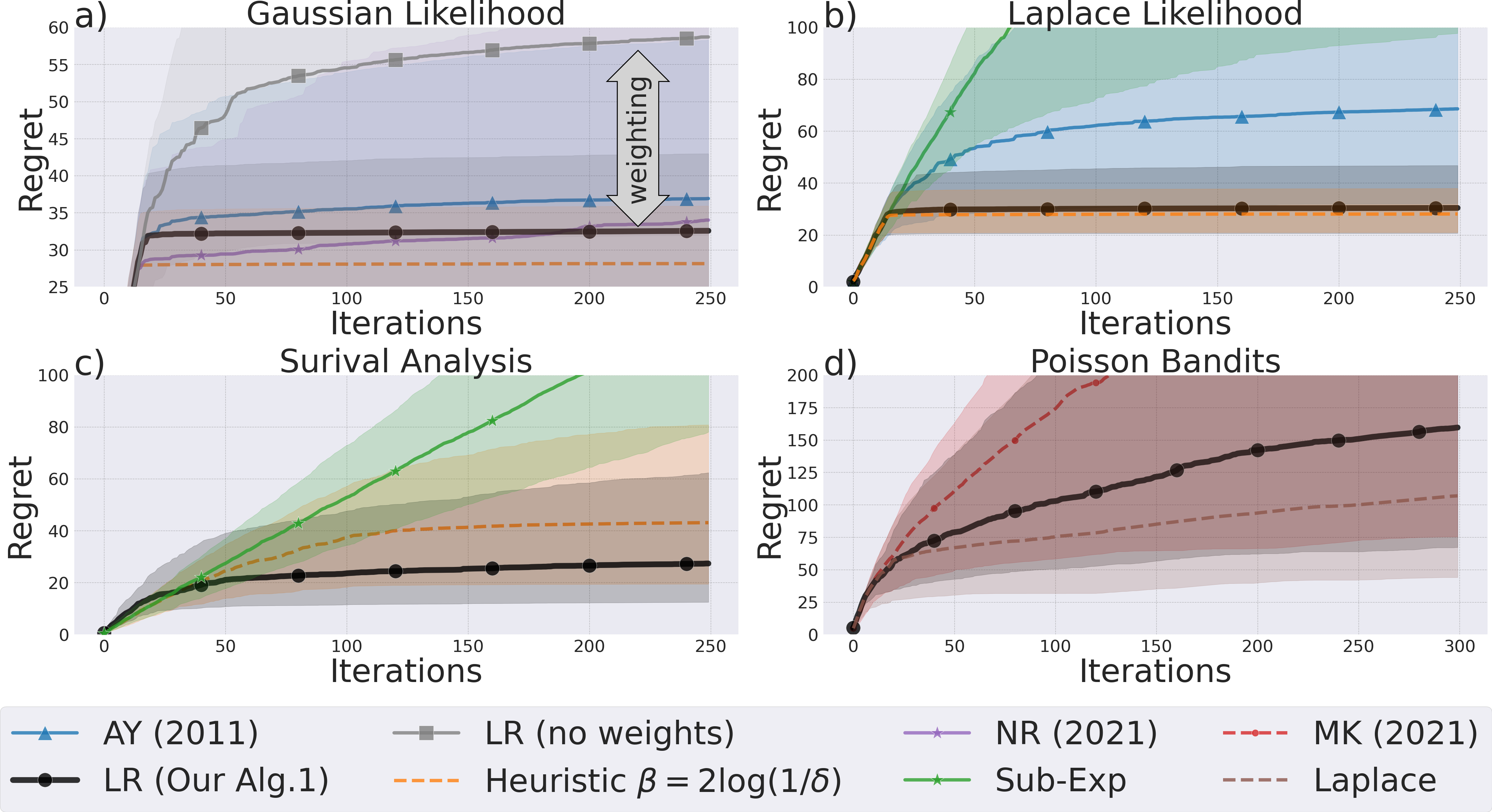}
    \caption{Bandit experiments: On the $y$-axis we report cumulative regret, while the $x$-axis shows the number of iterations. In a) and b) we report the results for linear models with different parametric additive noise. In c) we report the results on a survival analysis with a log-Weibull distribution ($p=2$) and in d) we showcase Poisson bandits. See App.~\ref{app:experiments} for more details. Heuristic methods are \emph{dashed}, while provable are \emph{solid}. Our sets perform the best among all provable methods. Notice in a) the difference in gray and black represents the improvement due to adaptive weighting over $w_s=1$ for all $s\in [t]$. For each experiment we did 10 reruns, median values are plotted.} 
    \label{fig:experiments} \vspace{\spaceafterbanner}
\end{figure*}
\vspace{\spacebeforeparagraphs}
\paragraph{Additive Noise Models}\looseness -1
Suppose that $r_\tstar$ is linear and we observe 
$y_s =  x_s^\trp \tstar + \eta_s$,
where $\eta_s$ is additive noise, and $\tstar$ is an element of a Hilbert space. We consider classical Gaussian noise as well as Laplace noise in Fig. \ref{fig:experiments}[a), b)]. Notice that in both cases our confidence sets yield lower regret than any other provably valid method. In both cases they are performing as good as \emph{heuristic} confidence sets with confidence parameter $\beta_t \equiv 2\log(1/\delta)$. The sub-Gaussian confidence sets of \citet{abbasi:improved} (AY 2011) are invalid for the Laplace distribution as it is not sub-Gaussian but only sub-Exponential. For this reason, we compare also with sub-exponential confidence sets derived similarly to those of \citep{faury:logistic}. The confidence sets of \citep{Neiswanger21a} (NR 2021) perform similarly on Gaussian likelihood, but are only applicable to this setting, as their generalization to other likelihood families involves intractable posterior inference. We note also the difference between the unweighted LR and the weighted one. The examples in Fig.~\ref{fig:experiments} use the true payoff functions $r(x) =-(1.4-3x)\sin(18x)$, which we model as an element of a RKHS with squared exponential kernel lengthscale $\gamma = 6\times 10^{-2}$ on $[0,1.2]$, which is the baseline function no.~4 in the global optimization benchmark database \emph{infinity77} \citep{gavana}. Additional experiments can be found in App.~\ref{app:experiments}. 
\vspace{\spacebeforeparagraphs}
\paragraph{Poisson Bandits}\looseness -1 
A prominent example of generalized linear bandits (GLB) are Poisson bandits, where the linear payoff is scaled by an exponential function. We instantiate our results on a common benchmark problem, and report the results in Fig.~\ref{fig:experiments}d). We improve the regret of UCB for GLBs compared to two alternative confidence sets: one that uses a Laplace approximation with a heuristic confidence parameter, and one inspired by considerations in \citet{mutny:poisson} (MK 2021), also with a heuristic confidence parameter. Note that we cannot compare directly to their provable results in their original form as they do not state them in the canonical form of the exponential family.
\vspace{\spacebeforeparagraphs}
\paragraph{Survival Analysis}\looseness -1
Survival analysis is a branch of statistics with a rich history that models the lifespan of a service or product \citep{Breslow1975,Cox1997,kleinbaum2010survival}. The classical approach postulates a well-informed likelihood model. Here, we use a specific hazard model, where the survival time $T$ is distributed with a Weibull distribution, parametrized by $\lambda$ and $p$. The \emph{rate} $\lambda_\theta(x)=\exp(x^\top \theta)$ differs for each configuration $x$, and $p$ -- which defines the shape of the survival distribution -- is fixed and known. We assume that the unknown part is due to the parameter $\theta$ which is the quantity we build a confidence set around to use within the UCB Algorithm. In particular, the probability density of the Weibull distribution is
$P(T=t|x) = \lambda_\theta(x)pt^{p-1}\exp(-t^p\lambda_\theta(x))$. In fact, with $p=2$, the confidence sets are convex and the UCB rule can be implemented efficiently. 

\looseness -1 Interestingly, this model admits an alternate linear regression formulation. Namely upon using the transformation $Y=\log T$, the transformed variables $Y|x$ follow a Gumbel-type distribution, with the following likelihood that can be obtained by the change of variables
$P(Y=y|x) = \lambda_\theta(x)p\exp(y)^p\exp(-\exp(y)^p \lambda_\theta(x))$.
The expectation over $Y$ allows us to express it as a linear regression problem since
$\E[Y|x] = -(\theta^\top x +\gamma)/p$,
where $\gamma$ is the Euler-Mascheroni constant. More importantly, $Y\vert x$ is sub-exponential. Hence, this allows us to use confidence sets for sub-exponential variables constructed with the pseudo-maximization technique inspired by \citet{faury:logistic}. More details on how these sets are derived can be found in App.~\ref{app:experiments}. However, these approaches necessarily employ crude worst-case bounds and as can be seen in Figure \ref{fig:experiments}c) the use of our LR-based confidence sequences substantially reduces the regret of the bandit learner.
\vspace{\spacebeforesections}
\section{Related Work and Conclusion}\label{sec:onlinetoconf} \looseness -1
\vspace{\spacebeforeparagraphs}
\paragraph{Related Work} \looseness -1
The adaptive confidence sequences stem from the seminal work of \citet{Robbins1972}, who note that these sets have $\alpha$-bounded Type I error. The likelihood ratio framework has been recently popularized by \citet{ramdas:universal} for likelihood families without known test statistics under the name \emph{universal inference}. This approach, although long established, is surprisingly uncommon in sequential decision-making tasks like bandits. This might be due to the absence of an analysis deriving the size of the confidence sets \citep{mutny:poisson}, a necessary ingredient to obtain regret bounds. We address this gap for generalized linear models.  Another reason might be that practitioners might be interested in non-parametric \emph{sub-}families -- a scenario our method does not cover. That being said, many fields such as survival analysis \citep{Cox1997} {\em do} have well-informed likelihoods. However, most importantly, if used naively, this method tends to fail when one departs from assumptions that our probabilistic model is identifiable (i.e., $p_{\theta}(\cdot\sep x) = p_{\ttilde}(\cdot\sep x)$ even if $\theta \not = \ttilde$). We mitigate this problem by introducing the scaling parameters $w_t$ in Eq. \eqref{eq:lrp} to deal with it. 

Prevalent constructions of anytime-valid confidence intervals rely on carefully derived concentration results and for a specific estimator such as the least-squares estimator and noise sub-families such as sub-Gaussian, sub-Bernoulli and sub-Poisson \cite{abbasi:improved, faury:logistic, mutny:poisson}. Their constructions involve bounding the suprema of collections of self-normalized stochastic processes \citep{faury:logistic, mutny:poisson, chowdhury2022bregman}. To facilitate closed-form expressions, worst-case parameters are introduced that prohibitively affect the size of the sets -- making them much larger than they need to be. 

\citet{chowdhury2022bregman} use the exact form of the likelihood to build confidence sets for parameters of exponential families. However, their approach is restricted to exponential family distributions. They use self-normalization and mixing techniques to explicitly determine the size of the confidence set and do not use an online learning subroutine as we do here. \citet{Neiswanger21a} use likelihood ratios for bandit optimization with possibly misspecified Gaussian processes but is not tractable beyond Gaussian likelihoods. The relation between online convex optimization and confidence sets has been noted in so-called online-to-confidence conversions \citep{Abbasi2012-online, jun:gloc, zhao:ftrl}, where the existence of a low-regret learner implies a small confidence set. However, these sets still use potentially loose regret bounds to define confidence sets. Our definition is \emph{implicit}. We do not necessarily need a regret bound to run our method, as the radius will depend on the actual, instance-dependent performance of the learner.
\paragraph{Conclusion} \looseness -1 \vspace{\spacebeforeparagraphs}
In this work, we generalized and analyzed sequential likelihood ratio confidence sets for adaptive inference. We showed that with well-specified likelihoods, this procedure gives small, any-time valid confidence sets with model-agnostic and precise coverage. For generalized linear models, we quantitatively analyzed their size and shape. We invite practitioners to explore and use this very versatile and practical methodology for sequential decision-making tasks.

\begin{ack}
We thank Wouter Koolen and Aaditya Ramdas for helpful discussions as well as for organizing the SAVI workshop where these discussions took place. NE acknowledges support from the Swiss Study Foundation and the Zeno Karl Schindler Foundation. MM has received funding from the Swiss National Science Foundation through NFP75. This publication was created as part of NCCR Catalysis (grant number 180544), a National Centre of Competence in Research funded by the Swiss National Science Foundation.
\end{ack}

\bibliographystyle{apalike}
\bibliography{refs}

\begin{thebibliography}{}

\bibitem[Abbasi-Yadkori et~al., 2011]{abbasi:improved}
Abbasi-Yadkori, Y., P{\'a}l, D., and Szepesvari, C. (2011).
\newblock Improved algorithms for linear stochastic bandits.
\newblock In {\em Advances in Neural Information Processing Systems}, pages
  2312--2320.

\bibitem[Abbasi-Yadkori et~al., 2012]{Abbasi2012-online}
Abbasi-Yadkori, Y., Pal, D., and Szepesvari, C. (2012).
\newblock Online-to-confidence-set conversions and application to sparse
  stochastic bandits.
\newblock In Lawrence, N.~D. and Girolami, M., editors, {\em Proceedings of the
  Fifteenth International Conference on Artificial Intelligence and
  Statistics}, volume~22 of {\em Proceedings of Machine Learning Research},
  pages 1--9, La Palma, Canary Islands. PMLR.

\bibitem[Abe and Long, 1999]{abe1999associative}
Abe, N. and Long, P.~M. (1999).
\newblock Associative reinforcement learning using linear probabilistic
  concepts.
\newblock In {\em ICML}, pages 3--11. Citeseer.

\bibitem[Auer, 2002]{Auer2002}
Auer, P. (2002).
\newblock Using confidence bounds for exploitation-exploration trade-offs.
\newblock {\em Journal of Machine Learning Research}, 3(Nov):397--422.

\bibitem[Azoury and Warmuth, 1999]{azoury:forecaster}
Azoury, K.~S. and Warmuth, M.~K. (1999).
\newblock Relative loss bounds for on-line density estimation with the
  exponential family of distributions.
\newblock {\em Machine Learning}, 43:211--246.

\bibitem[Baker, 2016]{reproducibilitycrisis}
Baker, M. (2016).
\newblock 1,500 scientists lift the lid on reproducibility.
\newblock {\em Nature}, 533:452--454.

\bibitem[Breslow, 1975]{Breslow1975}
Breslow, N.~E. (1975).
\newblock Analysis of survival data under the proportional hazards model.
\newblock {\em International Statistical Review / Revue Internationale de
  Statistique}.

\bibitem[Carpentier et~al., 2020]{carpentier:potential}
Carpentier, A., Vernade, C., and Abbasi-Yadkori, Y. (2020).
\newblock The elliptical potential lemma revisited.
\newblock {\em ArXiv}, abs/2010.10182.

\bibitem[Cesa-Bianchi and Lugosi, 2006]{Cesa-Bianchi2006}
Cesa-Bianchi, N. and Lugosi, G. (2006).
\newblock {\em Prediction, Learning, and Games}.
\newblock Cambridge University Press, New York, NY, USA.

\bibitem[Chowdhury et~al., 2022]{chowdhury2022bregman}
Chowdhury, S.~R., Saux, P., Maillard, O.-A., and Gopalan, A. (2022).
\newblock Bregman deviations of generic exponential families.
\newblock {\em arXiv preprint arXiv:2201.07306}.

\bibitem[Cox, 1997]{Cox1997}
Cox, D.~R. (1997).
\newblock Some remarks on the analysis of survival data.
\newblock In {\em Proceedings of the First Seattle Symposium in Biostatistics},
  pages 1--9. Springer.

\bibitem[Cucker and Smale, 2002]{Cucker2002}
Cucker, F. and Smale, S. (2002).
\newblock On the mathematical foundations of learning.
\newblock {\em Bulletin of the American mathematical society}, 39(1):1--49.

\bibitem[Faury et~al., 2020]{faury:logistic}
Faury, L., Abeille, M., Calauz{\`e}nes, C., and Fercoq, O. (2020).
\newblock Improved optimistic algorithms for logistic bandits.
\newblock In {\em ICML2020: Proceedings of the 37th International Conference on
  International Conference on Machine Learning}.

\bibitem[Filippi et~al., 2010]{filippi:glm}
Filippi, S., Cappe, O., Garivier, A., and Szepesv{\'a}ri, C. (2010).
\newblock Parametric bandits: The generalized linear case.
\newblock In {\em Advances in neural information processing systems}, pages
  586--594.

\bibitem[Gavana, 2021]{gavana}
Gavana, A. (2021).
\newblock infinity global optimization benchmarks and ampgo.
\newblock \url{http://infinity77.net/global_optimization/index.html#}.

\bibitem[Hazan, 2016]{hazan:book}
Hazan, E. (2016).
\newblock Introduction to online convex optimization.
\newblock {\em Found. Trends Optim.}, 2:157--325.

\bibitem[Hazan et~al., 2006]{hazan:logarithmic}
Hazan, E., Agarwal, A., and Kale, S. (2006).
\newblock Logarithmic regret algorithms for online convex optimization.
\newblock {\em Machine Learning}, 69:169--192.

\bibitem[Howard et~al., 2018]{howard:timeuniform}
Howard, S.~R., Ramdas, A., McAuliffe, J., and Sekhon, J. (2018).
\newblock Time-uniform chernoff bounds via nonnegative supermartingales.
\newblock {\em Probability Surveys}, 17:257--317.

\bibitem[Jun et~al., 2017]{jun:gloc}
Jun, K.-S., Bhargava, A., Nowak, R., and Willett, R. (2017).
\newblock Scalable generalized linear bandits: Online computation and hashing.
\newblock {\em Advances in Neural Information Processing Systems}, 30.

\bibitem[Katz-Samuels and Jamieson, 2020]{Katz-Samuels2020a}
Katz-Samuels, J. and Jamieson, K. (2020).
\newblock The true sample complexity of identifying good arms.
\newblock In Chiappa, S. and Calandra, R., editors, {\em Proceedings of the
  Twenty Third International Conference on Artificial Intelligence and
  Statistics}, volume 108 of {\em Proceedings of Machine Learning Research},
  pages 1781--1791. PMLR.

\bibitem[Kleinbaum and Klein, 2010]{kleinbaum2010survival}
Kleinbaum, D.~G. and Klein, M. (2010).
\newblock {\em Survival analysis}, volume~3.
\newblock Springer.

\bibitem[Lattimore and Szepesvári, 2020]{lattimore:book}
Lattimore, T. and Szepesvári, C. (2020).
\newblock {\em Bandit Algorithms}.
\newblock Cambridge University Press.

\bibitem[Makarova et~al., 2021]{makarova2021riskaverse}
Makarova, A., Usmanova, I., Bogunovic, I., and Krause, A. (2021).
\newblock Risk-averse heteroscedastic bayesian optimization.
\newblock In {\em Proc. Neural Information Processing Systems (NeurIPS)}.

\bibitem[McCullagh, 2018]{McCullagh2018}
McCullagh, P. (2018).
\newblock {\em Generalized linear models.}
\newblock Chapman and Hall.

\bibitem[Mukherjee et~al., 2022]{mukherjee2022chernoff}
Mukherjee, S., Tripathy, A.~S., and Nowak, R. (2022).
\newblock Chernoff sampling for active testing and extension to active
  regression.
\newblock In {\em International Conference on Artificial Intelligence and
  Statistics}, pages 7384--7432. PMLR.

\bibitem[Mutn\'{y} and Krause, 2021]{mutny:poisson}
Mutn\'{y}, M. and Krause, A. (2021).
\newblock No-regret algorithms for capturing events in poisson point processes.
\newblock In Meila, M. and Zhang, T., editors, {\em Proceedings of the 38th
  International Conference on Machine Learning}, volume 139 of {\em Proceedings
  of Machine Learning Research}, pages 7894--7904. PMLR.

\bibitem[Mutn\'y and Krause, 2022]{Mutny2022b}
Mutn\'y, M. and Krause, A. (2022).
\newblock Experimental design of linear functionals in reproducing kernel
  hilbert spaces.
\newblock In {\em Proc. Neural Information Processing Systems (NeurIPS)}.

\bibitem[Neiswanger and Ramdas, 2021]{Neiswanger21a}
Neiswanger, W. and Ramdas, A. (2021).
\newblock Uncertainty quantification using martingales for misspecified
  gaussian processes.
\newblock In Feldman, V., Ligett, K., and Sabato, S., editors, {\em Proceedings
  of the 32nd International Conference on Algorithmic Learning Theory}, volume
  132 of {\em Proceedings of Machine Learning Research}, pages 963--982. PMLR.

\bibitem[Orabona, 2019]{orabona:book}
Orabona, F. (2019).
\newblock A modern introduction to online learning.

\bibitem[Ouhamma et~al., 2021]{ouhamma:onlinelr}
Ouhamma, R., Maillard, O.-A., and Perchet, V. (2021).
\newblock Stochastic online linear regression: the forward algorithm to replace
  ridge.
\newblock In {\em Neural Information Processing Systems}.

\bibitem[Ramdas et~al., 2022]{Ramdas2022c}
Ramdas, A., Grünwald, P., Vovk, V., and Shafer, G. (2022).
\newblock Game-theoretic statistics and safe anytime-valid inference.

\bibitem[Robbins et~al., 1972]{Robbins1972}
Robbins, H., Siegmund, D., et~al. (1972).
\newblock A class of stopping rules for testing parametric hypotheses.
\newblock In {\em Proceedings of the Sixth Berkeley Symposium on Mathematical
  Statistics and Probability, Volume 4: Biology and Health}. The Regents of the
  University of California.

\bibitem[Sch{\"o}lkopf et~al., 2001]{Schoelkopf2001}
Sch{\"o}lkopf, B., Herbrich, R., and Smola, A. (2001).
\newblock A generalized representer theorem.
\newblock In {\em Computational learning theory}, pages 416--426. Springer.

\bibitem[Srinivas et~al., 2009]{srinivas:noregret}
Srinivas, N., Krause, A., Kakade, S.~M., and Seeger, M.~W. (2009).
\newblock Gaussian process optimization in the bandit setting: No regret and
  experimental design.
\newblock In {\em International Conference on Machine Learning}.

\bibitem[Ville, 1939]{ville:etudecritique}
Ville, J.-L. (1939).
\newblock {\'E}tude critique de la notion de collectif.

\bibitem[Vovk, 2001]{vovk:forecaster}
Vovk, V. (2001).
\newblock Competitive on‐line statistics.
\newblock {\em International Statistical Review}, 69.

\bibitem[Wainwright, 2019]{wainwright:book}
Wainwright, M.~J. (2019).
\newblock {\em High-dimensional statistics: A non-asymptotic viewpoint}.
\newblock Cambridge university press.

\bibitem[Wald, 1945]{wald:sequentiallrt}
Wald, A. (1945).
\newblock Sequential tests of statistical hypotheses.
\newblock {\em Annals of Mathematical Statistics}, 16:256--298.

\bibitem[Wasserman et~al., 2020]{ramdas:universal}
Wasserman, L.~A., Ramdas, A., and Balakrishnan, S. (2020).
\newblock Universal inference.
\newblock {\em Proceedings of the National Academy of Sciences}, 117:16880 --
  16890.

\bibitem[Zhao et~al., 2022]{zhao:ftrl}
Zhao, H., Zhou, D., He, J., and Gu, Q. (2022).
\newblock Bandit learning with general function classes: Heteroscedastic noise
  and variance-dependent regret bounds.
\newblock {\em ArXiv}, abs/2202.13603.

\end{thebibliography}

\newpage
\appendix

\section{Proofs of Theorem~\ref{thm:coverage} and \ref{thm:bias}}
\label{app:coverage}
\subsection{GLM Families}\label{app:glm-table}

\begin{table}[ht]
\centering
\caption{Examples of exponential family distributions.}
\label{tb:examples}
\begin{tabular}{l|l|l|l|l|l|}
\midrule
\textbf{Name}&$A(z)$& $A'(z)$ & $T(y)$ & $\mu$ & $L$ \\
\hline
Gaussian & $z^2/(2\sigma^2)$ & $z/\sigma^2$ &  $y/\sigma$ & $1/\sigma^2$ & $1/\sigma^2$\\
Poisson & $\exp(z)$ & $\exp(z)$ & $y$ & $\exp(-B)$  & $\exp(B)$ \\
Binomial & $\log(1 + \exp(z))$& $\frac{1}{1+\exp(-z)}$ & $y$ & $\calO(\exp(-B))$ & $1/4$ \\
Weibull & $k\log(z)-\log k$ & $k/z$ & $y^k$& $1/B^2$ & $\infty$
\end{tabular}
\end{table}

\subsection{Proof of Theorem \ref{thm:coverage} (Coverage)}
\begin{proof}
	Starting with $\condE{R_t(\tstar)}{\calF_{t-1}} 
	 $
\begin{eqnarray*}
	&=&\condE{R_{t-1}(\tstar)\frac{p^{w_t}_{\that_{t}}(y_t \sep x_t)}{p^{w_t}_{\tstar }(y_t\sep x_t)} }{ \calF_{t-1}} \\
	&=& R_{t-1}(\tstar) \int \frac{p^{w_t}_{\that_{t}}(y \sep x_t)}{p^{w_t}_{\tstar}(y \sep x_t)} p_{\tstar}(y \sep x_t) \mathrm{d}y \\ &=&  R_{t-1}(\tstar) e^{((w_t -1) D^r_{w_t}(p_{\tstar}(x_t),p_{\that_t}(x_t))} \leq R_{t-1}(\tstar).
    \end{eqnarray*}
	The second equality is due to the fact that $R_{t-1}(\tstar)$ only depends on $x_1,y_1$ through $x_{t-1},y_{t-1}$. Since $\that_t$ is $\calF_{t-1}$ measurable by assumption, $p_{\that_t}$ is a density, and if $w_t < 1$, the integral is equal to an exponential of the R\'enyi-divergence $D^r_{w_t}(\cdot,\cdot)$. The negativity of the exponent follows from $w_t < 1$ and the non-negativity of the divergence. Note that in the "degenerate" case of $w_t = 1$, we can easily see that the integral is over a density (cancellation), and hence also bounded by $1$.
 The last part of the statement follows easily by using Ville's inequality for supermartingales.
\end{proof}
  All the elements of the above proof appear in \citet{ramdas:universal} albeit separately, and not with time-varying powered robust likelihoods.

\subsection{Proof of Theorem \ref{thm:bias} (Bias)} 
\newcommand{\noisefree}{\theta^\times}
We will need a gradient characterization of strong-convexity, which we prove in the following lemma. 

\begin{lemma}[Convexity: Gradient]\label{lemma:convexity-gradient}
    Defining 
    $$
        F_t(\theta) = -\sum_{s=1}^t \E_{\tstar}[\nabla \log_\theta p(y_s\sep x_s) \sep \calF_{s-1}],
    $$
    under the assumption $p_\theta(y_s|x_s) = -g(x_s^\trp \theta)$ and $g$ is $\mu$-strongly convex, we have for any $\theta \in \Theta$: 
    \begin{equation*}
        (F_t(\theta) - F_t(\tstar))^\trp (\theta - \tstar) \geq \norm{\theta - \tstar}_{\VV_t^{\mu;0}}^2.
    \end{equation*}    
\end{lemma}
\begin{proof}
We assume that $g$ is $\mu$-strongly convex. Therefore, for any $s \leq t$, we get the two inequalities
\begin{align*}
g(x_s^\trp \theta) - g(x_s^\trp \tstar) &\geq g'(x_s^\trp \tstar)(x_s^\trp \theta - x_s^\trp\tstar) + \frac{\mu}{2}\norm{x_s^\trp \tstar - x_s^\trp \theta}_2^2 \\
g(x_s^\trp \tstar) - g(x_s^\trp \theta)  &\geq g'(x_s^\trp \theta)(x_s^\trp \tstar - x_s^\trp\theta) + \frac{\mu}{2}\norm{x_s^\trp \tstar - x_s^\trp \theta}_2^2.
\end{align*}
Adding these two together, we obtain
\begin{equation*}
    0 \geq (g'(x_s^\trp \tstar) - g'(x_s^\top \theta))(x_s^\trp (\theta - \tstar)) + \mu \norm{\theta - \tstar}_{x_s x_s^\top}^2. 
\end{equation*}
Observing that $-\nabla_\theta p_\theta(y_s\sep x_s) = -g'(x_s^\trp \theta)x_s$, we can equivalently write
\begin{equation*}
    0 \geq (\nabla \log p_{\theta}(y_s\sep x_s) - \nabla \log p_{\tstar}(y_s\sep x_s))^\top(\theta - \tstar)) + \mu \norm{\theta - \tstar}_{x_s x_s^\top}^2.
\end{equation*}
This holds for any realization of $y_s$, and hence taking expectations yields
\begin{equation*}
    \left(\E[\nabla \log p_{\tstar}(y_s\sep x_s) \sep \calF_{s-1}] - \E[\nabla \log p_{\theta}(y_s\sep x_s) \sep \calF_{s-1}]\right)^\trp (\theta - \tstar) \geq \mu \norm{\theta - \tstar}_{x_s x_s^\trp}^2.
\end{equation*}
Summing up over $s\leq t$ and using the definition of $F_t$ we get
\begin{equation*}
    (F_t(\theta) - F_t(\tstar))^\trp (\theta - \tstar) \geq \norm{\theta - \tstar}_{\VV_t^\mu;0}^2.
\end{equation*}
\end{proof}
\label{app:bias-proof}
Notice that the estimator in Alg. \ref{alg:online} has to fulfil the KKT conditions. We will denote the condition for belonging to the set as $h(\theta) \leq B^2$, where $h$ is a squared and twice-differentiable norm (there are many choices beyond $||\cdot||_2^2$). The KKT conditions are 
\begin{align}\label{eq:kkt}
 \sum_{s=1}^t -\nabla_\theta \log p_\theta(y_s|x_s) + \nabla \psi_t(\theta) + l\nabla h(\theta) = 0 \\
l (h(\theta) - B^2) = 0 \nonumber \\
l \geq 0, \nonumber
\end{align}
where the second and third conditions represent a complementary slackness requirement. Notice that the system of these equations has a unique solution due to the strong-convexity of the objective, and has to attain a unique minimum on a compact convex subset of $\mathbb{R}^d$. Adding the same quantity on both sides of \eqref{eq:kkt} yields
\begin{align} \label{eq:noisefreeestimator}
    &\sum_{s=1}^t  -\mathbb{E}_\tstar[\nabla_\theta \log p_\theta(y_s|x_s)|\calF_{s-1}]+ \nabla \psi_t(\theta) + l\nabla h(\theta)  \nonumber\\=\quad  &\underbrace{\sum_{s=1}^t[ \nabla_\theta \log p_\theta(y_s|x_s) -\mathbb{E}_\tstar[\nabla_\theta \log p_\theta(y_s|x_s)|\calF_{s-1}]]}_{:=E_t}. 
\end{align}
This line motivates the definition of the error-free estimator in $\ttimes_t$ \eqref{eq:bias-approximation}, where $E_t$ is set to zero. We will also make use of a fundamental property of the score (gradient of log-likelihood), namely
\begin{equation}\label{eq:score}
    \mE_{y_t \sim p_\tstar(\cdot \sep x)}[\nabla \log p_\tstar(y_t|x_t)|\calF_{t-1}] = 0.
\end{equation}
A classical textbook reference for this is e.g. \citet{McCullagh2018} but any other classical statistics textbook should contain it.
 Using these observations, we can already prove Theorem~\ref{thm:bias}.

\begin{proof}[Proof of Theorem \ref{thm:bias}]
    Using the optimality conditions of $\theta_t^\times$, $h(\theta) = ||\theta||^2_2$ and $\psi_t(\theta) = \norm{\theta}_2^2$, we obtain the following statements:
    \begin{align*}
    & \sum_{s=1}^t  -\mathbb{E}_\tstar[\nabla \log p_{\theta^\times_t}(y_s|x_s)|\calF_{s-1}]+ \lambda \theta^\times_t + 2l \theta^\times_t = 0 \notag \\
    \implies & \sum_{s=1}^t  -\mathbb{E}_\tstar[\nabla \log p_{\theta^\times_t}(y_s|x_s)|\calF_{s-1}] + \mathbb{E}_\tstar[\nabla \log p_{\tstar}(y_s|x_s)|\calF_{s-1}]+ {\lambda} \theta^\times_t + 2l\theta^\times_t = 0 \notag,
    \end{align*}
    where in the last line we used the property \eqref{eq:score}. Now, notice that since we know $\tstar$ is generating the data, the best possible explanation without enforcing the constraint and the regularization would be to set $\theta^\times_t = \tstar$ as
    the cross-entropy is minimized at this point, and the above is just the optimality condition for optimizing the cross-entropy between these two distribution. Of course, this is only in the absence of regularization or constraints i.e. ${\lambda}= 0$. Now, with the regularization constraint, as the true $\tstar$ lies inside the constraint $h(\theta)\leq B^2$, and both the regularization and constraints induce star-shaped sets, their effect is to make $\theta^\times_t$ smaller in norm than $\theta^*$. This holds generally for any $h$ which is a norm. As a consequence of this consideration, $||\theta^\times_t||_2 < B$, and then the complementary slackness dictates that $l = 0$.

    We can therefore proceed with this simplification. Let us use the shorthand $F_t(\theta) = \sum_{s=1}^t -\mathbb{E}_\tstar[\nabla \log p_{\theta}(y_s|x_s)|\calF_{s-1}]$ and compute
    \begin{align}
     &F_t(\theta_t^\times) -   F(\tstar) + \lambda (\noisefree_t-\tstar) &&= -\lambda\tstar \nonumber  \\
    \implies  &(\theta_t^\times - \tstar)^\top ( F_t(\theta_t^\times) - F(\tstar) + \lambda (\noisefree_t-\tstar)) &&= - \lambda(\noisefree_t - \tstar)^\top\tstar \nonumber \\\stackrel{\text{Lemma}~\ref{lemma:convexity-gradient}}{\implies} &\norm{\noisefree_t - \tstar}^2_{\VV_t^{\mu,\lambda}} &&\leq - \lambda(\theta_t^\times  -\tstar)^\top \tstar. \label{eq:usefulforsecondcauchy} %
    \end{align}

    It suffices to apply the Cauchy-Schwarz Inequality and invoke \eqref{eq:usefulforsecondcauchy}:
    \begin{eqnarray*}
            \operatorname{bias}_{x_s}(\that_s)^2  & = & (x_s^\top (\that_t^\times - \tstar))^2 \\ 
             & \leq &||x_s||_{(\VV^{\mu,\lambda}_t)^{-1}}^2  \norm{\noisefree_t - \tstar}^2_{\VV_t^{\mu,\lambda}} \\
            & \leq & \lambda ||x_s||_{(\VV^{\mu,\lambda}_t)^{-1}}^2 \lambda (\theta_t^\times - \tstar)^\top (-\tstar) \\
        & \leq & \lambda ||x_s||_{(\VV^{\mu,\lambda}_t)^{-1}}^2 ||(\theta_t^\times - \tstar)||_2  ||\tstar||_2\\
        & \leq & 2\lambda ||x_s||_{(\VV^{\mu,\lambda}_t)^{-1}}^2 ||\tstar||_2^2,
    \end{eqnarray*}
where in the last inequality we used $\norm{\ttimes_t}_2 \leq \norm{\tstar}_2$, due to the regularizer, as explained above.
\end{proof}

\paragraph{GLM models}
Let us define the processes
\[ S_t = \sum_{s=1}^{t} x_sT(y_s)  \quad \text{and} \quad W_t = \sum_{s=1}^{t} x_sA'(x_s^\top \tstar). \]
In this scenario, an equivalent of \eqref{eq:noisefreeestimator} then involves the gradient of the regularized (unweighted) log-partition function ${Z_t^\lambda}$ we defined in \eqref{eq:log-parition-definition} and is equal to
\begin{equation}\label{eq:z-operator-glm}
 \sum_{s=1}^t A'(x_s^\top \theta) x_s + \nabla \psi_t(\theta) + l \nabla h(\theta) = \tilde E_t,
 \end{equation}
where $\tilde E_t = S_t$ for $\that_t$ and $\tilde E_t = W_t$ for $\theta_t^\times$.

\section{Proof of Theorem~\ref{theorem:bregmanball} (Bregman Ball Confidence Set)}

\newcommand{\fullBgOffset}[1]{\mathcal{B}_{#1}}

\paragraph{Proof sketch} \looseness -1 We give a quick sketch of the proof. To bound the size of the sets, we will draw inspiration from the i.i.d. parameter estimation analysis of \citet{ramdas:universal} and separate out the likelihood ratio in a part that relates the true parameter with the estimator sequence (i.e. regret), and a part that is independent of the estimator and characterized by a supremum of a stochastic process. We want to show that \emph{any} point which is far away from the true parameter will eventually not be included in the confidence set anymore. Defining $\lik^{(t)}(\{\that_s\}_{s=1}^t)$ as $\prod_{i=1}^t p_{\that_{i}}^{w_i}(y_i \sep x_i)$, we wish to show that for any $\theta$ far from $\tstar$, we have
\begin{equation*}\label{eq:exclusion}
	\log \frac{1}{R_t(\theta)} = \log\frac{\lik^{(t)}(\theta)}{\lik^{(t)}(\tstar)} + \log\frac{\lik^{(t)}(\tstar)}{\lik^{(t)}(\{\that_s\}_{s=1}^t)} \leq \log(\alpha),
\end{equation*}
which is equivalent to saying that $\theta \not \in \mathcal{C}_{t}$. The second term corresponds to our notion of regret exactly ($\mathcal{R}_t$, as discussed above). The first term is what we will focus on. We will bound the supremum of $\log\frac{\lik^{(t)}(\theta)}{\lik^{(t)}(\tstar)}$ for all $\theta$ sufficiently far away from $\tstar$. "Far away" will be measured in the Bregman divergence outlined above. Note that this quantity can be expected to be negative, in general, (especially for "implausible" parameters), since with enough data, $\tstar$ should appear much more likely. Writing this ratio out, we will observe that it is equal to 
\begin{equation*} \looseness -1
   -D_{Z^0_t}(\theta, \tstar) + \underbrace{\inner{\theta-\tstar}{{\sum_{s=1}^t w_s x_s ({T(y_s) - \E_{\tstar}[T(y_s)]})}}}_{\approx \tilde{S}_t}.
\end{equation*}
\looseness -1 At this point, it will be sufficient to bound the cross term (second term) over the whole of $\Theta$. We view this supremum as part of the Legendre Fenchel transform of the function $\fullBgOffset{t}(\lambda) = D_{Z^\nu_t}(\tstar + \lambda, \tstar)$:
\begin{align*}
 \sup_{\lambda\in \R^d}\left(\lambda^\trp \tilde{S_t} - \fullBgOffset{t}(\lambda) \right) 
 {=} \left(\fullBgOffset{t}\right)^\star(\Tilde{S_t})\label{eq:dualityestimator}
 \end{align*}
and harness duality properties of the Bregman divergence, along with known concentration arguments \citep[Theorem A.1]{chowdhury2022bregman}.

\subsection{Technical Lemmas}
We need to introduce the concept of Legendre functions:
\begin{definition}
    Let $f: \R^d \rightarrow \R$ be a convex function and $C = \mathrm{int}(\mathrm{dom}(f))$. Then, a function is called Legendre if it satisfies
    \begin{enumerate}
        \item $C$ is non-empty.
        \item $f$ is differentiable and strictly convex on $C$.
        \item $\lim_{n \rightarrow \infty} \norm{\nabla f(x_n)} = \infty$ for any sequence $(x_n)_n$ with $x_n \in C$ for all $n$ and $\lim_{n \rightarrow \infty} x_n = x$ for some $x \in \partial C$.
    \end{enumerate}
\end{definition}
This means that the gradient has to blow up near the edge of the domain. Note as well that the boundary condition is vacuous if there is no boundary. Legendre functions have some nice properties, most importantly regarding the bijectivity of their gradients (see e.g. \cite{lattimore:book}):
\begin{lemma}
    \label{lemma:legendre}
    For a Legendre function $f: \R^d \rightarrow \R$
    \begin{enumerate}
        \item $\nabla f$ is a bijection between $\mathrm{int}(\mathrm{dom}(f))$ and $\mathrm{int}(\mathrm{dom}(f^*))$ with the inverse $(\nabla f)^{-1} = \nabla f^*$.
        \item $D_f(x,y) = D_{f^*}(\nabla f(y), \nabla f(x))$ for all $x, y \in \mathrm{int}(\mathrm{dom}(f))$.
        \item The Fenchel conjugate $f^*$ is also Legendre.
    \end{enumerate}
\end{lemma}
With this, we can prove a slightly extended result, that appears as Lemma 2.1. in \cite{chowdhury2022bregman}.
\begin{lemma}
\label{lemma:dualitypatrick}
For a Legendre function $f$ we have the identity
    $$
        D_f(x,y) = (D_{f, x})^*(\nabla f(y) - \nabla f(x))
    $$
    where we define $D_{f, x}(\lambda) = D_f(x+\lambda, x)$.
\end{lemma}

\paragraph{Notational Shorthands}
Remember the model \eqref{eq:glmmodel}, with log-partition function $A$. We define $A_s(\theta) = w_s A(x_s^\trp \theta)$ and $T_s(y) := w_s x_s T(y)$ to denote the log-partition function and the response function of the same exponential family distribution, but parametrized by $\theta$ instead of $x_s^\trp \theta$. That this is a valid parametrization can easily be seen from the likelihood definition. Indeed, denote by $p^{EF}_\beta$ the exponential family reward distribution with parameter $\beta$. Then our model \eqref{eq:glmmodel} can be seen to satisfy
$$
    p_\theta(y \sep x_s) = p_{x_s^\trp \theta}(y) = h(y)\exp(T(y) x_s^\trp \theta - A(x_s^\trp\theta)) = h(y)\exp(T_s(y)^\trp \theta - A_s(\theta)).
$$
Exponentiating the likelihood with a weighting $w_s$ gives rise to another exponential family distribution. We can see that
$$
    p_\theta^{w_s}(y \sep x_s) = {h}^{w_s}(y)\exp(w_s T(y) x_s^\trp \theta - w_s A(x_s^\trp\theta)) = {h}^{w_s}(y)\exp(T_s(y)^\trp \theta - A_s(\theta)).
$$
Note that this does not necessarily integrate to one, but it is easy to see that there is a normalization function $\tilde{h}$ that makes it integrate to one. Therefore, the following is a valid parametrization of an exponential family distribution: 
$$
\tilde{h}(y)\exp(T_s(y)^\trp \theta - A_s(\theta)).
$$
Additionally, let $A_0(\theta) = \frac{\nu}{2}\norm{\theta}_2^2$ be defined on $\R^d$ (i.e. a Legendre Function). We will also define the estimator
$$
        \ttilde_t = \left(\nabla Z^\nu_t\right)^{-1}\left(\sum_{s=1}^t T_s(y_s)\right).
$$
This is a well-defined quantity because the gradient will be invertible, by Lemma~\ref{lemma:legendre} above.

Conveniently, \cite{chowdhury2022bregman} prove the following Theorem 7 using an elegant application of the method of mixtures.
\begin{proposition}[Theorem 7\ in \cite{chowdhury2022bregman}]
\label{proposition:saux:a1}
    With probability $1-\delta$, for all $t \in \N$
    $$
        D_{Z^\nu_t}(\tstar, \ttilde_t) \leq \log(1/\delta) + A_0(\tstar) + \Gamma_t^\nu,
    $$
    where 
    $$
          \Gamma_t^\nu = \log\left(\frac{\int_{\R^d} \exp(-\frac{1}{2}\norm{\theta}^2_2)\mathrm{d}\theta }{\int_{\R^d} \exp(-D_{Z^\nu_t}(\theta, \ttilde_t))\mathrm{d}\theta}\right).
    $$
\end{proposition}
Lastly, we will need the (one-argument) function
\newcommand{\Bgoffset}{\mathcal{B}_{\tstar}^{(s:t)}(\tstar)}
$$
    \fullBgOffset{t}(\lambda) = D_{Z_t^\nu}(\theta + \lambda, \theta),
$$
i.e. a shortcut for the Bregman divergence of $Z_t^\nu$ at $\theta$. 
We use this one-argument function as we will be interested in its dual. We will also need a lemma on the sub-homogeneity properties of this object.
\begin{lemma}
    \label{lemma:sublinarity}
    Under Assumption~\ref{ass:bounded}, for $\theta \in \Theta$ and $\lambda$ such that $\theta + \lambda \in \Theta$, we have for any $\gamma \leq \frac{\mu}{2L}$
	$$
	\fullBgOffset{t}(\gamma \lambda) \leq \frac{1}{2} \gamma \fullBgOffset{t}(\lambda),
	$$
     i.e. function $g(\gamma) = \fullBgOffset{t}(\gamma \lambda)$ is sub-homogeneous with contraction parameter $\frac{1}{2}$ on $[0, \frac{\mu}{2L}]$.
\end{lemma}
See Appendix~\ref{sec:appendix:a:lemmaproof} for a proof.

\subsection{Proof of Theorem~\ref{theorem:bregmanball}}\label{app:proof-bregman}
	As mentioned in the main paper, our proof will show that all $\theta$ sufficiently far from $\tstar$ will be excluded from $\calC_t$ eventually. Equation~\eqref{eq:exclusion} in the main text specifies the exclusion criterion, i.e. $\theta \not \in \mathcal{C}_t$ if and only if
\begin{equation}
	\frac{1}{R_t(\theta)} = \log\frac{\lik^{(t)}(\theta)}{\lik^{(t)}(\tstar)} + \log\frac{\lik^{(t)}(\tstar)}{\lik^{(t)}(\{\that_s\}_{s=1}^t)} \leq \log(\alpha). \label{eq:fullinvertedratio}
\end{equation}
The second term is bounded by the regret of the online learner. And therefore, a sufficient condition for $\theta \not \in \calC_t$ is 
\begin{equation*}
    \log\left(\frac{\lik^{(t)}(\theta)}{\lik^{(t)}(\tstar)}\right) \leq \log(\alpha) - \mathcal{R}_t.
\end{equation*}
Henceforth, we will be interested in having an explicit set $\tilde{\calC}_t$ such that we can upper bound
\begin{equation}
    \sup_{\theta \notin \Tilde{\calC}_t} \log\left(\frac{\lik^{(t)}(\theta)}{\lik^{(t)}(\tstar)}\right). \label{eq:supremum}
\end{equation}
This will imply that that $\Tilde{\calC}_t^c \subset \calC_t^c$, or in other words, $\calC_t \subset \Tilde{\calC}_t$. Without further ado, let us derive a more convenient form of the ratio in question
\allowdisplaybreaks
	\begin{align*}
		\log\left(\frac{\lik^{(t)}(\theta)}{\lik^{(t)}(\tstar)}\right) &= \log\left( \frac{\prod_{s=1}^t h(y_s)\exp\left( w_s x_s^\trp \theta T(y_s) - w_s A(x_s^\trp \theta) \right) }{\prod_{s=1}^t h(y_s)\exp\left( w_s x_s^\trp \theta T(y_s) - w_s A(x_s^\trp \tstar) \right) } \right) \\
		&= \sum_{s=1}^t w_s x_s^\trp \theta T(y_s) - w_s A(x_s^\trp \theta)  - w_s x_s^\trp\tstar T(y_s) +  w_s A(x_s^\trp \tstar) \\
		&= \sum_{s=1}^t \inner{\theta - \tstar}{w_s T(y_s)x_s} + w_s A(x_s^\trp \tstar) - w_s A(x_s^\trp \theta) \\
		&= \sum_{s=1}^t \inner{\theta - \tstar}{ w_s T(y_s)x_s} - \big(w_s A(x_s^\trp \theta) - w_s A(x_s^\trp \tstar) - {x_s^\trp (\theta-\tstar)}{w_s A'(x_s^\trp \tstar)} \\
        &\quad + {x_s^\trp (\theta-\tstar)}{w_s A'(x_s^\trp \tstar)}\big) \\
		&= \sum_{s=1}^t \inner{\theta - \tstar}{w_s T(y_s)x_s} - w_s D_A(x_s^\trp\theta, x_s^\trp \tstar) - {x_s^\trp (\theta-\tstar)}{w_s A'(x_s^\trp \tstar)} \\
		&= -\sum_{s=1}^t w_s D_A(x_s^\trp \theta, x_s^\trp \tstar) + \sum_{s=1}^t \inner{\theta - \tstar}{w_s T(y_s)x_s - x_s w_s A'(x_s^\trp\tstar)}.
	\end{align*}
	We can switch parametrizations as described above:
	\begin{align}
		\log\left(\frac{\lik^{(t)}(\theta)}{\lik^{(t)}(\tstar)}\right) &= - D_{Z^0_t}(\theta, \tstar) + \sum_{s=1}^t \inner{\theta - \tstar}{T_s(y_s) - \nabla A_s(\tstar)} \nonumber \\
		&= - D_{Z^0_t}(\theta, \tstar) + \sum_{s=1}^t \inner{\theta - \tstar}{T_s(y_s) - \E_{\tstar}[T_s(y_s)]} \nonumber \\
        &= - D_{Z^0_t}(\theta, \tstar) + \inner{\theta-\tstar}{S_t} , \label{eq:likelihoodconcentrationpart}
	\end{align}
where we define $S_t := \sum_{s=1}^t \left({T_s(y_s) - \E_{\tstar}[T_s(y_s)]}\right)$. 

$Z_t^\nu$ is strictly convex whenever $\nu \not=0$, and convex otherwise (it might also be strictly convex otherwise, corresponding to some cases where the $x_s$ span the full $d$-dimensional Euclidean space and $w_s > 0$, which will be satisfied uniformly. We note that since $\mathrm{dom}(Z_t^\nu) = \R^d$, $Z^\nu_t$ is, therefore, Legendre, and its gradient is invertible. 
We will relate our problem to this estimator via the duality properties developed above. First, note that by the well-known fact $\E_\tstar[T_s(y_s)] = \nabla A_s(\tstar)$ and by the definition of $\ttilde_t$, we have
\begin{align}
    S_t &= \nabla Z^\nu_t\left( \left(\nabla Z^\nu_t\right)^{-1}\left(\sum_{s=1}^t T_s(y_s)\right)\right) - \nabla Z^0_t(\tstar) \nonumber \\
    &= \underbrace{\nabla Z^\nu_t\left( \ttilde_t \right) - \nabla Z^\nu_t(\tstar)}_{=: \Tilde{S}_t} + \underbrace{\nabla A_0(\tstar)}_{ =\nu\tstar} .\label{eq:stildedef}
\end{align}
Now, we leverage the duality properties: We can write
\begin{align}
 \sup_{\lambda\in \R^d}\left(\lambda^\trp \tilde{S_t} - \fullBgOffset{t}(\lambda) \right) 
 &\stackrel{(i)}{=} \left(\fullBgOffset{t}\right)^\star(\Tilde{S_t}) \nonumber \\
 &\stackrel{\eqref{eq:stildedef}}{=} \left(\fullBgOffset{t}\right)^\star(\nabla Z^\nu_t( \ttilde_t ) - \nabla Z^\nu_t(\tstar)) \nonumber \\
 &\stackrel{\text{Lemma }\ref{lemma:dualitypatrick}}{=} D_{Z^\nu_t}(\tstar, \ttilde_t), \label{eq:dualityestimator}
 \end{align}
 where $(i)$ is simply the definition of the Legendre-Fenchel transform. 
 Why did we do all this work? Well, we are interested in the supremum in Equation~\eqref{eq:supremum}. It is sufficient to bound the supremum over all $\theta \in \Theta$ of terms of the form (see Equation~\eqref{eq:likelihoodconcentrationpart})
 $$
    \inner{\theta - \tstar}{S_t}.
 $$

 While we could do a covering type argument (carefully relaxing the i.i.d. data assumptions typical in empirical process theory), it is much easier to relate this supremum to the estimator via duality.

\newcommand{\concentrateevent}{\mathcal{E}_{\text{con}}}
With probability at least $1-\delta$, Proposition~\ref{proposition:saux:a1} gives us a high-probability time-uniform bound on
$$
    D_{Z^\nu_t}(\tstar, \ttilde_t) \leq \log(1/\delta) + A_0(\tstar) + \Gamma_t^\nu,
$$
and therefore, by plugging into Equation~\eqref{eq:dualityestimator} and making the reparametrization $\lambda = \gamma(\theta - \tstar)$ for some positive $\gamma$, it gives us
\begin{align*}
	\forall t \geq 0 \; \forall \theta \in \R^d \; \forall \gamma \in \R_+ \; :\; \gamma \tilde{S}_t^\trp (\theta - \tstar) - {\fullBgOffset{t}(\gamma(\theta - \tstar))} \leq \log(1/\delta) + A_0(\tstar) + \Gamma_t^\nu.
\end{align*}
Therefore, for all $t \geq 0$ and all $\theta \in \R^d$, the following holds:
\begin{align*}
    S_t^\trp (\theta - \tstar) &= \Tilde{S}_t^\trp (\theta - \tstar) + \nabla A_0(\tstar)^\trp(\theta - \tstar) \\
    &\leq \frac{1}{\gamma}\log(1/\delta) + \frac{1}{\gamma} A_0(\tstar) + \frac{1}{\gamma}\Gamma_t^\nu + \frac{1}{\gamma} \fullBgOffset{t}(\gamma(\theta - \tstar)) + \nabla A_0(\tstar)^\trp(\theta - \tstar).
\end{align*}
Since $A_0(\theta) = \frac{\nu}{2}\norm{\theta}_2^2$, restricting our uniform bound over $\theta \in \Theta$ gives us $\forall t \geq 0 \; \forall \theta \in \Theta$:
\begin{align*}
    S_t^\trp (\theta - \tstar) &\leq \frac{1}{\gamma}\log(1/\delta) + \frac{\nu}{2\gamma} B^2 + \frac{1}{\gamma}\Gamma_t^\nu + \frac{1}{\gamma} \fullBgOffset{t}(\gamma(\theta - \tstar)) + \nu B^2.
\end{align*}
Now, we note that under Assumption~\ref{ass:bounded}, Lemma~\ref{lemma:sublinarity} kicks in and we have for any $t \geq 0, \theta \in \Theta$ and $\gamma = \frac{\mu}{2L}$
\begin{align}
    S_t^\trp (\theta - \tstar)
    &\leq \frac{1}{\gamma}\log(1/\delta) + \frac{\nu}{2\gamma} B^2 + \frac{1}{\gamma}\Gamma_t^\nu + \frac{1}{2}\fullBgOffset{t}(\theta - \tstar) + \nu B^2.\label{eq:sublinearityapplied}
\end{align}
Finally, we can use this in \eqref{eq:likelihoodconcentrationpart} to obtain
\begin{align}
    \log\left(\frac{\lik^{(t)}(\theta)}{\lik^{(t)}(\tstar)}\right)
        &\leq  -\fullBgOffset{t}(\theta - \tstar) + \inner{\theta-\tstar}{S_t} \nonumber \\
        &\stackrel{\eqref{eq:sublinearityapplied}}{\leq} -\frac{1}{2}\fullBgOffset{t}(\theta - \tstar) + \frac{1}{\gamma}\log(1/\delta) + \frac{\nu}{2\gamma} B^2 + \frac{1}{\gamma}\Gamma_t^\nu  + \nu B^2 \nonumber \\
        &\leq -\frac{1}{2}\fullBgOffset{t}(\theta - \tstar) + \frac{2L}{\mu}\left(\log(1/\delta) + \frac{\nu B^2}{2} + \Gamma_t^\nu \right) + \nu B^2. \label{eq:firstratiofinal}
\end{align}
It remains to investigate the full likelihood ratio in \eqref{eq:fullinvertedratio}:
\begin{align}
    &\frac{1}{R_t(\theta)} - \log(\alpha) \nonumber \\
    = \; &\log\frac{\lik^{(t)}(\theta)}{\lik^{(t)}(\tstar)} + \log\frac{\lik^{(t)}(\tstar)}{\lik^{(t)}(\{\that_s\}_{s=1}^t)} + \log(1/\alpha) \nonumber \\
    \stackrel{\eqref{eq:firstratiofinal}\, \& \,\eqref{eq:see-regret}}{\leq} \; &-\frac{1}{2}\fullBgOffset{t}(\theta - \tstar) + \frac{2L}{\mu}\left(\log(1/\delta) + \frac{\nu B^2}{2} + \Gamma_t^\nu \right) + \nu B^2 + \log(1/\alpha) + \mathcal{R}_t.
\end{align}
Note that crucially for $\theta \in \Theta$, we have
$$\theta \not\in \mathcal{C}_t \iff \frac{1}{R_t(\theta)} - \log(\alpha) \leq 0.$$
This is implied by
$$
    \mathcal{B}_{Z^\nu_t}(\theta, \tstar) \geq \frac{4L}{\mu}\left(\log(1/\delta) + \frac{\nu B^2}{2} + \Gamma_t^\nu \right) + 2\nu B^2 + 2\log(1/\alpha) + 2\mathcal{R}_t,
$$
 or, since $L \geq \mu$, more compactly by
$$
    \mathcal{B}_{Z^\nu_t}(\theta, \tstar) \geq \frac{4L}{\mu}\left(\log(1/\delta) + {\nu B^2} + \Gamma_t^\nu \right) + 2\log(1/\alpha) + 2\mathcal{R}_t. 
$$
\subsection{Proof of Technical Lemmas}
\label{sec:appendix:a:lemmaproof}
First, we will prove Lemma~\ref{lemma:dualitypatrick}. The proof exactly follows \cite{chowdhury2022bregman}, we include it here for convenience because it is very short. 
\begin{proof}
    By definition
    \begin{align*}
        &\quad (D_{f,x})^*(\nabla f(y) - \nabla f(x)) \\
        &= \sup_{a \in \R^d}\left(\inner{a}{\nabla f(y) - \nabla f(x)} - D_{f,x}(a) \right) \\
        &= \sup_{a \in \R^d}\left(\inner{a}{\nabla f(y) - \nabla f(x)} - D_{f}(x+a, x) \right) \\
        &= \sup_{a \in \R^d}\left(\inner{a}{\nabla f(y) - \nabla f(x)} - f(x+a) + f(x) + \inner{\nabla f(x)}{a} \right) \\
        &= \sup_{a \in \R^d}\left(\inner{a}{\nabla f(y)} - f(x+a) + f(x) \right).
    \end{align*}
    Since $f$ is strictly convex and differentiable, first-order optimality conditions imply that the optimal $a$ satisfies $\nabla f(y) - \nabla f(x+a) = 0$ ($a$ is unconstrained). Since the gradient is invertible, we must have $a = y - x$. If we plug this into the above, we have
    \begin{align*}
    \quad (D_{f,x})^*(\nabla f(y) - \nabla f(x)) 
    &= \inner{y - x}{\nabla f(y)} - f(y) + f(x) \\ 
    &= f(x) - f(y) - \inner{\nabla f(y)}{x - y} \\
    &= D_f(x, y).
    \end{align*}
\end{proof}

Now we prove Lemma~\ref{lemma:sublinarity}. To this end, we will do a reduction to the one-dimensional case, and prove the one-dimensional result below.
\begin{lemma}
    \label{lemma:onedimensional}
	Under Assumption \ref{ass:bounded}, for any $a \in [B, B]$, any $\gamma \in (0, \frac{\mu}{2L}]$ and any $\Delta$ with $a + \Delta \in [B, B]$
	$$
	A(a+\gamma\Delta) - A(a) - A'(a)\gamma \Delta \leq \frac{1}{2} \gamma \left[A(a+\Delta) - A(a) - A'(a) \Delta\right].
	$$
\end{lemma}
We prove that this implies the desired sublinearity of the full Bregman difference.

\begin{proof}(of Lemma~\ref{lemma:sublinarity}).
    Let $\theta$ and $\lambda$ be such that $\theta, \theta + \lambda \in \Theta$. We will first show that for any $s \in \{0,\ldots, t\}$, $B_{A_s(\theta, \theta + \cdot)}$ is sublinear, and then the result follows by the linearity of the Bregman divergence. Define $a_s = x_s^\trp \theta$ and $\Delta_s = x_s^\trp \lambda$. Then we have $\abs{\Delta_s} = \abs{x_s^\trp(\lambda)} \leq \norm{x_s}\norm{\lambda} \leq \norm{x_s}(\norm{\theta} + \norm{\theta + \lambda} \leq (B+B) \leq 2B$. Similarly we have $\abs{a_s} \leq B$. Hence we satisfy the premise of Lemma~\ref{lemma:onedimensional} and we deduce that
	\begin{align*}
		D_{A_s}(\theta + \gamma \lambda, \theta) &= w_s A(x_s^\trp\theta + \gamma x_s^\trp\lambda) - w_s A(x_s^\trp \theta) + \inner{x_s w_s A'(x_s^\trp\theta)}{\gamma\lambda} \\
		&= w_s (A(a_s + \gamma {\Delta_s}) - A(a_s) + A'(a_s){\gamma \Delta_s}) \\
        &\leq \frac{w_s}{2} \gamma \left[A(a_s + \gamma {\Delta_s}) - A(a_s) + A'(a_s){\gamma \Delta_s}\right] \\
		&= \frac{1}{2} \gamma \left[w_s A({x_s^\trp\theta} + {x_s^\trp\lambda}) - w_s A(x_s^\trp \theta) + w_s A'(x_s^\trp\theta){x_s^\trp\lambda}\right] \\
        &= \frac{1}{2} \gamma D_{A_s}(\theta + \lambda, \theta).
	\end{align*}
    We also note that for $\gamma \leq \frac{\mu}{2L} \leq \frac{1}{2}$,
    \begin{equation}
    \label{eq:regularizersublinearity}
        D_{A_0}(\theta + \gamma \lambda, \theta) = \frac{\nu}{2}\norm{\gamma\lambda}^2 =  \frac{\gamma^2\nu}{2}\norm{\lambda}^2 \leq \frac{\gamma\nu}{4}\norm{\lambda}^2 = \frac{1}{2}\gamma D_{A_0}(\theta + \lambda, \theta).
    \end{equation}
    Therefore, by summing up the terms, we obtain
    \begin{equation*}
        \mathcal{B}_t(\gamma \lambda) \leq \frac{1}{2}\gamma \mathcal{B}_t(\lambda).
    \end{equation*}
\end{proof}

Then it remains to prove that Assumption~\ref{ass:bounded} implies Lemma~\ref{lemma:onedimensional}.
\begin{proof}(Lemma~\ref{lemma:onedimensional})
    $L$-Lipschitzness of $A'$ implies smoothness of $A$. Additionally, $\mu$'s existence implies strong convexity of $A$. With this, we can write for any $a$ and $\Delta$ with $a + \Delta \in [-B, B]$
    \begin{align*}
        &A(a + \Delta) \geq A(a) + A'(a) \Delta + \frac{\mu}{2} \Delta^2 \\
        \implies &A(a+\Delta) - A(a) - A'(a)\Delta \geq \frac{\mu}{2}\Delta^2.
    \end{align*}
    Similarly, 
    $$
        A(a+\gamma\Delta) - A(a) - A'(a)\gamma\Delta \leq \frac{L}{2}\gamma^2\Delta^2.
    $$
    Putting this together, we have 
    $$
    A(a+\gamma\Delta) - A(a) - A'(a)\gamma\Delta \leq \frac{L}{2}\gamma^2\Delta^2 = \frac{L \gamma^2}{\mu} \frac{\mu}{2} \Delta^2 \leq \frac{L \gamma}{\mu} \gamma [A(a+\Delta) - A(a) - A'(a)\Delta].
    $$
    The question is therefore: when is $\frac{L \gamma}{\mu} \leq \frac{1}{2}$? Clearly, choosing $\gamma_0 = \frac{\mu}{2L}$ makes $\frac{L \gamma}{\mu} \leq \frac{1}{2}$ for all $\gamma \leq \gamma_0$.
\end{proof}

\section{FTRL Results: Proofs}
\subsection{Technical Lemmas I: Exponential Families}

\begin{lemma} [MGF for Exponential family]\label{lemma:mgf}
	\[ \mE[\exp(T(y)u)|x] = \exp(A(\tstar^\top x + u)-A(\tstar^\top x )) .	 \]
\end{lemma}
\begin{proof}
	\begin{align*}
		 \mE[\exp(T(y)u)|x]  = & \int_y \exp(T(y)u)h(y)\exp(T(y)\tstar^\top x - A(\tstar^\top x))dy \\
		  = & \int \exp( T(y)(\tstar^\top x + u )) h(y) \exp(-A(\tstar^\top x)) \\  &\; \times \exp(-A(\tstar^\top x + u)) \exp(A(\tstar^\top x + u)) dy \\
		  = & \exp(A(\tstar^\top x + u)-A(\tstar^\top x )),
	\end{align*}
    where the last step follows because the density of a new exponential family distribution with parameter $\tstar^\trp x + u$ also integrates to 1.
\end{proof}

\subsection{Technical Lemmas II: Elliptical Potential Lemma}
We will repeatedly use instantiations of the following key lemma, known as the elliptical potential lemma. We will use the version from \citet{hazan:logarithmic}. Other variants are stated in \citet{abbasi:improved} or \citet{carpentier:potential}.
\begin{lemma}[Lemma 11 in \citet{hazan:logarithmic}]
	\label{lemma:ellipticalpotential}
	Let $u_s \in \R^d$ be a sequence of vectors such that $\norm{u_s} \leq r$. Define $\bar{\mathbf{V}}_t = \sum_{s=1}^t u_s u_s^\trp + \lambda \mathbf{I}$. Then
	$$
	\sum_{s=1}^t \norm{u_s}_{\bar{\VV}_s^{-1}}^2 \leq \log\left(\frac{\det \bar \VV_t}{\det \lambda \mathbf{I}}\right) \leq d\log\left(\frac{r^2 t}{\lambda} + 1\right).
	$$
\end{lemma}
We will also need a result where the time indices of the matrix are shifted. For this, note that if $\lambda \geq r^2$, then $u_s u_s^\trp \preceq r^2 \II \preceq \lambda \II$, and so we get $\Bar{\VV}_{s} \leq \Bar{\VV}_{s-1} + u_s u_s^\trp \preceq \Bar{\VV}_{s-1} + \lambda \II \preceq 2\Bar{\VV}_{s-1}$. Under our conditions, it follows that
$$
    \sum_{s=1}^t \norm{u_s}_{\bar{\VV}_{s-1}^{-1}}^2 \leq 2\sum_{s=1}^t \norm{u_s}_{\bar{\VV}_{s}^{-1}}^2
$$
\begin{corollary}
\label{cor:ellipticalpotential}
    We have the following bounds:
    \begin{align*}
        \gamma_t^\lambda = \log\left(\frac{\det(\sum_{s=1}\mu x_s x_s^\top + \lambda \II )}{\det(\lambda \II)} \right) \leq d\log\left(\frac{\mu t}{\lambda} + 1\right),
    \end{align*}
    and
    \begin{align*}
        \sum_{s=1}^t \norm{x_s}_{(\VV^{\mu;\lambda}_{s-1})^{-1}}^2 \leq \frac{2}{{\mu}} \gamma_t^\lambda.
    \end{align*}
\end{corollary}
\begin{proof}
The first bound is trivial by instantiating $u_s = \sqrt{\mu}x_s$. The second bound is by noting
\begin{equation*}
    \sum_{s=1}^t \norm{x_s}^2_{(\VV^{\mu;\lambda}_{s-1})^{-1}} = \frac{1}{{\mu}} \sum_{s=1}^t \norm{u_s}^2_{(\VV^{\mu;\lambda}_{s-1})^{-1}} \leq \frac{2}{{\mu}} \sum_{s=1}^t \norm{u_s}^2_{(\VV^{\mu;\lambda}_{s})^{-1}} \leq \frac{2}{{\mu}} \gamma_t^\lambda .
\end{equation*}
\end{proof}

\subsection{Technical Lemmas III: Supermartingales}
\begin{lemma}[Martingale Increment] \label{lemma:martingale}
Define the parametrized random processes
\begin{align*}
\mathcal{M}_j(r) & =  \exp( \nabla f_j(\tstar) ^\top r -  A'(x_j^\top \tstar)x_j^\top r  -A(x_j^\top \tstar - x_j^\top r) +A(x_j^\top \tstar))
\end{align*}
and 
\[\mathcal{N}_j(r) = \exp( \nabla f_j(\tstar)^\top r - \frac{L}{2} r^\top x_j x_j^\top r  ). \] Then, under Assumption \ref{ass:bounded} we have for any $r \in \R^d$ that $\mE[ \mathcal{M}_j(r) \sep \calF_{j-1} ] = 1$ and $\mE[ \mathcal{N}_j(r) \sep \calF_{j-1} ] \leq 1$.
\end{lemma}
\begin{proof}
	First, using the form of the exponential family and and recalling that $\nabla_\theta f_j(\theta) = -\nabla_\theta \log p_\theta(y_j \sep x_j) = \nabla_\theta[A(x_j^\trp \theta) - T(y_j) x_j^\trp \theta]$ we obtain
\begin{align*}
	&\quad \; \mE[\exp( \nabla f_j(\tstar) ^\top r  ) \sep \calF_{j-1}] \\
        &=  \int_y \exp( \nabla f_j(\tstar) ^\top r) \times  h(y) \exp(T(y)x_j^\top \tstar - A(x_j^\top \tstar)) dy \\
	&= \int_y \exp( -T(y) x_j^\top r +A'(x_j^\top \tstar)x_j^\top r ) \times h(y) \exp(T(y)x_j^\top \tstar - A(x_j^\top \tstar)) dy \\
	& =   \exp(A'(x_j^\top \tstar)x_j^\top r ) \underbrace{\int_y h(y) \exp(T(y)(x_j^\top \tstar - x_j^\top r))\exp(-A(x_j^\top \tstar - x_j^\top r))dy}_{ = 1} \\  & \quad \, \times \exp(A(x_j^\top \tstar - x_j^\top r))\exp(-A(x_j^\top \tstar)) \\
	   &=   \exp(A'(x_j^\top \tstar)x_j^\top r )\exp(A(x_j^\top \tstar - x_j^\top r))\exp(-A(x_j^\top \tstar)),
\end{align*}
which finishes the proof.  The second statement follows by using $L$-smoothness on the last equation and therefore noting that $\mathcal{N}_j(r) \leq \mathcal{M}_j(r)$.
\end{proof}

\begin{lemma}\label{lemma:bound-mixture}(Sequential Mixing)
	Define the martingale process, 
	\[ M_t(r_1, \dots r_t) = \prod_{s=1}^{t} \mathcal{N}_s(r_s), \]
	and recursively define the mixture martingale, 
	\[\bar{M}_s = \bar{M}_{s-1}\times \int_{r} \mathcal{N}_s(r) p_s(r) \mathrm{d}r,\]
	where $p_s$ is a probability distribution equal $\mathcal{N}(0, \mathbf{H}_s^{-1})$, $\mathbf{H}_s = \sum_{j=1}^{s-1} L x_j x_j^\top + \II \lambda \frac{L}{\mu}$, and $\bar{M}_0 = 1$. Then the following statements hold
	\begin{itemize}
		\item $\{\bar{M}_s\}_s$ is an adapted super-martingale with respect to the usual filtration. 
		\item $\bar{M}_t = \exp( \frac{\mu}{L} \sum_{s=1}^t \nabla f_s(\tstar)^\top (\VV^{\mu;\lambda}_s)^{-1}\nabla f_s(\tstar))\sqrt{\frac{\det(\II\lambda )}{\det(\VV_s^{\mu; \lambda})}}$.
	\end{itemize}
where 
\[\VV_s^{\mu;\lambda} =  \sum_{j=1}^{s} \mu x_j x_j^\top + \lambda\II.  \]
\end{lemma}
\begin{proof}
The first point follows from the fact that $p_s(r)$ is deterministic conditioned on the sub-$\sigma$-algebra $\calF_{s-1}$ (since $p_s$ makes use of $x_s$ but not $x_{s+1}$). Therefore, under mild regularity conditions  
\begin{align*}
    \E[\bar M_s \sep \mathcal{F}_{s-1}] = \E\left[\bar M_{s-1} \int p_s(r)\mathcal{N}_s(r) d r \sep \mathcal{F}_{s-1}\right] = \bar M_{s-1} \int_r p_s(r) \E[\mathcal{N}_s(r) \sep \mathcal{F}_{s-1}] dr \leq \bar M_{s-1}.
\end{align*}
In other words, mixing does not affect the supermartingale properties.
For the second point, we derive an explicit form of the mixture martingale. Note that we can write out
\begin{align}
    \int_{r} \mathcal{N}_s(r) p_s(r) \mathrm{d}r = \frac{1}{\sqrt{(2\pi)^{d} \det(\mathbf{H}_s^{-1})}}\int_r \exp\left(\nabla f_s(\tstar)^\top r - \frac{L}{2}\norm{r}^2_{x_s x_s^\trp} - \frac{1}{2}r^\top \mathbf{H}_s r \right) \label{eq:lemmasequentialmixing:gaussianform} \mathrm{d}r.
\end{align}
We can complete the square to obtain
\begin{align*}
    & \quad \; \nabla f_s(\tstar)^\top r - \frac{L}{2}\norm{r}^2_{x_s x_s^\trp} - \frac{1}{2}r^\top \mathbf{H}_s r \nonumber \\
    &= \frac{1}{2}\norm{ \nabla f_s(\tstar)}_{(\mathbf{H}_s +L x_s x_s^\top)^{-1}}^2 - \frac{1}{2} \norm{r - (\mathbf{H}_s + Lx_s x_s^\top)^{-1} \nabla f_s(\tstar)}^2_{\mathbf{H}_s + Lx_s x_s^\top}.
\end{align*}
The second term is the exponent of a exponent of a Gaussian integral with covariance $\mathbf{H}_{s+1}^{-1}$, and therefore results in 
\begin{equation*}
    \int_r \exp\left( - \frac{1}{2} \norm{r - (\mathbf{H}_s + Lx_s x_s^\top)^{-1} \nabla f_s(\tstar)}^2_{\mathbf{H}_s + Lx_s x_s^\top} \right) dr = \sqrt{(2\pi)^{d} \det(\mathbf{H}_{s+1}^{-1})}.
\end{equation*}
Plugging this into \eqref{eq:lemmasequentialmixing:gaussianform} we get
\begin{align*}
\int_{r} \mathcal{N}_s(r) p_s(r) \mathrm{d}r = \sqrt{\frac{\det \mathbf{H}_s}{\det \mathbf{H}_{s+1}}} \exp\left( \frac{1}{2}\norm{ \nabla f_s(\tstar)}_{\mathbf{H}_{s+1}^{-1}}^2 \right).
\end{align*}
By multiplying the individual steps, we can see that the determinant terms cancel in a telescoping product. This leads to the formulation
\begin{equation*}
    \bar M_t = \exp\left(\sum_{s=1}^t \nabla f_s(\tstar)^\trp \mathbf{H}_{s+1}^{-1} \nabla f_s(\tstar) \right)\sqrt{\frac{\det \frac{\lambda L}{\mu}\mathbf{I}}{\det \mathbf{H}_{t+1}}}.
\end{equation*}
To conclude the proof, note that 
$\mathbf{H}_{s+1} = \frac{L}{\mu}\mathbf{V}^{\mu;\lambda}_s$.
\end{proof}

\begin{lemma} \label{lemma:ville-bound} Under assumption of Lemma \ref{lemma:bound-mixture}, 
\begin{equation}\label{eq:ville-bound}
\mathbb{P} \left(\sum_{s=1}^t \norm{\nabla f_s(\tstar)}_{(\VV^{\mu;\lambda}_s)^{-1}}^2 \leq \frac{L}{\mu}\log\left(\frac{\det(\VV^{\mu;\lambda}_s)}{\det(\II \lambda)}\right) + \frac{L}{\mu}\log\left(\frac{1}{\delta}\right) \right) \leq \delta.
\end{equation}
with probability $1-\delta$. 
\end{lemma}
\begin{proof}The statement, follows by applying Ville's inequality for supermartingales, applying the logarithm, and rearranging. Namely, 
\[ \mathbb{P}(\bar{M}_t \geq \delta) = \mathbb{P}(\log(\bar{M}_t) \geq \log(\delta)) \leq \delta. \]
\end{proof}
The following results allow us to upper bound the weighted regret by the unweighted regret:
\begin{lemma}[Weighting Reduction]\label{lemma:weight}
Let $\{\theta_s\}_{s=1}^t$ be a sequence of vectors adapted to the filtration $\{\mathcal{F}_{s-1}\}_s$. Define
\[ \Delta_t(\{\theta_s\}) = \sum_{s=1}^t w_s(f_s(\theta_s) - f_s(\tstar)) - f_s(\theta_s) + f_s(\tstar) = \sum_{s=1}^t (1-w_s)(f_s(\tstar) - f_s(\theta_s)). \] Then, 
$P_t = \exp(\Delta_t(\{\theta_s\}_s))$ is a non-negative super-martingale for any choice of adapted $\{w_s\}$, and hence, 
\[ \sum_{s=1}^t w_s(f_s(\theta_s) - f_s(\tstar)) \leq \sum_{s=1}^t (f_s(\theta_s) - f_s(\tstar)) + \log\left(\frac{1}{\delta}\right) \]
with probability $1-\delta$ for all $t\geq 0$. 
\end{lemma}

\begin{proof}
\begin{eqnarray*}
    \mE[P_t \sep \calF_{t-1}] & = & \mE_{\tstar}\left[\exp\left( \sum_{s=1}^t -(1-w_s)f_s(\theta_s) + (1-w_s)f_s(\tstar)\right) \Big\vert \mathcal{F}_{t-1} \right] \\
    & = & P_{t-1} \E_{y_t \sim \mathbb{P}_\star}  \exp( -(1-w_t) f_t(\theta_t) + (1-w_t) f_t(\tstar)) \\
    & = & P_{t-1} \int_{y_t}  \exp( -(1-w_t) f_t(\theta_t) + (1- w_t) f_t(\tstar)) \exp(-f_t(\tstar)) dy_t \\
    & = & P_{t-1} \int_{y_t}  \exp( -(1-w_t) f_t(\theta_t) - w_t f_t(\tstar)) dy_t \\
    & = & P_{t-1} \int_{y_t}  p_{\theta_t}(y_t \sep x_t)^{1-w_t} p_{\tstar}(y_t \sep x_t)^{w_t}  dy_t \\
    & = & P_{t-1}\exp(-(1-w_t)D_{w_t}(\tstar, \theta_t)) \leq P_{t-1}.
\end{eqnarray*}
    We have used here the definition of the Renyi-divergence and the fact that it is always non-negative, namely
    $$
        D_w(\theta_1, \theta_2) = \frac{1}{w-1} \log\int_y p_{\theta_1}(y\sep x)^{1-w} p_{\theta_1}(y\sep x)^{w} dy \geq 0,
    $$
    for $0<w \not= 1$.\footnote{The case $w_t = 1$ is trivial for us.}
    The rest follows by the application of Ville's inequality. 
\end{proof}

\subsection{FTRL Proof: the Unweighted Case}\label{app:proof-unweighted}
\begin{proof}[Proof of Theorem \ref{thm:ftrl-main} (first part)]
	We define the function that FTRL minimizes in each step (to pick~$\that_t$) as,
    $g_t(\theta) = \sum_{s=1}^{t-1} -\log p_\theta(y_s \sep x_s) + \lambda ||\theta||_2^2$. We can rewrite this objective as
	\[\that_t = \arg\min_{\theta \in \Theta} g_t(\theta) = \arg\min_{\theta \in \Theta}\sum_{s=1}^{t-1} f_s(\theta) + \lambda \norm{\theta}_2^2 = \arg\min_{\theta \in \Theta}\sum_{s=1}^{t-1} m_s(\theta) + \phi_{t}(\theta), \]

   where we recall that\footnote{The $\log h(y_s)$ term does not play any role in the regret nor the FTRL objective.}
    \[f_s(\theta) = A(x_s^\top \theta) - T(y_s) x_s^\top \theta - \log h(y_s),\]
    and we have introduced the shorthands
    $$ m_s(\theta) = - T(y_s)x_s^\top \theta$$
    and
    \[\phi_t(\theta) =
      \sum_{s=1}^{t-1} A(\theta^\top x_s) + \lambda||\theta||_2^2. \]
In essence, we have shifted some of the objective into what is commonly looked at as the regularizer. By a standard telescoping sum argument, we obtain for any $u$
	\begin{eqnarray*}
		& & \sum_{s=1}^t (m_s(\that_s) - m_s(u)) \\
		& =&  \phi_{t+1}(u) - \min_\theta \phi_{1}(\theta) + \sum_{s=1}^t[g_s(\that_s) - g_{s+1}(\that_{s+1}) + m_s(\that_s)] + \underbrace{g_{t+1}(\that_{t+1}) - g_{t+1}(u)}_{\leq 0 } \\
		&\leq & \phi_{t+1}(u) + \sum_{s=1}^t[g_s(\that_s) - g_{s+1}(\that_{s+1}) + m_s(\that_s)] \\ 
		&= & \phi_{t+1}(u) + \sum_{s=1}^t[g_s(\that_s) - g_{s+1}(\that_{s+1}) + g_{s+1}(\that_s) - \phi_{s+1}(\that_s) -g_s(\that_s) + \phi_s(\that_s)] \\ &=& \phi_{t+1}(u) + \sum_{s=1}^t[g_{s+1}(\that_s) - g_{s+1}(\that_{s+1}) - \phi_{s+1}(\that_s) + \phi_s(\that_s)].
	\end{eqnarray*}
	
	 Now we use the strong-convexity of $g_{s+1}$ under the norm $\norm{\cdot}_{\VV_s^{\mu;\lambda}}$ where $\VV_s^{\mu;\lambda} = \sum_{j=1}^{s} \mu x_sx_s^\top + \lambda \II $, 
	\begin{eqnarray*}
        &  & 				\sum_{s=1}^t (m_s(\that_s) - m_s(u)) \\
		& \leq & \phi_{t+1}(u) + \sum_{s=1}^t [ (\that_{s} - \that_{s+1})^\top \nabla g_{s+1}(\that_{s}) \\ & & - \frac{1}{2}(\that_{s}-\that_{s+1} )^\top \VV_s^{\mu;\lambda} (\that_{s} - \that_{s+1}) - \phi_{s+1}(\that_s) + \phi_s(\that_s)]\\ 
		& \leq & \phi_{t+1}(u) + \sum_{s=1}^t [ (\that_{s} - \that_{s+1})^\top \nabla f_{s}(\that_{s}) \\ & & -  \frac{1}{2}(\that_{s}-\that_{s+1} )^\top \VV_s^{\mu;\lambda} (\that_{s} - \that_{s+1}) - \phi_{s+1}(\that_s) + \phi_s(\that_s) ] \\
		& \leq & \phi_{t+1}(u)+ \sum_{s=1}^t \left[\frac{1}{2}\norm{\nabla f_s(\that_s)}_{(\VV_s^{\mu;\lambda})^{-1}}^2 - \phi_{s+1}(\that_s) + \phi_s(\that_s) \right],
	\end{eqnarray*}
	where in the second inequality we used that $\nabla g_{s} (\that_s)^\top (x - \that_s) \geq 0$ due to the first-order optimality conditions for convex constrained minimization. Lastly, we optimized the resulting quadratic function over $\that_{s+1}$ (over $\R^d$) to get a worst case bound involving the dual-norm.
 
    Note that for the shorthands we defined above: \[\sum_{s=1}^t [-\phi_{s+1}(\that_s) + \phi_s(\that_s)] = \sum_{s=1}^t-A(\that_s ^\top x_{s}). \]
    Using our previous observations and the definition of $\phi_{t+1}(\tstar)$, we get for the overall regret: 
    \begin{align*}
        &\quad \; \mathcal{R}_t 
        \\ &= \sum_{s=1}^t f_s(\that_s) - f_s(\tstar) \notag \\
         &=   \sum_{s=1}^t 
 m_s(\that_s) - m_s(\tstar) + \sum_{s=1}^t  A(x_s^\top \that_s) - A(x_s^\top \tstar)  \notag \\
  &=   \sum_{s=1}^t A(x_s^\top \tstar) - A(x_s^\top \that_s) + \sum_{s=1}^t  A(x_s^\top \that_s) - A(x_s^\top \tstar) + \frac{1}{2}\sum_{s=1}^t \norm{\nabla f_s(\that_s)}_{(\VV_s^{\mu;\lambda})^{-1}}^2 + \lambda \norm{\tstar}^2 \notag \\
  &\leq   \frac{1}{2}\sum_{s=1}^t \norm{\nabla f_s(\that_s)}_{(\VV_s^{\mu;\lambda})^{-1}}^2 + \lambda \norm{\tstar}^2   \label{eq:proof-for-later} \\
&\leq  \frac{1}{2}\sum_{s=1}^t \norm{ T(y_s)x_s - A'(x_s^\top \that_s)x_s}_{(\VV_s^{\mu;\lambda})^{-1}}^2 + \lambda \norm{\tstar}^2   \notag \\
 &\leq  \sum_{s=1}^t \left[  \norm{ T(y_s)x_s - A'(x_s^\top \tstar)x_s}_{(\VV_s^{^\mu;\lambda})^{-1}}^2 +  \norm{ (A'(x_s^\top \that_s) - A'(x_s^\top \tstar))x_s}_{(\VV_s^{\mu;\lambda})^{-1}}^2 \right] + \lambda \norm{\tstar}^2  \\
  &\leq  \sum_{s=1}^t \left[ \norm{ T(y_s)x_s - A'(x_s^\top \tstar)x_s}_{(\VV_s^{^\mu;\lambda})^{-1}}^2 + 2 L^2B^2\norm{x_s}_{(\VV_s^{\mu;\lambda})^{-1}}^2 \right] + \lambda \norm{\tstar}^2  \notag \\
      &\leq  \sum_{s=1}^t \left[ \norm{\nabla f_s(\tstar)}_{(\VV_s^{\mu;\lambda})^{-1}}^2 + 2  L^2B^2\norm{x_s}_{(\VV_s^{\mu;\lambda})^{-1}}^2 \right] + \lambda B^2 \notag \\
   &\leq  \lambda B^2 + \frac{L}{\mu} \left(\gamma_t^\lambda + \log\left(\frac{1}{\delta}\right)\right) + \sum_{s=1}^t 2 L^2B^2\norm{x_s}_{(\VV_s^{\mu;\lambda})^{-1}}^2 \notag \\ &\leq  \lambda B^2 + \frac{L}{\mu} \left(\gamma_t^\lambda + \log\left(\frac{1}{\delta}\right)\right) + \frac{2 L^2B^2}{\mu}\gamma_t^\lambda .\notag 
 \end{align*}
 The last line follows because of {Lemma}~\ref{lemma:ellipticalpotential}, and the second to last one follows because of Lemma~\ref{lemma:ville-bound}. 
 Notice that if we wish to deal with arbitrary weights $\{w_t\}$, we can simply resort to Lemma~\ref{lemma:weight} and bound the weighted case with the unweighted case. In that case, we incur an additional additive $\log(1\delta)$ term.
 \end{proof}

\subsection{FTRL Analysis: the Weighted Case (Vovk-Azoury-Warmuth Forecaster)} \label{app:proof-ftrl-azoury}

\begin{proof}
We define the function that FTRL minimizes in each step (to pick $\that_t$) as
    $\tilde{g}_t(\theta) = \sum_{s=1}^{t-1} [A(x_s^\top \theta) - T(y_s) x_s^\top \theta] + \psi_t(\theta)$ with $\psi_t(\theta) = A(x_t^\trp \theta) + \lambda \norm{\theta}_2^2$. We can rewrite this objective as
	\[\that_t = \arg\min_{\theta \in \Theta} \tilde{g}_t(\theta) = \arg\min_{\theta \in \Theta}\sum_{s=1}^{t-1} m_s(\theta) + \phi_{t}(\theta), \]
   by introducing the shorthands
    $$ m_s(\theta) = - T(y_s)x_s^\top \theta$$
    and (notice the difference in time index of the second sum when compared to the proof in the previous subsection): 
    \[\phi_t(\theta) =
     \sum_{s=1}^{t} A(\theta^\top x_s) + \lambda||\theta||_2^2.\]
In addition consider the objective $g_t$ from the classical FTRL analysis in Section~\ref{app:proof-unweighted}. It is not used to run the online algorithm, but is helpful in our analysis. With our new components, it is equal to
\[g_t(\theta) = \sum_{s=1}^{t-1} m_s(\theta) + \sum_{s=1}^{t-1}  A(\theta^\top x_s) + \lambda||\theta||_2^2 = \sum_{s=1}^{t-1} m_s(\theta) + \phi_{t-1}(\theta), \]
and its minimizer is $\tbar_t = \arg\min_{\theta \in \Theta} g_t(\theta)$. Also, consider a weighted version of the regularizer
$$\bar{\phi}_t(\theta) = \sum_{s=1}^t w_s A(x_s^\top \theta) + \lambda ||\theta||_2^2,
$$
which will be useful. We use a variant of a similar telescoping sum argument as in the previous proof of Section~\ref{app:proof-unweighted}. We specifically use $\tstar$ as the comparator to compete against. Notice that we insert a telescoping sum involving the objective $g_s$, which is not the objective that our estimator is minimizing:
	\begin{eqnarray*}
		& & \sum_{s=1}^t w_s(m_s(\that_s) -  m_s(\tstar)) \\
		& \stackrel{(*)}{=} &  \bar{\phi}_{t}(\tstar) - \phi_{0}(\bar{\theta}_1) + \sum_{s=1}^t[{g}_s(\tbar_s) - {g}_{s+1}(\tbar_{s+1}) + w_s m_s(\that_s)] + \underbrace{{g}_{t+1}(\tbar_{t+1}) - {g}_{t+1}(\tstar)}_{\leq 0} \\ & & + \sum_{s=1}^t(1-w_s)f_s(\tstar) \\
		&\leq & \bar{\phi}_{t}(\tstar) + \sum_{s=1}^t[w_s({g}_s(\tbar_s) - {g}_{s+1}(\tbar_{s+1})) + w_s m_s(\that_s)]   \\ & &  + \sum_{s=1}^t(1-w_s)(g_s(\tbar_s) - g_{s+1}(\tbar_{s+1}) + f_s(\tstar))   \\ 
	&\stackrel{(**)}{=} & \bar{\phi}_{t}(\tstar) + \sum_{s=1}^t w_s[{g}_s(\tbar_s) - {g}_{s+1}(\tbar_{s+1}) + {g}_{s+1}(\that_s) - {\phi}_{s}(\that_s) - {g}_s(\that_s) + {\phi}_{s-1}(\that_s)] \\ & &  \sum_{s=1}^t (1-w_s) [g_s(\tbar_s) - g_{s+1}(\tbar_{s+1}) + f_s(\tstar)] \\
 & \stackrel{(***)}{\leq} & \bar{\phi}_{t}(\tstar) + \sum_{s=1}^tw_s[ - {g}_{s+1}(\tbar_{s+1}) + {g}_{s+1}(\that_s) - {\phi}_{s}(\that_s) + {\phi}_{s-1}(\that_s)] + \tilde{\Delta}_t  .
	\end{eqnarray*}
 In $(*)$, we used the shorthands and definitions introduced above. In $(**)$, we used the identity ${g}_{s+1}(\that_s) - {\phi}_{s}(\that_s) - {g}_s(\that_s) + {\phi}_{s-1}(\that_s) = m_s(\that_s)$. Finally, for $(***)$, recall that $\tbar_s$ is the minimizer of $g_s$, and hence, $g_s(\tbar_s) - g_s(\that_s) \leq 0$. Next, define $\tilde{\Delta}_t =\sum_{s=1}^t (1-w_s) [g_s(\tbar_s) - g_{s+1}(\tbar_{s+1}) + f_s(\tstar)]$. We will bound this term later. 
    
    Now, we use the strong-convexity of $g_{s+1}(\theta)$ under the norm $\norm{\cdot}_{\VV_s^{\mu;\lambda}}$ where $\VV_s^{\mu;\lambda} = \sum_{j=1}^{s} \mu x_jx_j^\top +  \lambda \II $, namely
    $$
        g_{s+1}(\tbar_{s+1}) \geq g_{s+1}(\that_{s+1}) + \nabla g_{s+1}(\that_{s})^\trp (\tbar_{s+1} - \that_{s}) + \frac{1}{2}\norm{\tbar_{s+1}-\that_{s}}_{\VV_s^{\mu;\lambda}}^2.
    $$
    We can then proceed as follows:
	\begin{eqnarray}
        & & 	\sum_{s=1}^t w_s(m_s(\that_s) - m_s(u)) \nonumber \\
		& \leq & \tilde{\Delta}_t + \bar{\phi}_{t}(\tstar) + \sum_{s=1}^t w_s[ \nabla g_{s+1}(\that_{s})^\trp(\that_{s} - \tbar_{s+1}) \nonumber \\ & & - \frac{1}{2}\norm{\tbar_{s+1}-\that_{s}}_{\VV_s^{\mu;\lambda}}^2 - {\phi}_{s}(\that_s) + {\phi}_{s-1}(\that_s)  ] \nonumber \\ 
      & \leq & \tilde{\Delta}_t + \bar{\phi}_{t}(\tstar) + \sum_{s=1}^t w_s[ (\nabla \tilde{g}_{s}(\that_{s}) + \nabla m_s(\that_s))^\trp(\that_{s} - \tbar_{s+1}) \nonumber \\ & & - \frac{1}{2}\norm{\tbar_{s+1}-\that_{s}}_{\VV_s^{\mu;\lambda}}^2 - {\phi}_{s}(\that_s) + {\phi}_{s-1}(\that_s)  ] \nonumber \\
		& \leq & \tilde{\Delta}_t+ \bar{\phi}_{t}(\tstar)+ \frac{1}{2}\sum_{s=1}^t w_s \norm{ \nabla m_s(\that_s)}_{(\VV_s^{^\mu })^{-1}}^2 - w_s({\phi}_{s}(\that_s) + {\phi}_{s-1}(\that_s) ), \label{eq:azourywarmuth:mterms}
	\end{eqnarray}
	where in the second to last line we used that $\nabla \tilde{g}_{s} (\that_s)^\top (x-\that_s) \geq 0$ for any $x$, due to the optimality of $\that_s$ for the FTRL objective. In the last line, we optimized over $\tbar_{s+1}$ to get a worst-case bound on the quadratic function involving it.
  Also, note that for the shorthands we defined above: 
    \[\sum_{s=1}^t w_s(-{\phi}_{s}(\that_s) + {\phi}_{s-1}(\that_s)  ) = \sum_{s=1}^t w_s(-A(\that_s ^\top x_{s})). \]

    Our goal here is to upper bound the overall regret: 
    \begin{eqnarray}
        \mathcal{R}_t &= & \sum_{s=1}^t w_s(f_s(\that_s) - f_s(\tstar)) \notag \\
        & =  & \sum_{s=1}^t 
 w_s(m_s(\that_s) - m_s(\tstar)) + \sum_{s=1}^t  w_s(A(x_s^\top \that_s) - A(x_s^\top \tstar)) \notag  \nonumber \\
 & \leq  & \bar{\phi}_{t}(\tstar) - \sum_{s=1}^t w_s A(x_s^\top \that_s) + \frac{1}{2}\sum_{s=1}^t w_s \norm{\nabla  m_s(\that_s)}_{(\VV_s^{^\mu})^{-1}}^2 \notag + \tilde{\Delta}_t \nonumber \\
 & &  +  \sum_{s=1}^t  w_s(A(x_s^\top \that_s) - A(x_s^\top \tstar)) \nonumber \\
 &  =  & \sum_{s=1}^t w_s (A(x_s^\top \tstar)) - \sum_{s=1}^t w_s A(x_s^\top \that_s)  + \frac{1}{2}\sum_{s=1}^t w_s \norm{\nabla  m_s(\that_s)}_{(\VV_s^{^\mu})^{-1}}^2 + \lambda \norm{\tstar}^2 \notag + \tilde{\Delta}_t  \nonumber \\ & &  + \sum_{s=1}^t  w_s(A(x_s^\top \that_s ) - A(x_s^\top \tstar)) \nonumber \\
  & = &  \tilde{\Delta}_t+ \frac{1}{2}\sum_{s=1}^t w_s \norm{ \nabla m_s(\that_s)}_{(\VV_s^{^\mu})^{-1}}^2 + \lambda \norm{\tstar}^2.  \label{eq:azourywarmuth:firstequationregret}
\end{eqnarray}
The first inequality follows by plugging in \eqref{eq:azourywarmuth:mterms}.

We return back to the term $\tilde{\Delta}_t$, 
\begin{align}
 \tilde{\Delta}_t & = \sum_{s=1}^t (1-w_s) [g_s(\tbar_s) - g_{s+1}(\tbar_{s+1}) + f_s(\tstar)] \nonumber \\ 
 & = \sum_{s=1}^t (1-w_s) [g_s(\tbar_s) - g_{s}(\tbar_{s+1}) - f_s(\tbar_{s+1}) + f_s(\tstar)] \nonumber \\ 
 & \leq  \sum_{s=1}^t (1-w_s) (f_s(\tstar) - f_s(\tbar_{s+1}))  \nonumber \\
 &\leq \frac{L}{\mu}(\gamma_t^\lambda + \log(1/\delta)), \label{eq:bytheassumptionwemade}
\end{align}
where the second to last line is by the optimality of $\tbar_s$ and the last one by the assumption in the theorem. 

Carrying on with the analysis, i.e. with \eqref{eq:azourywarmuth:firstequationregret}, we insert the definition of $m_s$ and obtain
  \begin{align}
  &\quad \mathcal{R}_t \nonumber \\
 &\leq   \Delta_t+ \frac{1}{2}\sum_{s=1}^t w_s \norm{-T(y_s)x_s}_{(\VV_s^{^\mu;\lambda})^{-1}}^2 + \lambda \norm{\tstar}^2 \nonumber \\
 &\leq  \Delta_t+ \sum_{s=1}^t \left[w_s \norm{ T(y_s)x_s - A'(x_s^\top \tstar)x_s}_{(\VV_s^{^\mu;\lambda})^{-1}}^2 + w_s\norm{ (A'(x_s^\top \tstar))x_s}_{(\VV_s^{^\mu;\lambda})^{-1}}^2\right] + \lambda \norm{\tstar}^2   \notag \\
 &\stackrel{(*)}{\leq}  \Delta_t+ \sum_{s=1}^t \norm{ T(y_s)x_s - A'(x_s^\top \tstar)x_s}_{(\VV_s^{^\mu;\lambda})^{-1}}^2 + L^2B^2 \sum_{s=1}^t w_s\norm{x_s}_{(\VV_s^{^\mu;\lambda})^{-1}}^2 + \lambda \norm{\tstar}^2  \notag \\
     &\stackrel{(**)}{=}  \Delta_t+ \sum_{s=1}^t \norm{\nabla f_s(\tstar)}_{(\VV_s^{^\mu})^{-1}}^2 + L \sum_{s=1}^t\frac{B^2}{1/L+\operatorname{bias}^2_{x_s}(\that_s)} \norm{x_s}_{(\VV_s^{^\mu;\lambda})^{-1}}^2 + \lambda B^2 \notag \\
 &\stackrel{(***)}{\leq}  \lambda B^2 + \frac{2L}{\mu} \left(\gamma_t^\lambda + \log\left(\frac{1}{\delta}\right)\right) + L \sum_{s=1}^t\frac{B^2}{1/L+\operatorname{bias}^2_{x_s}(\that_s)} \norm{x_s}_{(\VV_s^{^\mu;\lambda})^{-1}}^2 . \notag
 \end{align}
 In $(*)$, we use $w_s \leq 1$ and the Lipschitzness of $A'$. In $(**)$, we use the definition of the weights. Finally, in ($*$$*$$*$), we used Lemma~\ref{lemma:ville-bound} and \eqref{eq:bytheassumptionwemade}. By substituting  $\Delta\gamma_s = \mu \norm{x_s}_{(\VV_s^{^\mu;\lambda})^{-1}}^2$, we finish the proof. The event in Lemma~\ref{lemma:ville-bound} holds with probability $1-\delta$, completing the proof.
 \end{proof}

\subsection{FTRL Analysis: Beyond Global Smoothness}\label{app:beyond-smooth}
In this subsection, we give alternative analysis which avoids the necessity to impose a global smoothness condition our likelihood; instead strong convexity within a bounded domain suffices, and we will only assume that
$$
\epsilon_s := \E_{x_s^\trp \tstar }[T(y_s)] -  T(y_s)
$$
are sub-Exponential random variables, setting us apart from \cite{zhao:ftrl} which assume sub-Gaussianity.
In particular, we can show the following theorem
\begin{theorem}
\label{app:theo:ftrl:crude}
    With probability $1-\delta$, uniformly over time $t\in \N$, we have 
    $$
        \mathcal{R}_t \leq c d\log^2(t/\delta))\log(t),
    $$
    where the universal constant $c$ hides all constants independent of $t,d$ and $\delta$.
\end{theorem}

\subsubsection{Lemmas}
\label{section:b1:lemmaproof}
\newcommand{\Vsinvnorm}[1]{\norm{#1}_{(\VV^{\mu;\lambda}_s)^{-1}}}

We state the following result on sub-Exponential random variables.
\begin{proposition}[Theorem 2.13 in \citet{wainwright:book}]
\label{proposition:subexponentialvariables}
   If $X$ is a centered sub-Exponential variable with some finite variance proxy, then there exist constants $c_1,c_2>0$ such that for any $t >0 $ $$
   \P{\abs{X} \geq a} \leq c_1e^{-c_2a}.
   $$
\end{proposition}
By some careful union bounds (akin to a stitching argument), we can also provide upper bounds on anytime-valid upper bounds on the process $S_t = \max_{s\leq t} \epsilon_s$.
\begin{lemma}
   \label{lemma:subexponential:unionbound}
   For any sequence $(\epsilon_s)_{s=1}^\infty$ of sub-Exponential-variables, there exists a constant $\Tilde{c}$ independent of $t$ such that
   $$
       \P{\exists t: \; \max_{s\leq t} \abs{\epsilon_s} \geq \tilde{c}\log(s/\delta)} \leq \delta.
   $$
\end{lemma}
\begin{proof}(of Lemma~\ref{lemma:subexponential:unionbound})
    By Proposition~\ref{proposition:subexponentialvariables}, there exists $c_1,c_2 > 0$ such that we have 
    $\P{\abs{\epsilon_s} \geq a} \leq c_1e^{-c_2a}$ for any fixed $s$.
    Note that $c_1e^{-c_2a} \leq \delta$ is satisfied for 
    $a \geq \frac{1}{c_2}\log(c_1/\delta) =: c_3\log(c_1/\delta)$. Let us denote by $\mathcal{E}_i$ the event all $j \in [2^i, 2^{i+1}) \cap \N$ satisfy the inequality
    $$
     \abs{\epsilon_j} < c_3\log(c_1(2^{2i+1})/\delta).
    $$
    For a single $j$, this happens with probability at least $1-\delta/2^{2i+1}$. Therefore, by a union bound, as $\abs{[2^i, 2^{i+1}) \cap \N} = 2^i$, we can bound the probability of the complement, namely $\P{\mathcal{E}_i^c} \leq 2^i\frac{\delta}{2^{2i+1}} = \frac{\delta}{2^{i+1}}$. Now, by another union bound, we can conclude that 
    $$
        \P{\cup_{i=0}^\infty \mathcal{E}_i^c} \leq \sum_{i=0}^\infty \frac{\delta}{2^{i+1}} = \frac{\delta}{2}\frac{1}{1-\frac{1}{2}} = \delta.
    $$
    Now we also have for any $j$ in this range that 
    $
        2^{2i+1} \leq 2j^2,
    $
    and therefore, if $\mathcal{E}_i$ holds, we have for any $j \in [2^i, 2^{i+1}) \cap \N$:
    $$
        \epsilon_j \leq c_3\log(c_1(2j^2)/\delta) \leq 2c_3\log(2 c_1 j/\delta) \leq \tilde{c}\log(j/\delta) .
    $$
    We can immediately see that this implies 
    $$
        \P{\exists t: \; \max_{s\leq t} \abs{\epsilon_s} \geq \tilde{c}\log(s/\delta)} \leq \delta.
    $$
    as desired. 
\end{proof}

\subsubsection{Proof of Theorem \ref{app:theo:ftrl:crude}}
    Our proof initially follows the FTRL regret bound proofs in the adversarial setting \cite{hazan:book, orabona:book}. It also has overlap with the proof in \citet{zhao:ftrl}.
    We define the function that FTRL minimizes in each step as (to pick $\that_t$)
$$
    g_t(\theta) = \sum_{s=1}^{t-1} -\log p_\theta(y_s \sep x_s) + \phi(\theta)
$$
for convenience. 
We initially use the same steps as in Theorem~\ref{thm:ftrl-main} to see that for any $u \in \Theta$
\begin{align*}
&\quad \;\sum_{s=1}^t (f_s(\that_s) - f_s(u)) \\
&\leq \phi(u) - \min_\theta \phi(\theta) + \sum_{s=1}^t[g_s(\that_s) - g_{s+1}(\that_{s+1}) + f_s(\that_s)] + g_{t+1}(\that_{t+1}) - g_{t+1}(u) \\
&\leq \lambda B^2 + \sum_{s=1}^t[g_s(\that_s) - g_{s+1}(\that_{s+1}) + f_s(\that_s)]. \label{eq:orabona:71}
\end{align*}
Similarly to the proof of Theorem~\ref{thm:ftrl-main} in Appendix~\ref{app:proof-unweighted} we bound these increments by the dual norm of the gradient of the objective.
\begin{equation}
g_s(\that_s) - g_{s+1}(\that_{s+1}) + f_s(\that_s) \leq \frac{\norm{\nabla f_s(\that_s)}_{(\VV_s^{\mu;\lambda})^{-1}}^2}{2}.\label{eq:orabonaresult}
\end{equation}

Now we note that
$$
    \nabla f_s(\theta) = A'(x_s^\trp \theta) x_s - T(y_s)x_s.
$$
Using properties of the exponential family, we deduce that
\begin{align*}
\nabla f_s(\theta) &= \left(\E_{x_s^\trp \theta }[T(y_s)] - T(y_s)\right)x_s \\
&= \left(\E_{x_s^\trp \theta }[T(y_s)] - \E_{x_s^\trp \tstar }[T(y_s)] + \E_{x_s^\trp \tstar }[T(y_s)] -  T(y_s)\right)x_s. \\
\end{align*}
From here on out, we proceed more crudely than in our previous analyses, since we are only concerned with asymptotic behavior when $d$ and $t$ are large. Let us define 
$$
    U := \sup_{\theta \in \Theta} \abs{\E_{x_s^\trp \theta }[T(y_s)] - \E_{x_s^\trp \tstar }[T(y_s)]},
$$
which is a model-dependent, deterministic quantity. Let us define the noise variables
$$
    \epsilon_s := \E_{x_s^\trp \tstar }[T(y_s)] -  T(y_s).
$$
We bound
\begin{align*}
    \Vsinvnorm{\nabla f_s(\theta)}^2 &\leq 2(U^2 +  \epsilon_s^2) \Vsinvnorm{x_s}^2. 
\end{align*}
Note that the $\epsilon_s$ are centered, independent sub-Exponential variables, and as such are guaranteed to satisfy  
\begin{equation*}
     \mathbb{P}(\exists t: \; \max_{s\leq t} \abs{\epsilon_s} \geq c_4\log(s/\delta)) \leq \delta,
\end{equation*}
by Lemma~\ref{lemma:subexponential:unionbound}. This tells us that conditional on this event, we can upper bound for any $t$
$$
    \mathcal{R}_t \leq \lambda B^2 + \frac{1}{\mu}(U^2 + \Tilde{c}^2\log^2(t/\delta)) \sum_{s=1}^t \Vsinvnorm{x_s}^2.
$$
for some constant $\Tilde{c}$ independent of $t$. By Lemma~\ref{lemma:ellipticalpotential}, 
there is thus a constant $c'$ independent of $t$ and $d$ such that
$$
     \mathcal{R}_t \leq c'd\log^2(t/\delta))\log(t) = \mathcal{O}(d\log^3(t)),
$$
with probability $1-\delta$ uniformly over $t \in \N$.

\newcommand{\Abf}{\mathbf{A}}
\newcommand{\Ibf}{\mathbf{I}}
\newcommand{\Ainvnorm}[1]{\norm{#1}_{\Abf^{-1}}}

\section{Regret Consequences for Stochastic Linear Bandits}
\label{app:linearbandits}
As a corollary of our analysis, we provide the regret for stochastic linear bandits that use our confidence sets within the LinUCB algorithm.
\begin{proof}
\newcommand{\weightedVt}[1][t]{\mathbf{W}_{#1}^{\sigma^{-2};\nu}}
\newcommand{\unweightedVt}[1][t]{\mathbf{V}_{#1}^{\sigma^{-2};\nu}}
We proceed in two parts: first, we instantiate Theorem~\ref{theorem:bregmanball} and then we follow the classical regret analysis for stochastic linear bandits. 
\paragraph{Specializing the Bregman divergence results} 
By Theorem~\ref{theorem:bregmanball}, we know that for any $\nu > 0$, we have that with probability $1-\delta$, 
\begin{equation}
     D_{Z_t^\nu}(\theta, \tstar) \leq \frac{4L}{\mu} \xi_t + 2\log\left(\frac{1}{\delta}\right) + 2\mathcal{R}_t, \label{eq:theoremballinitial}
\end{equation}
for all $t$, 
where $ \xi_t = \left(\log\left(\frac{1}{\alpha}\right) + {\nu B^2} + \Gamma_t^\nu \right) $ and $\mathcal{R}_t$ is the online convex optimization regret. We also recall that
$$
    Z_t^\nu(\theta) = \sum_{s=1}^t w_s A(x_s^\trp \theta) + \frac{\nu}{2}\norm{\theta}^2_2.
$$
In the Gaussian case, where $A(z) = z^2/(2\sigma^2)$. This implies that
$$
    \nabla Z_t^\nu(\theta) = \sum_{s=1}^t \frac{w_s}{\sigma^2} x_s x_s^\trp \theta + \nu \theta = \weightedVt \theta,
$$
where we have defined a weighted version of $\VV_t^{\sigma^{-2};\nu}$ as 
$
    \weightedVt = \sum_{s=1}^t  \frac{w_s x_s x_s^\top}{\sigma^2} + \nu \II
$
and therefore the Bregman divergence is given by 
$
    D_{Z_t^\nu}(\theta, \tstar) = \frac{1}{2}\norm{\theta - \tstar}_{\weightedVt}^2.
$
We can also see that the Bregman information gain is given by
\begin{align*}
    \Gamma_t^\nu &= \log\left(\frac{\int_{\R^d} \exp(-\frac{1}{2}\norm{\theta}_2^2)\mathrm{d}\theta }{\int_{\R^d} \exp(-D_{Z_t^\nu}(\theta, \ttilde_t))\mathrm{d}\theta}\right)
    = \log\left(\frac{\int_{\R^d} \exp(-\frac{1}{2}\norm{\theta}_2)^2\mathrm{d}\theta }{\int_{\R^d} \exp(-\frac{1}{2}\norm{\theta - \tstar}_{\weightedVt}^2 )\mathrm{d}\theta}\right).
\end{align*}
These Gaussian integrals are straightforward to evaluate. We know that
$$
\int_{\R^d} \exp\left(-\frac{1}{2}\norm{\theta}^2_2\right)\mathrm{d}\theta = (2\pi)^{d/2} \sqrt{\det((\nu\mathbf{I}_d)^{-1})}.
$$
Similarly, 
\begin{align*}
\int_{\R^d} \exp\left(-\frac{1}{2}\norm{\theta - \tstar}_{\weightedVt}^2\right) \mathrm{d}\theta
=  (2\pi)^{d/2} \sqrt{\det((\weightedVt)^{-1})}.
\end{align*}
Then, we can compute 
\begin{align*}
    \Gamma_t^\nu = \log\left( \frac{\det(\weightedVt)}{\det(\nu \mathbf{I})}\right) = \log\left(\det\left(\sum_{s=1}^t \frac{w_s x_s x_s^\trp }{\sigma^2\nu} + \mathbf{I}\right)\right).
\end{align*}
In the unweighted case, with which we proceed, we have $\Gamma_t^\nu = \gamma_t^\nu$, that is we recover the classical upper bound on the information gain \citep{srinivas:noregret}.
To summarize, we have specialized the bound \eqref{eq:theoremballinitial} to say that for any $\theta \in \mathcal{C}_t$, we have (since $L = \mu = 1/\sigma^2$)
\begin{equation*}
\norm{\theta - \tstar}_{\unweightedVt}^2 \leq 8\left(\log(1/\alpha) + \nu B^2 + \gamma_t^\nu \right) + 4\log(3/\delta)) + 4\mathcal{R}_t.
\end{equation*}
Now, we instantiate the regret of the online learner using Theorem~\ref{thm:ftrl-main}. With probability $1-\delta$, uniformly over $t$, we have
\begin{equation}
\mathcal{R}_t \leq \lambda B^2 + \frac{L}{\mu} \left(\gamma_t^\lambda + \log\left(\frac{1}{\delta}\right)\right) + \frac{2 L^2B^2}{\mu}\gamma_t^\lambda. \label{eq:ftrlboundusage}
\end{equation}
We get by chosing $\alpha = \delta$, and setting $\nu = \lambda$ that
\renewcommand{\weightedVt}[1][t]{\mathbf{W}_{#1}^{\sigma^{-2};\lambda}}
\renewcommand{\unweightedVt}[1][t]{\mathbf{V}_{#1}^{\sigma^{-2};\lambda}}
\begin{align*}
&\quad \; \norm{\theta - \tstar}_{\unweightedVt}^2 \nonumber \\
&\leq  8\left(\log(1/\delta) + \nu B^2 + \gamma_t^\nu \right) + 4\log(1/\delta)) + 4\lambda B^2 + 4 (\gamma_t^\lambda + \log({1}/{\delta})) + \frac{8 B^2}{\sigma^2}\gamma_t^\lambda \nonumber \\
&\leq  16\log(1/\delta) + 12\lambda B^2 + 8\left(\frac{B^2}{\sigma^2}+1\right)\gamma_t^\lambda =: \beta_t.
\end{align*}

\paragraph{Linear bandit regret analysis} 
We are ready to proceed with the bandit analysis for the UCB Algorithm. We follow \cite{lattimore:book} and bound the pseudo-regret, letting $\xstar$ be the optimal action. We bound the instantaneous regret at step $0 \leq s \leq t$ as
\begin{align*}
r_s &= \inner{\tstar}{\xstar - x_s} \\
&= \inner{\tstar}{\xstar} - \inner{\tstar}{x_s} \\
&\leq \max_{\theta \in \mathcal{C}_{s-1}}\inner{\theta}{\xstar} - \inner{\tstar}{x_s} \\
&\leq \max_{x \in \mathcal{X}} \max_{\theta \in \mathcal{C}_{s-1}}\inner{\theta}{x} - \inner{\tstar}{x_s} \\
&\stackrel{(*)}{=} \max_{\theta \in \mathcal{C}_{s-1}}\inner{\theta}{x_s} - \inner{\tstar}{x_s} \\
&\stackrel{(**)}{\leq} \norm{\ttilde_s - \tstar}_{\unweightedVt[s-1]} \norm{x_s}_{(\unweightedVt[s-1])^{-1}} \\
&\;\;{\leq} \sqrt{\beta_{s-1}} \norm{x_s}_{(\unweightedVt[s-1])^{-1}} \\
& \; \leq \sqrt{\beta_{t}} \norm{x_s}_{(\unweightedVt[s-1])^{-1}}. \\
\end{align*}
The first inequality replaces $\tstar$ by the upper confidence bound for action $\xstar$, which is valid with probability $1-\alpha = 1-\delta$ uniformly over time. Then, $(*)$ uses the fact that $x_t$ is chosen to maximize the upper confidence bound. Finally $(**)$ defines the UCB parameter $\ttilde_t$.
By Corollary~\ref{cor:ellipticalpotential}, we have
\begin{equation*}
    \sum_{s=1}^t \norm{x_s}^2_{(\mathbf{V}_{s-1}^{\sigma^{-2};\lambda})^{-1}} \leq 2\sigma^2 \gamma_t^\lambda.
\end{equation*}
Plugging all this together and using an $\ell_1/\ell_2$-norm inequality, we get
\begin{align*}
    \mathfrak{R}_t &= \sum_{s=1}^t r_s \\
    &\leq \sqrt{\beta_t}\sqrt{t\sum_{s=1}^t \norm{x_s}^2_{(\unweightedVt[s-1])^{-1}}} \\
    &\leq \sqrt{2 t\beta_t \sigma^2\gamma_t^\lambda} \\
    &\leq 
    \sqrt{2\sigma^2 t (16\log(1/\delta) + 12\lambda B^2 + 8\gamma_t^\lambda + 8B^2/\sigma^2\gamma_t^\lambda) \gamma_t^\lambda} \\
    &\leq 6\sqrt{t\gamma_t}\left(\sigma\sqrt{\log(1/\delta) + \gamma_t^\lambda} + \sigma\lambda^{1/2} B +B\sqrt{\gamma_t^\lambda}\right).
\end{align*}
To summarize and to justify why this bound holds with probability $1-3\delta$ uniformly over time, note that we have bounded the probability of the FTRL bound~\eqref{eq:ftrlboundusage} not holding for some $t$ by $\delta$. Then, the probability of \eqref{eq:theoremballinitial} not holding for some $t$ is at most $\delta$. Finally, the anytime Type I error of our sets is also bounded by $\delta$. A union bound therefore concludes the proof.
\end{proof}

\subsection{Comparison to Abbasi-Yadkori et.\ al.\ (2011)}
We compare our result to the one from \citet{abbasi:improved}. Under the assumption that $\lambda \geq 1$, they show that the regret satisfies
$$
\mathfrak{R}_t \leq 4\sqrt{t d \log(\lambda + t/d)}\left(\sqrt{\lambda}B + \sigma \sqrt{2\log(1/\delta) + d \log(1+t/(\lambda d)}\right).
$$
Observe that there is a reparametrization for the regularizer to get even more similar bounds. If we take $\lambda = \Tilde{\lambda}/\sigma^2$ for some $\Tilde{\lambda} \geq 1$, our bound reads as
$$
    6\sqrt{t\gamma_t^{\tilde\lambda/\sigma^2}}\left(\sigma\sqrt{\log(1/\delta) + \gamma_t^{\tilde\lambda/\sigma^2}} + \tilde \lambda^{1/2} B +B\sqrt{\gamma_t^{\tilde\lambda/\sigma^2}}\right).
$$
Given that by Corollary~\ref{cor:ellipticalpotential}, we have
$$
\gamma_t^{\tilde\lambda/\sigma^2} \leq d\log\left(\frac{t}{\tilde \lambda}+1\right),
$$
we get almost matching bounds, up to an additional $B\sqrt{\gamma_t^{\tilde\lambda/\sigma^2}}$ term blowing up the regret, which we attribute to the accumulation of bias without the reweighting scheme. The remaining differences are down to using slightly different versions of the elliptical potential lemma, trading off generality and tightness \citep{abbasi:improved, hazan:logarithmic, lattimore:book}. 
\section{Experimental Details}
\label{app:experiments}

\subsection{Calibration Plots}
In Figure \ref{fig:calibration} we report the calibration of heuristics as well as other theoretically motivated works. The other theoretically motivated works are very pessimistic and are not appropriately calibrated. Note that one caveat of reporting calibration is that it is very much influenced by the data collection scheme in the sequential regime. In our case we use a bandit algorithm to collect the data. Arguably, in this setting, regret might be a better measure rather than looking at the calibration of the confidence sets. Additionally, the calibration depends on the true value $\tstar$. We report the results for zero parameter and a random parameter from a unit ball. We also report results for i.i.d. data.
 
\begin{figure}[ht!]
	\centering   
	\begin{subfigure}[b]{0.45\textwidth}
		\includegraphics[width=\textwidth]{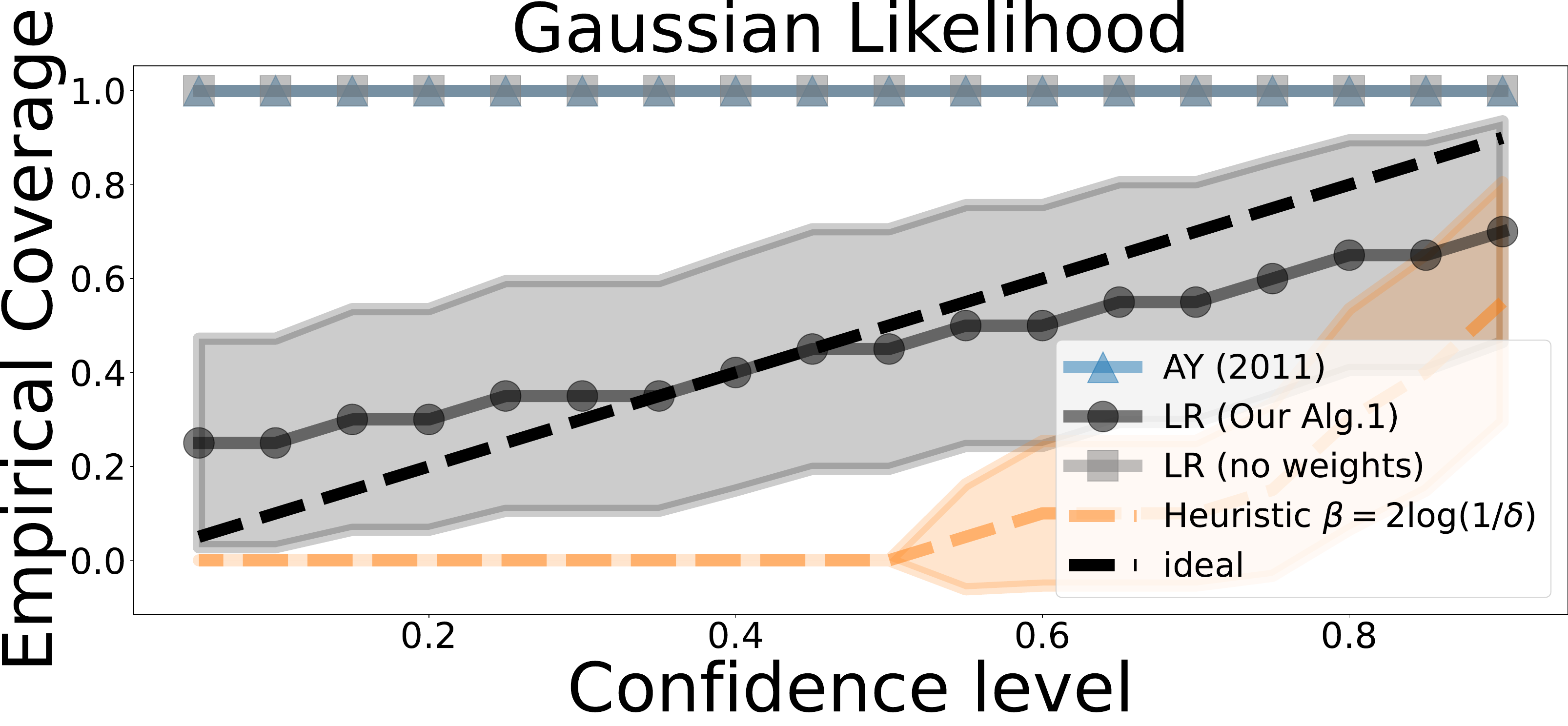}
		\caption{\textbf{ADAPTIVE} Bandit sequence, random $||\theta_\star||_2 =1$, $\sigma = 0.1$}
	\end{subfigure}
	\begin{subfigure}[b]{0.45\textwidth}
		\includegraphics[width=\textwidth]{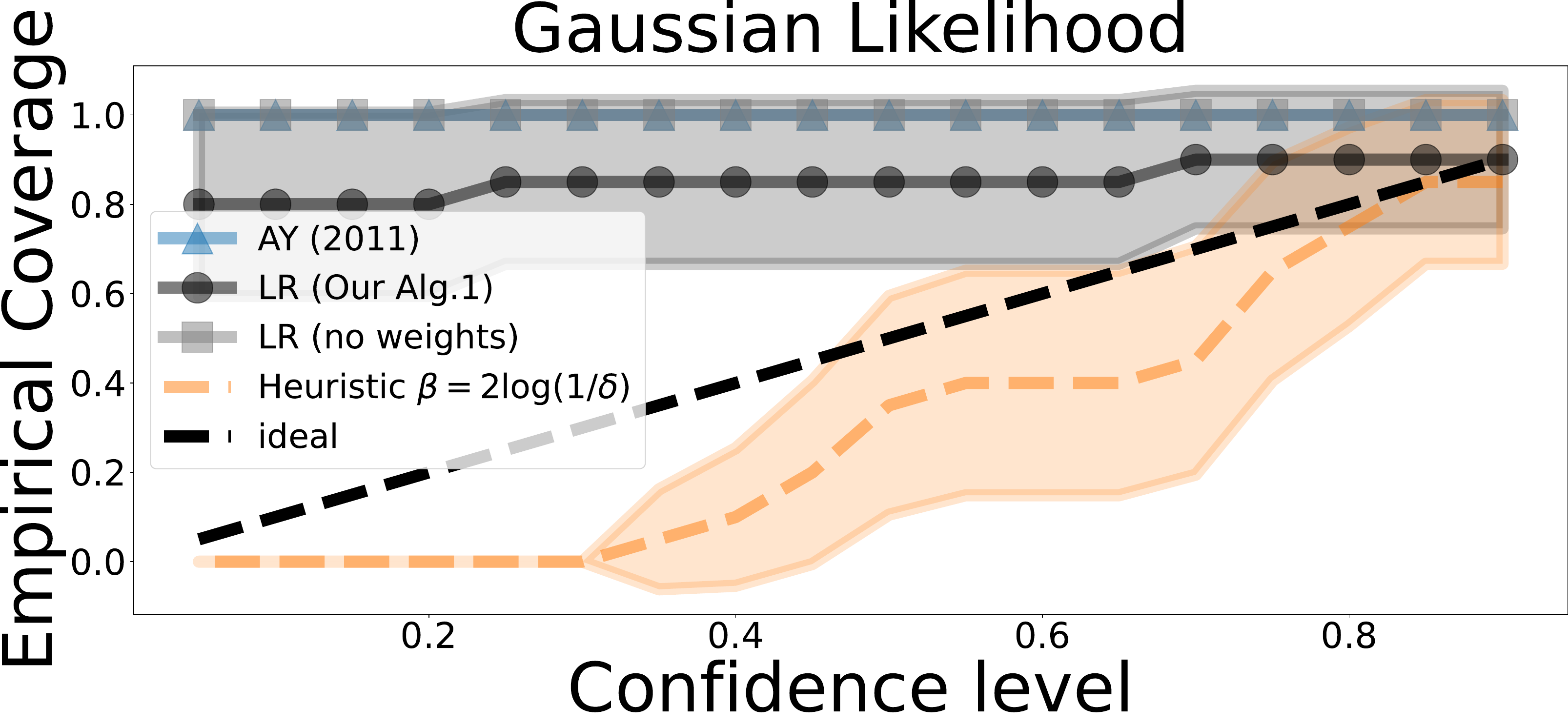}
		\caption{\textbf{ADAPTIVE} Bandit sequence, random $||\theta_\star||_2 =1$, $\sigma = 0.01$}
		
	\end{subfigure}
	
	\begin{subfigure}[b]{0.45\textwidth}
		\includegraphics[width=\textwidth]{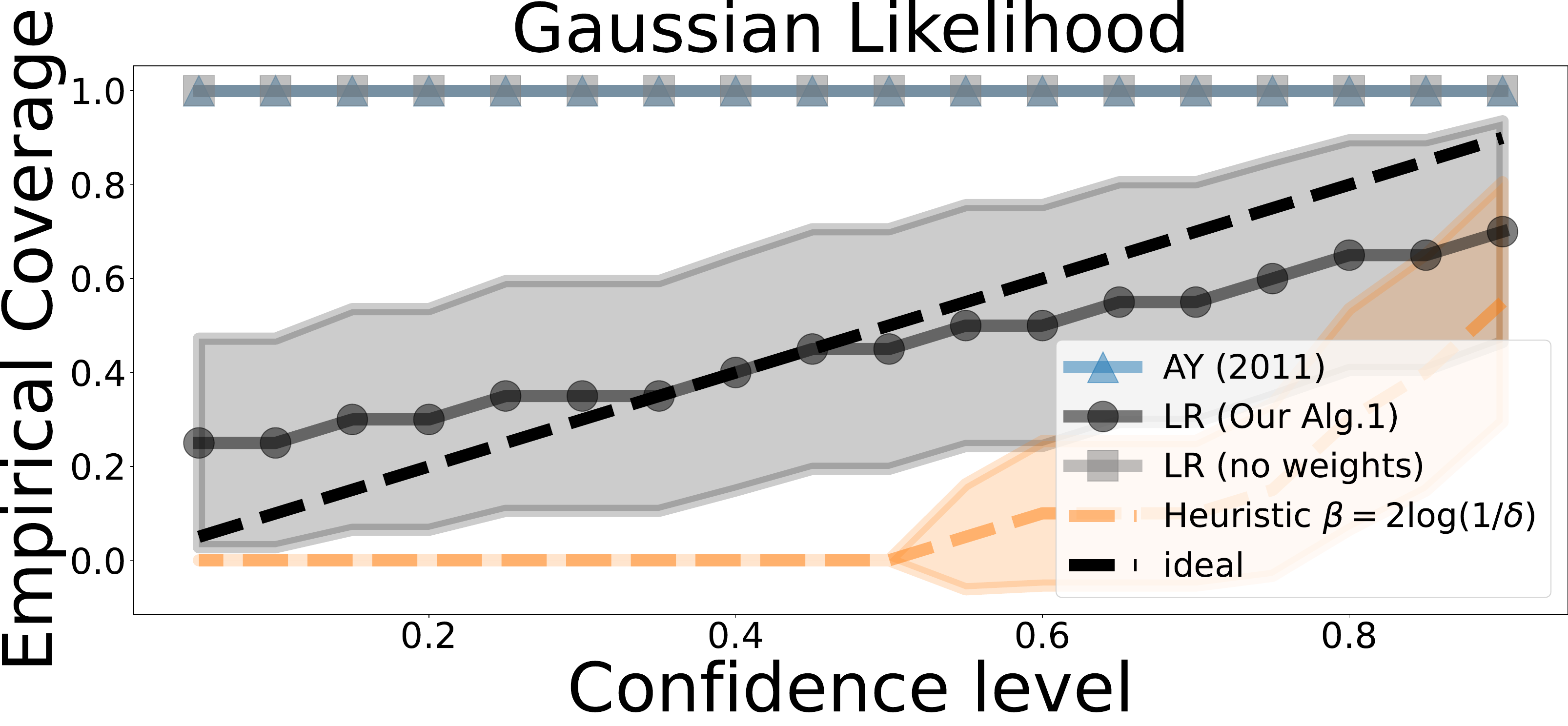}
		\caption{\textbf{ADAPTIVE} Bandit sequence, $\theta_\star = 0$, $\sigma = 0.1$}
		
	\end{subfigure}
	\begin{subfigure}[b]{0.45\textwidth}
		\includegraphics[width=\textwidth]{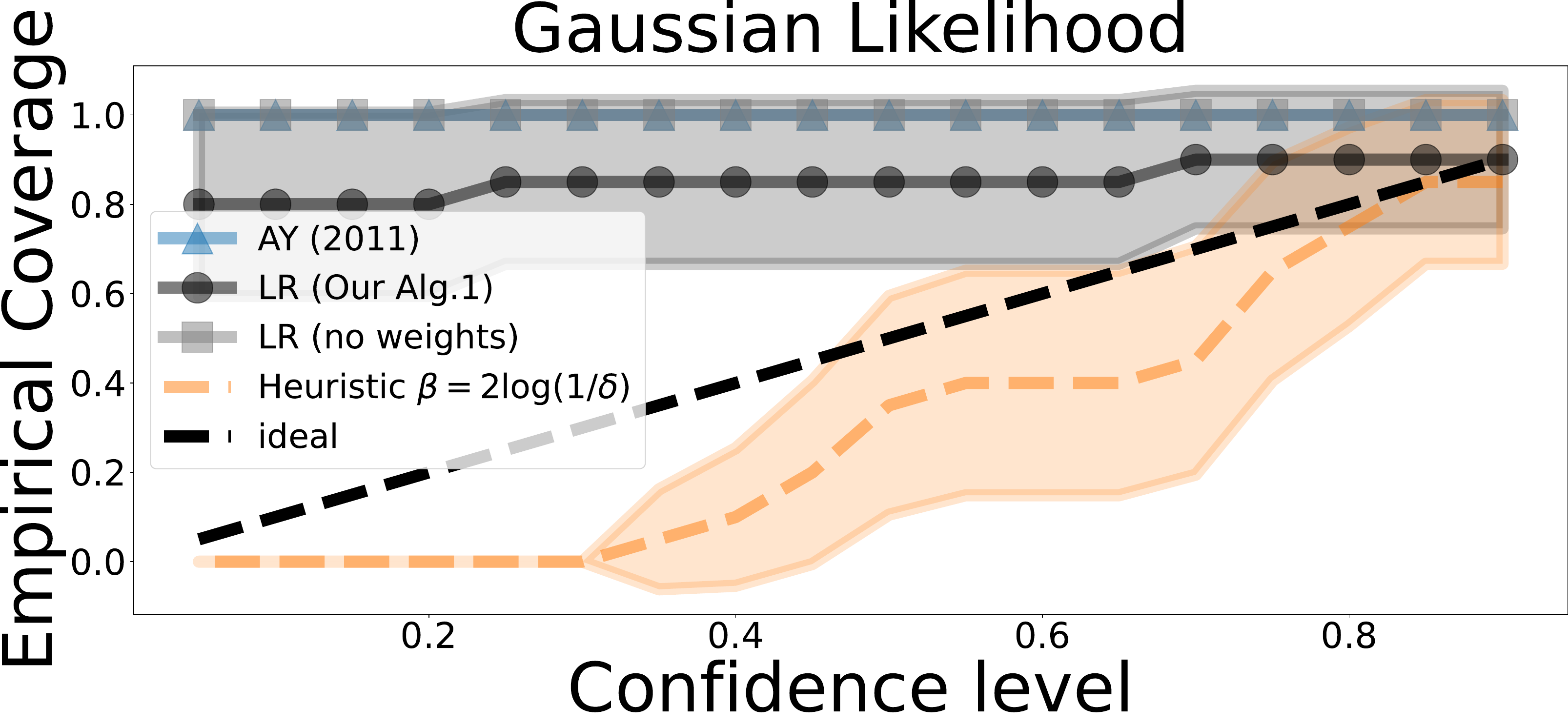}
		\caption{\textbf{ADAPTIVE} Bandit sequence $\theta_\star = 0$, $\sigma = 0.01$}
		
	\end{subfigure}

	\begin{subfigure}[b]{0.45\textwidth}
		\includegraphics[width=\textwidth]{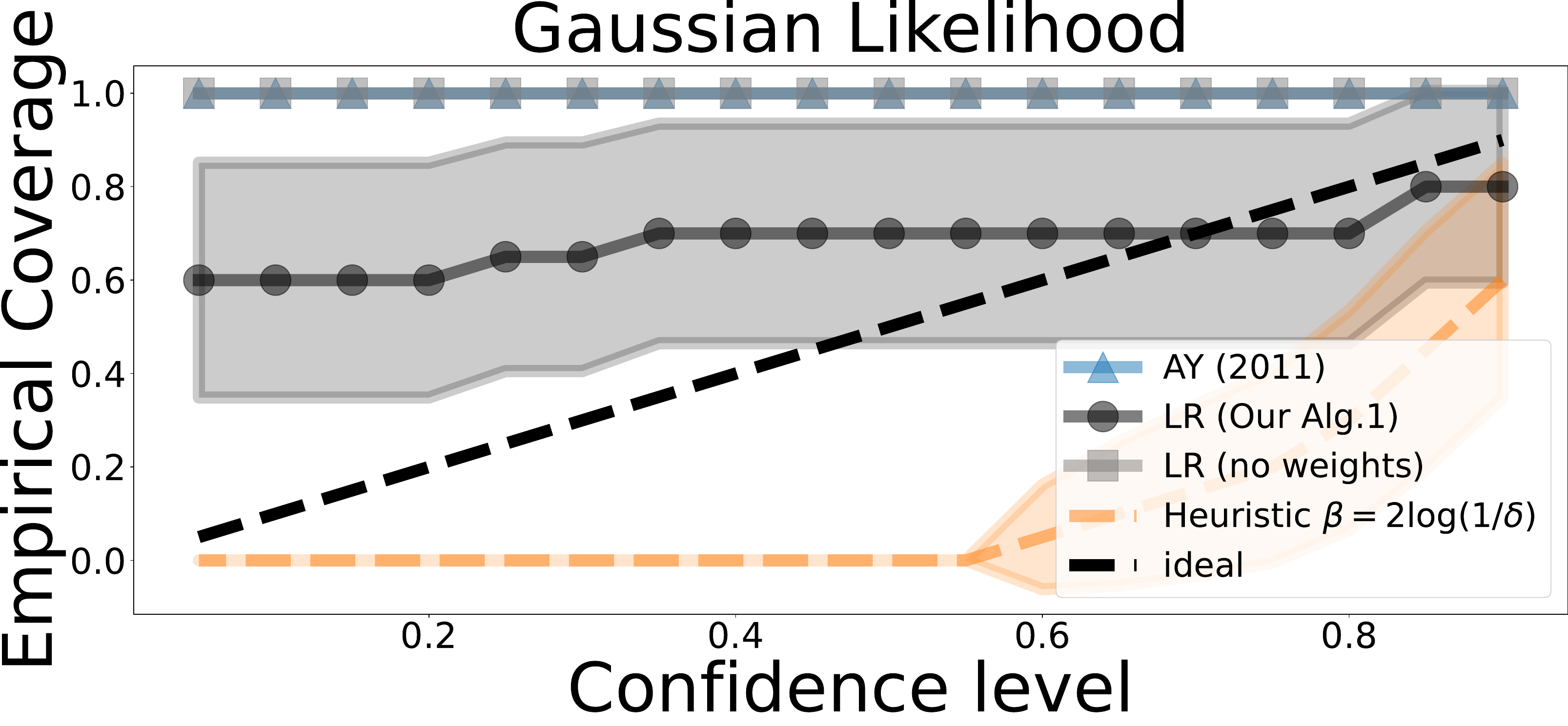}
		\caption{\textbf{IID} sequence, random $||\theta_\star||_2 =1$, $\sigma = 0.1$}
		
	\end{subfigure}
	\begin{subfigure}[b]{0.45\textwidth}
		\includegraphics[width=\textwidth]{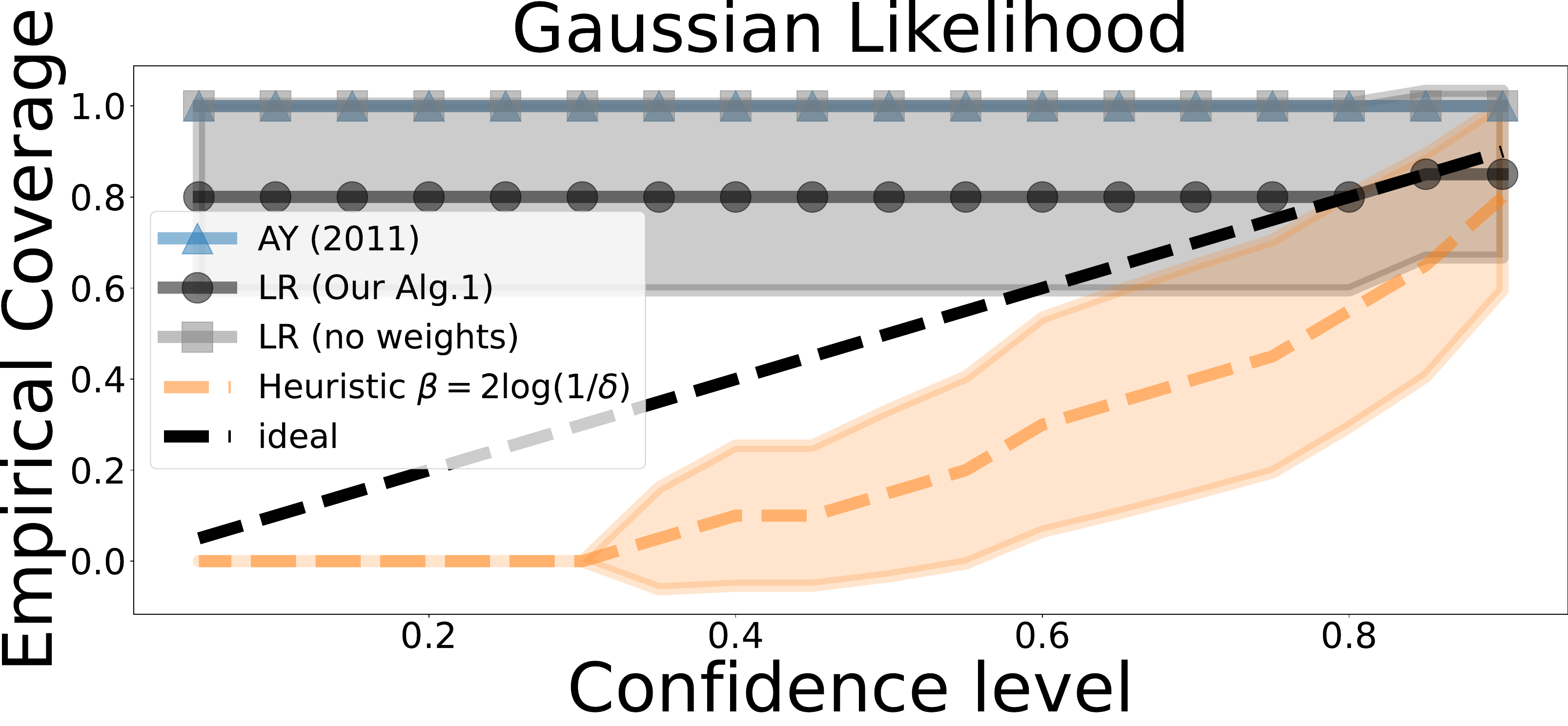}
		\caption{\textbf{IID} sequence, random $||\theta_\star||_2 =1$, $\sigma = 0.01$}
		
	\end{subfigure}
	
	\begin{subfigure}[b]{0.45\textwidth}
		\includegraphics[width=\textwidth]{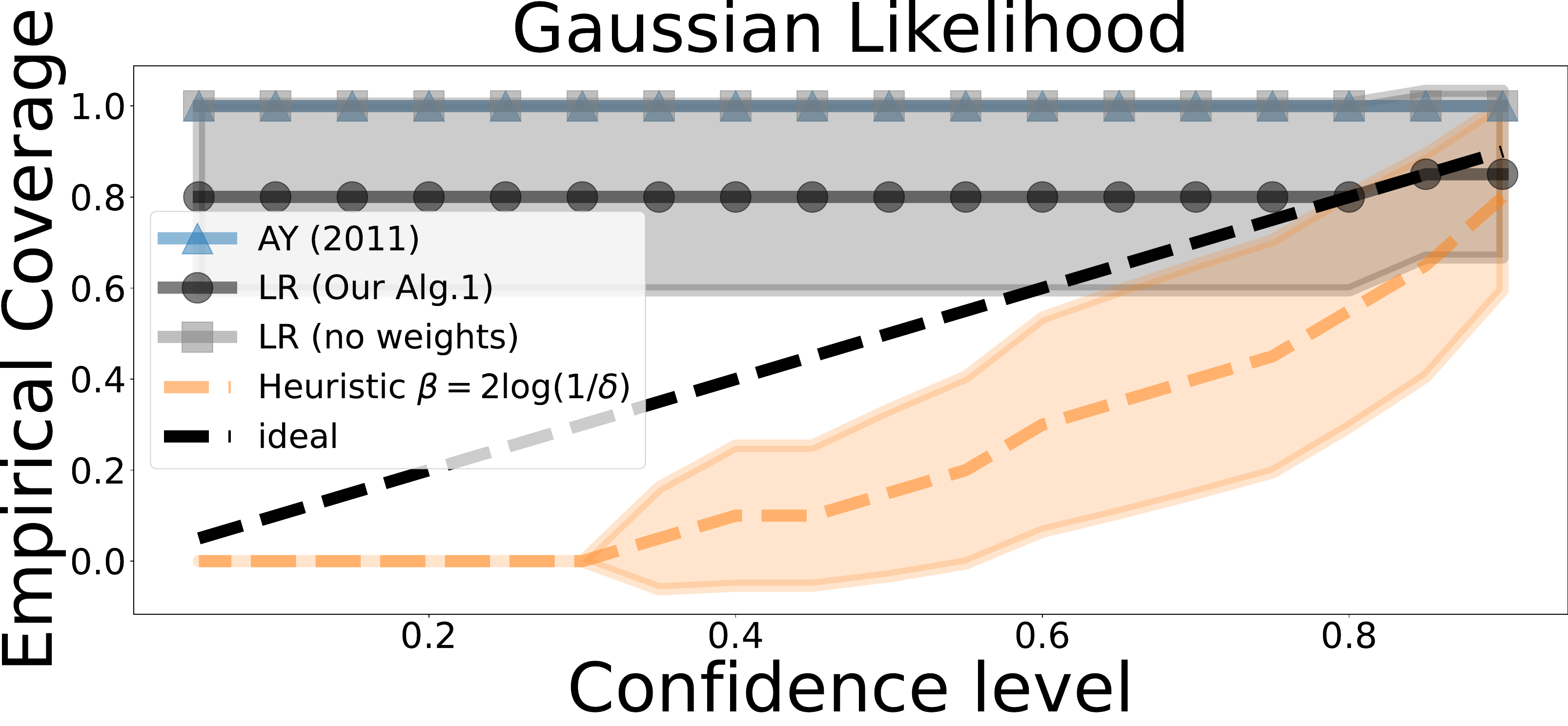}
		\caption{\textbf{IID} sequence, $\theta_\star = 0$, $\sigma = 0.01$}
		
	\end{subfigure}
	\begin{subfigure}[b]{0.45\textwidth}
		\includegraphics[width=\textwidth]{figures/calibration_zero_iidplot0.01.pdf}
		\caption{\textbf{IID} sequence $\theta_\star = 0$, $\sigma = 0.01$}
		
	\end{subfigure}
	
	\caption{We plot the calibration diagram for data collected from a bandit game trying to optimize a ground truth function using the same model as in Fig. 2a). Instead of the test function, we use an explicit member of the confidence set to avoid a potential mismatch between models. We check after $T=15$ whether $\theta_\star \in C_t$ and average over 200 runs. We see that the (LR) are more conservative than the ideal calibration, however, they are provably valid and substantially better than any theoretically motivated confidence sets. We also see that the heuristic is not calibrated and fails many times. We see that for i.i.d. data, our sets are somewhat conservative since the data is not adapted, and our approach is not necessary. We note that the results depend on $\theta_\star$.}
	\label{fig:calibration}
\end{figure}

\subsection{Baselines and Details}
In the following section, we describe the baseline we used in the comparison. The details of parameters used in the experiments can be found in Table \ref{tab:experiments}. As there is no explicit statement for sub-exponential variables, we give a formal derivation in Appendix~\ref{app:subexponentialsection}. 

\subsubsection{Sub-Exponential Random Variables: Confidence Sets}
For this baseline, we will assume a linear model with additive sub-exponential noise, namely
that there is $\theta_* \in \Theta$ such that $y_t = \inner{\theta_*}{x_t} + \eta_t$, where $\eta_t$ is $(\nu,\gamma)$-conditionally sub-exponential \citep{wainwright:book}.
We let as usual
$$
    \VV_t^{\nu^{-2};\lambda} = \sum_{s=1}^t \frac{ x_s x_s^\trp }{\nu^2} \lambda \II.
$$
With this, one can prove the following time-uniform concentration result:
\begin{proposition}
\label{app:confsubexp}
For any $k \in (0,1)$, the following holds:
   $$ 
    \mathbb{P}\left(\exists t \, : \, \norm{\that_t - \theta^*}_{\VV_t^{\nu^{-2};\lambda}} \geq \sqrt{\beta_{SE}}\right)\leq \delta,
    $$
    where 
    $$
        \sqrt{\beta_{SE}} = \sqrt{\lambda}\norm{\tstar}_2 + \sqrt{\lambda}kB + \frac{d}{\sqrt{\lambda}kB}\log\left(\frac{1}{1-k}\right) + \frac{1}{\sqrt{\lambda}kB}  \log\left(\frac{(\det (\VV_t^{\nu^{-2};\lambda}))^{1/2}}{\delta \det(\sqrt{\lambda}\II)}\right).
    $$
\end{proposition}
The proof is very similar to prior work \cite{faury:logistic, mutny:poisson}, and can be found in Appendix \ref{app:subexponentialsection}. This can readily be applied in the survival analysis (after a suitable transformation explained in the main paper) and Laplace noise experiments.

\subsubsection{Poisson Heuristics}
We implement two heuristics. One is a Bayesian formulation due to the Laplace method, and one is due to considerations in \citet{mutny:poisson} of how to construct a valid confidence set using the Fisher information. The Laplace method creates an ellipsoidal confidence set using the second-order information evaluated at the penalized maximum likelihood estimator. Namely, the second derivative of the likelihood evaluated at the maximum penalized likelihood $\that_t$ is
\[ \VV_{\text{laplace}} = \sum_{s=1}^t \exp(\that_t^\top x_s)x_sx_s^\top.\]
We use this to define ellipsoidal confidence sets as $||\that_t - \theta||_{\VV_{\text{laplace}}}^2 \leq 2 \log(1/\delta)$. The other heuristic suggests using the worst-case parameter instead, namely 
\[ \VV_{\text{mutny}} = \sum_{s=1}^t \exp(B)x_sx_s^\top.\]
This method would have provable coverage with a proper confidence parameter. Its derivation is beyond the scope of this paper. 
\subsubsection{NR (2021)}
This method follows from \citet{Neiswanger21a}. Per se, this method was not developed to be versatile in terms of likelihood but instead provides a confidence set on $f$, even if it originates from a misspecified prior. Nevertheless, it provides a likelihood-aware confidence sequence that is anytime-valid and doesn't employ worst-case parameters, and hence is a good benchmark for our analysis. The confidence sets are of the form 
\[  \{  \theta ~ | ~ \log \mathcal{L}(\theta) \leq \log(1/\delta) + \log(p(\mathcal{D}))\}, \]
where $\log(p(\mathcal{D}))$ is the current log-evidence of the data given the Gaussian prior. For more information, see \cite{Neiswanger21a}. 

\begin{table}
    \centering
    \begin{tabular}{c|c|c|c|c|c|c}
         Benchmark function & dim & $|\mathcal{X}|$&  $\gamma$ & $B$ & $\lambda$ & Gaussian/Laplace $\sigma$/$b$ \\ \hline
         1D & 1 & $2^6$ & 0.06 & 4& 1 & 0.15 \\
         Camelback & 2& $10^2$ & 0.2& 2 & 1 & 0.10
    \end{tabular}
    \caption{Summary of experimental parameters}
    \label{tab:experiments}
\end{table}

\subsection{Additive Models}
We implemented two likelihoods, namely Gaussian and Laplace. We implemented the discretization of the domain $|\mathcal{X}|$, and in the implementation we used Nystrom features defined on $|\mathcal{X}|$ providing the exact representation of the RKHS on the $\mathcal{X}$, The points were chosen to be on a uniform grid. Notice that for the regularized estimator, we chose the rule of thumb $\lambda = 1/B$ as is motivated in \citep{Mutny2022b}.

The laplace parameter was picked as $b = 0.15$ likewise. Note that Laplace distribution is sub-exponential with parameters $(b, 2b^2)$. We use $1/\sigma^2$ or $1/b$ respectively for the value $L$. Strictly speaking, the Laplace likelihood is not smooth, but a smoothed objective would most likely inherit a value depending on $b$. As we maintain coverage with any choice of weighting, we do not risk invalidity by using a heuristic choice for $L$. 

\subsection{Survival Analysis}
We implemented the method exactly as specified, using the Weibull likelihood with parameter $p=2$. Upon log-transformation, the Gumbel distribution is sub-exponential. To determine the parameter, consider the moment-generating function of the Gumbel distribution ($\beta = 1/p$ in the canonical parameterization): 
\[ \E[e^Xt] = \Gamma(1-t/2)\exp(t) \leq \exp(t^2/2) \quad \text{for} \quad t<1/2,\]
hence, the sub-exponentiality parameter is $1$, and we can use the above sub-exponential confidence sets with value $b =1 $. For the likelihood ratio code, we used $L = \exp(B)$, as this is the leading term of the Hessian of log-likelihood. The function is not smooth everywhere, but on a bounded domain, this turns out to be an appropriate scaling. 

\subsection{Poisson Bandits}
In this case, we implemented a bandit game, where we used the parametrization $r_\theta(x) \sim \text{Poisson}(\exp(-\theta^\top \Phi(x)))$, where $\Phi(x)$ is the RKHS evaluation operator, and $\theta$ is the unknown value. 

We used $L = \exp(B)$, as this is the leading term of the Hessian of log-likelihood in this parametrization.  The function is not smooth everywhere but on a bounded domain this is an appropriate scaling. 
\subsection{Additional Benchmark Functions}
We focus on an additional baseline function: Camelback in 2D, a standard BO benchmark function. The results can be see in Figure~\ref{fig:camelback}.

\begin{figure}
    \centering
    \includegraphics[width=\textwidth]{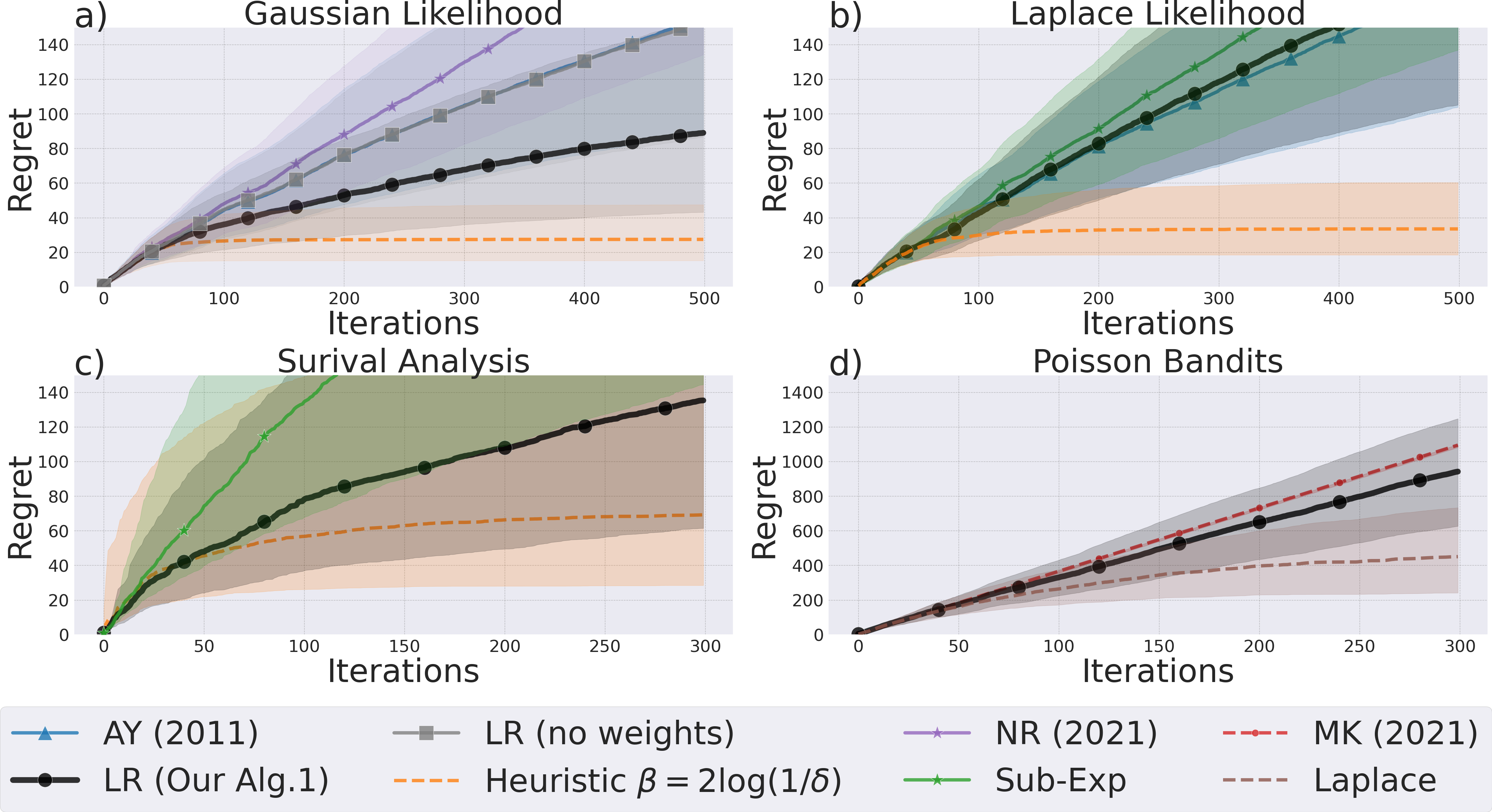}
    \caption{Camelback function. 10 repeats with median and standard quantiles plotted. Note that our method is the best method with provable coverage.}
    \label{fig:camelback}
\end{figure}

\section{Proof of Proposition \ref{app:confsubexp}}
\label{app:subexponentialsection}
Define $S_t$ and the shorthand $\VV_t$
$$
    S_t = \sum_{s=1}^t \eta_s \frac{x_s}{\nu^2} \quad \text{and} \quad \VV_t := \VV^{\nu^{-2};\lambda}_t = \sum_{s=1}^t \frac{x_s x_s^\trp}{\nu^2} + \lambda \II.
$$
and the parametrized process
$$
    \mathcal{M}_t(x) = \exp(\inner{x}{S_t}-\frac{1}{2}\norm{x}_{\VV_t}^2).
$$
\begin{lemma}
If $\eta_t$ is $(\nu,\gamma)$-conditionally sub-Exponential, then $\mathcal{M}_t(x)$ is a super-martingale on the ball $\{x \in \R^d \, \vert \, \norm{x}_2 \leq \frac{\nu^2}{\gamma} \}$ with $\mathcal{M}_0(x) \leq 1.$
\end{lemma}
Note that for $\gamma \rightarrow 0$ (sub-Gaussian case), this recovers Lemma 20.2. in \citet{lattimore:book}.

\begin{proof}
It is easy to observe that for any $x$, we have
$$
    \exp\left(S_0^\trp x - \frac{1}{2}\norm{x}_{\VV_0}^2\right) = \exp\left(- \frac{1}{2}\norm{x}_{\lambda \II}^2\right) \leq 1.
$$
For the first part, we can write
\newcommand{\prevf}{\mathcal{F}_{t-1}}
\begin{align*}
    \E[\mathcal{M}_t(x) | \prevf] &= \E[\exp(\inner{x}{S_t}-\frac{1}{2}\norm{x}_{\VV_t}^2) \sep \prevf] \\
    &= \E[\exp(\inner{x}{S_{t-1}}-\frac{1}{2}\norm{x}_{\VV_{t-1}}^2)\exp(\frac{1}{\nu^2}\inner{x}{\eta_t x_t} - \frac{1}{2\nu^2}\norm{x}^2_{x_t x_t^\trp}) \sep \prevf] \\
    &= \mathcal{M}_{t-1}(x) \E[\exp(\frac{1}{\nu^2}\eta_t\inner{x}{x_t})\sep \mathcal{F}_{t-1}]\exp(-\frac{1}{2\nu^2}\norm{x}_{x_t x_t^\trp}^2),
\end{align*}
where in the last step we use that $x_t$ is $\mathcal{F}_{t-1}$-measurable. Now, as long as $\frac{1}{\nu^2}\abs{\inner{x}{x_s}} \leq \frac{1}{\gamma}$, we can apply our definition of conditional sub-Exponential noise to bound
$$
    \E[\exp(\eta_t \frac{1}{\nu^2} \inner{x}{x_t})] \leq \exp\left(\frac{(\inner{x}{x_t})^2\nu^2}{2\nu^4}\right) \leq \exp\left(\frac{\norm{x}_{x_tx_t^\trp}^2}{2\nu^2}\right).
$$
From this we directly conclude 
$$
    \E[\mathcal{M}_t(x) | \prevf] \leq \mathcal{M}_{t-1}(x).
$$
By Cauchy-Schwarz, a sufficient condition is $\norm{x}_2 \leq \frac{\nu^2}{\gamma}$ as this implies (with our assumptions on the actions) $$
\abs{\inner{x}{x_t}} \leq \norm{x}_2\norm{x_t}_2 \leq \norm{x}_2 \leq \frac{\nu^2}{\gamma}.
$$
\end{proof}
In the following, will use this result to prove any-time confidence estimates for the parameter $\theta$ using the technique of pseudo-maximization, closely following \cite{mutny:poisson}.

Recall that $\mathcal{M}_t(x)$ is defined on the ball of radius $\frac{\nu^2}{\gamma L }$. This allows us some freedom in choosing the radius of the ball on which we integrate. In particular, let $K$ be this radius. While we need $K \leq \frac{\nu^2}{\gamma L }$, we can make $K$ larger by choosing larger $\nu^2$ (increasing $\nu^2$ only makes the set of noise distributions larger).

Ultimately, we wish to bound (following \citet{lattimore:book})
\begin{align*}
    \norm{\that_t - \theta^*}_{\VV_t} 
    &= \norm{\VV_t^{-1}X_{1:t}(X_{1:t}\tstar + \eta_{1:t}) - \tstar}_{\VV_t} \\
    &=  \norm{\VV_t^{-1}X_{1:t}X_{1:t}\tstar - \tstar + \VV_t^{-1}S_t}_{\VV_t} \\
    &\leq  \norm{S_t}_{\VV_t^{-1}} + \sqrt{\lambda}\norm{\tstar}_2.
\end{align*}

\newcommand{\Vtil}{\tilde{\VV}_t}
We can not control the second term, so we focus on the first: the self-normalized residuals. Via fenchel duality, one can motivate that the right object to study is the supremum of the martingale $\mathcal{M}_t(x)$ over all $x\in \R$.\footnote{But that is in our case ill-defined}. Define $\tilde{M}_t$ to be the martingale $\mathcal{M}_t$ from above but with $\lambda = 0$, i.e. no regularisation term. Similarly, let $\Vtil = \VV_t - \lambda \II$ be the design matrix without the regularisation term. Slightly counterintuitively, we will study 
$$
    \bar{M}_t = \int_{\norm{x}_2 \leq K} \tilde{M}_t(x)\dd h(x),
$$
where $h$ is the probability density function of a truncated normal distribution with inverse variance $\lambda$, that is with covariance matrix $\frac{1}{\lambda}\II$. By Lemma 20.3 in \cite{lattimore:book}, $\bar{M}_t$ is also a super-martingale with $\bar{M}_0 \leq 1$.
Then we have
\begin{align*}
    \bar{M}_t &= \frac{1}{N(h)}\int_{\norm{x}_2 \leq K} \exp\left(x^\trp S_t - \frac{1}{2} \norm{x}_{\Vtil}^2\right)\exp\left(-\frac{1}{2} x^\trp\lambda \II x  \right) \dd x \\
    &= \frac{1}{N(h)}\int_{\norm{x}_2 \leq K} \exp\left(x^\trp S_t - \frac{1}{2}\norm{x}_{\VV_t}^2\right) \dd x.
\end{align*}
We will define the shorthand $f_t(x) = x^\trp S_t - \frac{1}{2}x^\trp \VV_t x = f_t(x^*) + \nabla f_t(x^*)^\trp (x-x^*) - \frac{1}{2}(x-x^*)^\trp \VV_t (x-x^*)$ (by Taylor's theorem), where $x^* = \arg \max_{\norm{x}\leq kK} f_t(x)$, $k \in (0,1)$ will be chosen later. We can lower bound $\bar{M}_t$ by
\begin{align}
    \bar{M}_t &= \frac{1}{N(h)}\int_{\norm{x}_2 \leq K} \exp\left(x^\trp S_t - \frac{1}{2}\norm{x}_{\VV_t}^2\right) \dd x \nonumber \\
    &= \frac{\exp(f_t(x^*))}{N(h)}\int_{\norm{x}_2 \leq K} \exp\left( \nabla f_t(x^*)^\trp (x-x^*) - \frac{1}{2}(x-x^*)^\trp \VV_t (x-x^*) \right) \dd x \nonumber \\
    &= \frac{\exp(f_t(x^*))}{N(h)}\int_{\norm{y + x^*}_2 \leq K} \exp\left( \nabla f_t(x^*)^\trp y - \frac{1}{2}y^\trp \VV_t y \right) \dd y \label{eq:pseudo:changeofvar} \\
    &\geq \frac{\exp(f_t(x^*))}{N(h)}\int_{\norm{y}_2 \leq (1-k)K} \exp\left( \nabla f_t(x^*)^\trp y - \frac{1}{2}y^\trp \VV_t y \right) \dd y  \label{eq:pseudo:volumechange} \\
    &= \frac{\exp(f_t(x^*))}{N(h)}\int_{\norm{y}_2 \leq (1-k)K} \exp\left( \nabla f_t(x^*)^\trp y\right)\exp\left( - \frac{1}{2}y^\trp \VV_t y \right) \dd y  \nonumber \\
    &= \frac{\exp(f_t(x^*))N(g)}{N(h)} \E_{y \sim g}\left[\exp\left( \nabla f_t(x^*)^\trp y\right) \right] \nonumber \\
    &\geq \frac{\exp(f_t(x^*))N(g)}{N(h)} \exp\left(\E_{y \sim g}[\nabla f_t(x^*)^\trp y] \right) \label{eq:pseudo:jensen}\\
    &= \frac{\exp(f_t(x^*))N(g)}{N(h)}. \nonumber
\end{align}
where in step \eqref{eq:pseudo:changeofvar} we used the change of variables $x=y+x^*$. In \eqref{eq:pseudo:volumechange} we use that if $\norm{y}_2 \leq (1-k)K$, then $\norm{x^*+y}_2 \leq \norm{x^*}_2 + \norm{y}_2 \leq (1-k)K + kK = K$. Finally, in \eqref{eq:pseudo:jensen}, we used Jensen's inequality. The last inequality follows from symmetry. Note that we implicitly defined $g$ to be a truncated normal distribution with covariance matrix $\VV_t^{-1}$ on the ball of radius $(1-k)K$. \\

This puts us in a position to put Ville's inequality to good use:
\begin{align*}
    \delta &\geq \mathbb{P}\left( \exists t \,:\, \log(\bar{M_t}) \geq \log(1/\delta) \right) \\
    &\geq \mathbb{P}\left( \exists t \,:\, f_t(x^*) + \log\left(\frac{N(g)}{N(h)}\right) 
    \geq \log(1/\delta) \right) \\
    &\geq \mathbb{P}\left( \exists t \,:\, f_t(x^*)
    \geq \log\left(\frac{N(h)}{N(g)\delta}\right) \right).
\end{align*}

We now wish to recover $\norm{S_t}_{\VV_t}$. Recall the definition of $f_t(x^*)$ as the maximum over all $x$ in a ball of radius $kK$. Consequently, we can choose $x = \frac{\VV_t^{-1}S_t}{\norm{S_t}_{\VV_t^{-1}}}\sqrt{\lambda}kK$, which has norm bounded by $kK$. We have
$$
    f_t(x^*) \geq f_t\left(\frac{\VV_t^{-1}S_t}{\norm{S_t}_{\VV_t^{-1}}}\sqrt{\lambda}kK\right) = \norm{S_t}_{\VV_t^{-1}}\sqrt{\lambda}kK - \lambda k^2K^2,
$$
which immediately yields
\begin{align*}
    \mathbb{P}\left( \norm{S_t}_{\VV_t^{-1}} \geq \sqrt{\lambda}kK + \frac{1}{\sqrt{\lambda}kK} \log\left(\frac{N(h)}{N(g)\delta}\right) \right) \leq \delta.
\end{align*}
The only thing that remains is bounding $\log\left(\frac{N(h)}{N(g)}\right)$.

We give the following Lemma that is a slightly generalized version of \citet{mutny:poisson} and originally inspired by \citet{faury:logistic}. 
\begin{lemma}
\label{lemma:faury}
The normalizing constants satisfy
$$
\log\left(\frac{N(h)}{N(g)}\right) \leq d\log\left(\frac{1}{1-k}\right) + \log\left(\frac{(\det (\VV_t))^{1/2}}{\det( \sqrt{\lambda}\II)}\right).
$$
\end{lemma}
We can use the bound from Lemma~\ref{lemma:faury} to conclude that
\begin{align*}
    \mathbb{P}\left( \norm{S_t}_{\VV_t^{-1}} \geq \sqrt{\lambda}kK + \frac{d}{\sqrt{\lambda}kK}\log\left(\frac{1}{1-k}\right) + \frac{1}{\sqrt{\lambda}kK}  \log\left(\frac{(\det (\VV_t))^{1/2}}{\delta \det(\sqrt{\lambda}\II)}\right) \right) \leq \delta.
\end{align*}

We stated earlier that 
$$
\norm{\that_t - \theta^*}_{\VV_t} \leq  \norm{S_t}_{\VV_t^{-1}} + \sqrt{\lambda}\norm{\tstar}_2.
$$
Combining this with our analysis, we get the Proposition \ref{app:confsubexp}.

We may now choose the parameters $k$, $K$ and $\lambda$. Note that to get sub-Gaussian rates as in Abbasi-Yadkori, one needs to pick a regularization parameter of the order of $\lambda = d\log(T)$.

\paragraph{Proof of Lemma~\ref{lemma:faury}}
We give a proof of the Lemma for completeness, and because the additional generality makes for a slightly different proof, even though the bound stays the same.
\begin{proof}
We have
\begin{align*}
    N(h) &= \int_{\norm{x}_2 \leq K} \exp(-\lambda \norm{x}_2^2) \mathrm{d}x \\
    &= \frac{1}{\abs{\det(\sqrt{2\lambda} \II)}}\int_{\norm{x}_2 \leq K} \exp\left(-\frac{1}{2} \norm{\sqrt{2\lambda}x}_2^2 \right) \abs{\det(\sqrt{2\lambda} \II)} \mathrm{d}x \\
    &= \frac{1}{\abs{\det(\sqrt{2\lambda} \II)}}\int_{\norm{x}_2 \leq \sqrt{2\lambda}K} \exp\left(-\frac{1}{2} \norm{x}_2^2 \right) \mathrm{d}x.
\end{align*}
Further we have
\begin{align*}
    N(g) &= \int_{\norm{x}_2 \leq (1-k)K} \exp(-\frac{1}{2} x^\trp \VV_t x) \mathrm{d}x \\
    &= \frac{1}{\abs{\det(\VV_t^{1/2})} }\int_{\norm{x}_2 \leq (1-k)K} \exp(-\frac{1}{2} \norm{\VV_t^{1/2}x}_2^2)\abs{\det (\VV_t^{1/2})} \mathrm{d}x \\
    &= \frac{1}{\abs{\det(\VV_t^{1/2})} }\int_{S}\exp(-\frac{1}{2} \norm{x}_2^2) \mathrm{d}x,
\end{align*}
where $S = \{\VV_t^{1/2}x \sep \norm{x} \leq (1-k)K \} = \{x \sep \norm{\VV_t^{-1/2}x} \leq (1-k)K \} =  \{x \sep x^\trp \VV_t^{-1}x \leq (1-k)K \}$. Note that $\VV_t \succeq \lambda \II$ and so $\VV_t^{-1} \preceq \frac{1}{\lambda} \II$. Therefore if
$\norm{x}_2^2 \leq (1-k)K \sqrt{\lambda}$, we have 
$$
    \sqrt{x^\trp \VV_t^{-1}x} \leq \frac{1}{\sqrt{\lambda}} \norm{x}_2 \leq (1-k)K \implies x \in S.
$$
Thus $\{x \sep \norm{x}_2 \leq (1-k)\sqrt{\lambda}K\} \subseteq S$ and
$$
    N(g) \geq \frac{1}{\abs{\det(\VV_t^{1/2})} }\int_{\norm{x}_2 \leq (1-k)\sqrt{\lambda}K}\exp(-\frac{1}{2} \norm{x}_2^2) \mathrm{d}x.
$$
We may therefore bound
$$
    \frac{N(g)}{N(h)} \leq \frac{(\det \VV_t)^{1/2}}{(\det \sqrt{2\lambda}\II)} \frac{\int_{\norm{x}_2 \leq \sqrt{2\lambda}K} \exp\left(-\frac{1}{2} \norm{x}_2^2 \right) \mathrm{d}x}{\int_{\norm{x}_2 \leq (1-k)\sqrt{\lambda}K}\exp(-\frac{1}{2} \norm{x}_2^2) \mathrm{d}x}.
$$
By a rather crude bound (as $1-k \leq \sqrt{2}$ in any case) we get
\allowdisplaybreaks

\begin{align*}
    & \frac{\int_{\norm{x}_2 \leq \sqrt{2\lambda}K} \exp\left(-\frac{1}{2} \norm{x}_2^2 \right) \mathrm{d}x}{\int_{\norm{x}_2 
    (1-k)\sqrt{\lambda}K}\exp(-\frac{1}{2} \norm{x}_2^2) \mathrm{d}x} \\
    \leq \, \, & \frac{\int_{\norm{x}_2 \leq (1-k)\sqrt{\lambda}K} \exp\left(-\frac{1}{2} \norm{x}_2^2 \right) \mathrm{d}x +\int_{(1-k)\sqrt{\lambda}K \leq \norm{x}_2 \leq \sqrt{2\lambda}K} \exp\left(-\frac{1}{2} \norm{x}_2^2 \right) \mathrm{d}x}{\int_{\norm{x}_2 \leq (1-k)\sqrt{\lambda}K}\exp(-\frac{1}{2} \norm{x}_2^2) \mathrm{d}x} \\
    = \,\, & 1+\frac{\int_{(1-k)\sqrt{\lambda}K \leq \norm{x}_2 \leq \sqrt{2\lambda}K} \exp\left(-\frac{1}{2} \norm{x}_2^2 \right) \mathrm{d}x}{\int_{\norm{x}_2 \leq (1-k)\sqrt{\lambda}K}\exp(-\frac{1}{2} \norm{x}_2^2) \mathrm{d}x} \\
     \leq \,\, & 1+ \frac{\exp\left(-\frac{1}{2} (1-k)^2\lambda K^2 \right) }{\exp\left(-\frac{1}{2} (1-k)^2\lambda K^2 \right)} \frac{\int_{(1-k)\sqrt{\lambda}K \leq \norm{x}_2 \leq \sqrt{2\lambda}K} \mathrm{d}x}{\int_{\norm{x}_2 \leq (1-k)\sqrt{\lambda}K} \mathrm{d}x} \\
     = \,\, & 1 + \frac{\mathrm{vol}_d(\sqrt{2\lambda}K) - \mathrm{vol}_d((1-k)\sqrt{\lambda}K)}{\mathrm{vol}_d((1-k)\sqrt{\lambda}K)} \\
     = \,\, & \frac{\mathrm{vol}_d(\sqrt{2\lambda}K)}{\mathrm{vol}_d((1-k)\sqrt{\lambda}K)} \\
     = \,\, & (1-k)^{-d} \sqrt{2}^d.
\end{align*}
We can put this together to obtain
$$
\frac{N(h)}{N(g)} \leq (1-k)^{-d} \sqrt{2}^d \frac{(\det (\VV_t))^{1/2}}{\det (\sqrt{2\lambda}\II)} = (1-k)^{-d} \frac{(\det (\VV_t))^{1/2}}{\det( \sqrt{\lambda}\II)}.
$$

\end{proof}

\end{document}